\documentclass[twoside,11pt]{article}

\usepackage{blindtext}

\usepackage{amssymb}
\usepackage{url}
\usepackage{makecell}
\usepackage{bbding}

\usepackage{booktabs}
\usepackage{amsthm}
\usepackage{multirow}
\usepackage{subfigure}
\usepackage{graphicx}
\usepackage{xcolor}

\usepackage{jmlr2e}

\usepackage{amsmath,amsfonts,bm}

\def\eqref#1{equation~\ref{#1}}
\def\Eqref#1{Equation~\ref{#1}}

\def\1{\bm{1}}

\def\eps{{\epsilon}}

\def\va{{\bm{a}}}
\def\vb{{\bm{b}}}

\def\vf{{\bm{f}}}

\def\vh{{\bm{h}}}

\def\vm{{\bm{m}}}

\def\vs{{\bm{s}}}

\def\vv{{\bm{v}}}
\def\vw{{\bm{w}}}
\def\vx{{\bm{x}}}
\def\vy{{\bm{y}}}
\def\vz{{\bm{z}}}

\def\mA{{\bm{A}}}
\def\mB{{\bm{B}}}
\def\mC{{\bm{C}}}
\def\mD{{\bm{D}}}

\def\mH{{\bm{H}}}
\def\mI{{\bm{I}}}

\def\mL{{\bm{L}}}
\def\mM{{\bm{M}}}

\def\mP{{\bm{P}}}

\def\mS{{\bm{S}}}

\def\mU{{\bm{U}}}

\def\mW{{\bm{W}}}

\def\mY{{\bm{Y}}}
\def\mZ{{\bm{Z}}}

\DeclareMathAlphabet{\mathsfit}{\encodingdefault}{\sfdefault}{m}{sl}
\SetMathAlphabet{\mathsfit}{bold}{\encodingdefault}{\sfdefault}{bx}{n}

\def\gB{{\mathcal{B}}}

\def\gD{{\mathcal{D}}}
\def\gE{{\mathcal{E}}}

\def\gL{{\mathcal{L}}}

\def\gS{{\mathcal{S}}}

\def\sR{{\mathbb{R}}}
\def\sS{{\mathbb{S}}}

\newcommand{\E}{\mathbb{E}}

\DeclareMathOperator*{\argmax}{arg\,max}
\DeclareMathOperator*{\argmin}{arg\,min}

\DeclareMathOperator{\Tr}{Tr}

\usepackage[capitalize,noabbrev]{cleveref}

\usepackage{lastpage}
\jmlrheading{24}{2023}{1-\pageref{LastPage}}{6/23; Revised
10/23}{12/23}{23-0771}{Xiong Zhou, Xianming Liu, Hanzhang Wang, Deming Zhai, Junjun Jiang, and Xiangyang Ji}
\ShortHeadings{title}
{Xiong Zhou, Xianming Liu, Hanzhang Wang, Deming Zhai, Junjun Jiang, and Xiangyang Ji}

\ShortHeadings{On the Dynamics Under the Unhinged Loss and Beyond}{Zhou, Liu, Wang, Zhai, Jiang, and Ji}
\firstpageno{1}

\newcommand\revise[1]{{#1}}

\begin{document}

\title{On the Dynamics Under the Unhinged Loss and Beyond}

\author{\name Xiong Zhou\email cszx@hit.edu.cn \\
        \name Xianming Liu\thanks{The corresponding author} \email csxm@hit.edu.cn\\
        \name Hanzhang Wang \email cswhz@hit.edu.cn\\
        \name Deming Zhai \email zhaideming@hit.edu.cn\\
        \name Junjun Jiang \email jiangjunjun@hit.edu.cn\\
       \addr School of Computer Science and Technology\\
       Harbin Institute of Technology\\
       Harbin, 150001, China
       \AND
       \name Xiangyang Ji \email xyji@tsinghua.edu.cn \\
       \addr Department of Automation\\
       Tsinghua University\\
       Beijing, 100084, China}

\editor{Samy Bengio}

\maketitle

\begin{abstract}
    Recent works have studied implicit biases in deep learning, especially the behavior of last-layer features and classifier weights. However, they usually need to simplify the intermediate dynamics under gradient flow or gradient descent due to the intractability of loss functions and model architectures. In this paper, we introduce the unhinged loss, a concise loss function, that offers more mathematical opportunities to analyze the closed-form dynamics while requiring as few simplifications or assumptions as possible. The unhinged loss allows for considering more practical techniques, such as time-vary learning rates and feature normalization. Based on the layer-peeled model that views last-layer features as free optimization variables, we conduct a thorough analysis in the unconstrained, regularized, and spherical constrained cases, as well as the case where the neural tangent kernel remains invariant. To bridge the performance of the unhinged loss to that of Cross-Entropy (CE), we investigate the scenario of fixing classifier weights with a specific structure, (e.g., a simplex equiangular tight frame). Our analysis shows that these dynamics converge exponentially fast to a solution depending on the initialization of features and classifier weights. These theoretical results not only offer valuable insights, including explicit feature regularization and rescaled learning rates for enhancing practical training with the unhinged loss, but also extend their applicability to other loss functions. Finally, we empirically demonstrate these theoretical results and insights through extensive experiments.
\end{abstract} 


\section{Introduction}
Deep learning with neural networks has achieved great success in a variety of tasks \citep{lecun2015deep},
which, however, is not entirely understood in the interpolation and generalization of the learned models \citep{2017Understanding, neyshabur2017exploring, nakkiran2019deep, bubeck2021universal, Mei2021TheGE, zhou2023asymmetric, LANGER2021103473, QI2021103435}. Many modules, including loss functions \citep{lin2017focal, hui2021evaluation} and optimization algorithms \citep{auer2002adaptive, duchi2011adaptive, zeiler2012adadelta, kingma2015adam}, play a crucial role in the training of deep neural networks, but lack convincing explanations due to the complexity of multilayered architectures. Recent works are devoted to simplifying modeling to better understand the behavior of DNNs,
\citep{papyan2020prevalence, mixon2022neural, fang2021exploring, tian2021understanding, han2022neural} and then to gain insights for new algorithms, theoretical, and experimental investigations.

To better understand the implicit regularization that improves the generalization of trained models in deep learning, many studies have investigated the implicit bias of gradient descent \citep{hardt2016train, sekhari2021sgd}, with an emphasis on the behavior of linear predictors (or called classifiers) over linearly separable data \citep{soudry2018implicit, gunasekar2018implicit, Nacson2019ConvergenceOG, ji2019implicit, ji2020gradient, Shamir2021GradientMN}. In particular, \citet{soudry2018implicit} demonstrated that gradient descent iterates under exponentially-tailed loss minimization on separable data are biased toward $\ell_2$-maximum-margin solutions and that continuing to optimize can still lead to performance improvements, even if the validation loss increases. Further, \citet{ji2020gradient} showed that the gradient descent path and the algorithm-independent regularization path converge to the same direction for general losses. \citet{Shamir2021GradientMN} formally proved that standard gradient methods never overfit on separable data. These works impressively expose the implicit regularization induced by optimization algorithms and help to understand the generalization of the learned models, but they \textit{mainly focus on the behavior of linear classifiers} that is only the last layer of neural networks, while the classifier actually interacts strongly with the features produced by many nonlinear layers and parameterized layers. Thus, the relevant conclusions do not always apply to deep learning. For example, in \citep{soudry2018implicit}, the convergence rate of gradient descent is rather slow, wherein for almost all datasets, the distance to the maximum-margin solution decreases only as $O(1/\log t)$, and in some degenerate datasets, the rate further slows down to $O(\log\log t/\log t)$.
However, the training of DNNs typically takes only a few hundred epochs. In this paper, we show that exponential convergence is more realistic.

This paper is closely related to another research line which emerged after the empirical discovery of Neural Collapse by \citet{papyan2020prevalence}. This phenomenon precisely characterizes a pervasive inductive bias of both features and linear classifiers at the terminal phase of training, and has opened a rich area of exploring this area with simplified mathematical frameworks \citep{mixon2022neural, fang2021exploring,lu2022neural, galanti2021role, fang2021mathematical, zhu2021geometric, hui2022limitations, tirer2022extended, lu2022neural, kothapalli2022neural, zhouall}. Neural collapse provides a clear view of how the last-layer features and linear classifiers behave after interpolation and enables us to understand the benefit of achieving zero training error in terms of generalization and robustness \citep{poggio2020explicit, wang2021benign, kornblith2021better, thrampoulidis2022imbalance, gao2023study}, but the intermediate dynamics have remained challenging to analyze due to the intractability of cross entropy (CE). To alleviate this issue, some studies \citep{mixon2022neural, han2022neural, pmlr-v162-zhou22c, tirer2022extended, kothapalli2022neural, xu2023dynamics} have explored the more tractable mean squared error (MSE) loss, which performs comparably to those trained with CE \citep{demirkaya2020exploring, hui2021evaluation}. However, these studies \textit{still need to make additional simplifications or assumptions for intermediate dynamics}. For instance, \citet{mixon2022neural} formulate the gradient flow of the unconstrained feature model as a nonlinear ordinary differential equation and then linearize the equation by claiming that nonlinear terms are negligible for models initialized near the origin. \citet{han2022neural} assume that the gradient flow is along the central path which requires the linear classifier to stay MSE-optimal for features throughout the dynamics. \revise{Furthermore, \citet{tirer2023perturbation} delve into a more practical model, demonstrating the exponential decay of within-class variability. They assume that features remain in close proximity to predefined features under MSE loss function, but note that the gradient flow is still essential along the ``central path".} Therefore, MSE is still not simple enough to derive exact dynamics in certain mathematical frameworks, making it difficult to grasp and bridge the gap between the modeling and practical optimization.

\begin{table}[t]
\scriptsize
\centering
\begin{tabular}{c|c|c|c}
    \toprule
    Reference & Contribution & \makecell[c]{Extra Simplification\\ or Assumption} & \makecell[c]{Loss\\ Function}\\
    \midrule 
    \citet{papyan2020prevalence} & \makecell{Empirical results and \\ theoretical formulation} & \XSolidBrush & The CE loss\\
    \midrule
    \citet{fang2021exploring} & \makecell{Global optimum in\\ \textbf{regularized cases}} & \XSolidBrush & The CE loss\\
    \midrule
    \citet{zhu2021geometric} & \makecell{Landscape analysis in\\ \textbf{regularized} cases} & \XSolidBrush & The CE loss\\
    \midrule
    \citet{ji2022unconstrained} & \makecell{Landscape analysis in \\ \textbf{unconstrained} cases} & \XSolidBrush & The CE loss\\
    \midrule
    \citet{tirer2022extended} & \makecell{Global optimum with \\ \textbf{extended unconstrained}\\ \textbf{feature model}} & \XSolidBrush & The MSE loss\\
    \midrule
    \citet{mixon2022neural} & \makecell{Intermediate dynamics in\\ \textbf{unconstrained} cases} & \Checkmark & The MSE loss\\
    \midrule
    \citet{han2022neural} & \makecell{Intermediate dynamics in \\\textbf{weight-regularized} cases} & \Checkmark & The MSE loss  \\
    \midrule
    \revise{\citet{tirer2023perturbation}} & \makecell{Intermediate dynamics by\\ perturbation analysis in\\ \textbf{regularized} cases} & \Checkmark & The MSE loss\\
    \midrule
    This paper & \makecell{Intermediate dynamics and \\ convergence analysis in\\ \textbf{unconstrained}, \textbf{regularized}, \\ \textbf{spherical constrained} cases\\ with \textbf{time-varying} \\\textbf{learning rates}} & \XSolidBrush & The unhinged loss \\
    \bottomrule
\end{tabular}
\vskip2pt
\caption{Comparison of recent analysis for investigating the behavior of last-layer features and prototypes. Compared with prior work, this paper considers the time-varying learning rate, which is often used in practice, and provides intermediate dynamics and convergence analysis in unconstrained, regularized and spherical constrained cases while requiring fewer simplifications or assumptions.}
\label{tab:recent-comparison}

\end{table}

In this paper, our objective is to analyze the closed-form dynamics under gradient descent within the layer-peeled model \citep{fang2021exploring} (also known as the unconstrained features model \citep{mixon2022neural}) with minimal simplifications and assumptions. To achieve this, we introduce the unhinged loss, which possesses a concise form and intuitively expresses the classification objective. \revise{To further contextualize the performance of the unhinged loss in relation to the CE loss, we investigate scenarios involving fixing classifier weights with specific structures.} Compared to previous works (as depicted in Table 1), the unhinged loss provides more mathematical opportunities for gaining insights into deep learning with closed-form dynamics, while demanding fewer simplifications or assumptions for intermediate dynamics. This equips us for more practical considerations and more rational designs. Our primary contributions are outlined as follows:
\begin{itemize}
    \item We introduce the unhinged loss for analyzing closed-form dynamics in deep learning with as few simplifications or assumptions as possible. The unhinged loss with some auxiliary techniques can achieve comparative performance to the CE loss.
    
    \item We derive exact dynamics of last-layer features and prototypes in \revise{unconstrained, regularized, prototype-anchored and spherical constrained cases as well as the NTK-invariant case}. For spherical constrained cases that do not exhibit convexity, Lipschitzness, and $\beta$-smoothness, we also prove that gradient descent biases the normalized features towards a global minimizer.
    \item We provide the corresponding convergence analysis, which shows that the features and classifier weights converge to a solution depending on the initialization rather than induce the neural collapse solution that forms a simplex equiangular tight frame, suggesting that not all losses under gradient descent would lead to neural collapse (as verified in \Cref{main-results}).
    \item We prove that the rate of convergence is exponential as a function of $\zeta(t)=\int_0^t\eta(\tau)\mathrm{d}\tau$, where $\eta(\tau)$ denotes the learning rate over time. This exponential convergence rate highlights the impact of the interaction of features and classifier weights.
    \item Moreover, we provide some insights and extensive verification for improvements in practical training with the unhinged loss and other losses (cf. \Cref{insights}).
\end{itemize}

\section{The Unhinged Loss}
In this paper, we primarily investigate the behavior of last-layer features and classifier weights in DNNs for classification tasks. We conduct our study to datasets comprising inputs from $C$ different classes, each with $N$ examples. The last-layer features $\mathbf{h}_{i,c} = \mathbf{f}_{\Theta}(\mathbf{x}_{i,c}) \in \mathbb{R}^p$ \footnote{For the closed-form solution of neural collapse, it is worth noting that \citet{zhou2022learning} assume $p \geq C-1$, while \citet{han2022neural} assume $p > C$, as last-layer features typically have a higher dimension than the number of classes. This work will explicitly address the relationship between $p$ and $C$ in some scenarios, covering all potential choices of feature dimension in others.} are obtained from the $i$-th example $\mathbf{x}_{i,c}$ through a series of parameterized layers $\mathbf{f}_{\Theta}:\mathcal{X}\rightarrow\mathbb{R}^{p}$, commonly treated as free optimization variables \citep{mixon2022neural, fang2021exploring, han2022neural, ji2022unconstrained}. The final layer of the network, known as the linear classifier, is equipped with a class prototype $\mathbf{w}_c \in \mathbb{R}^p$ and a bias $b_c \in \mathbb{R}$ for each class $c \in [C]$. It predicts a label using the rule $\argmax_{c'}(\langle\mathbf{w}_{c'}, \mathbf{h}\rangle + b_{c'})$ \footnote{Unless specified otherwise, we denote vectors using boldface italicized letters, and elements within the vector are denoted with italicized letters and subscripts.}.

To better understand the dynamics of last-layer features and prototypes based on gradient flow or gradient descent, we consider a concise loss that offers more mathematical opportunities than the hard-to-analyze CE loss and the MSE loss \footnote{The CE loss is known to have non-analytical intermediate dynamics, which makes it difficult to investigate the precise nature of the dynamics. In addition, the intermediate dynamics under MSE can be derived into closed-form, which however is quadratic and cannot be easily analyzed. Therefore, some additional simplifications or assumptions on CE and MSE are required to simplify the intermediate dynamics.}. Specifically, we investigate a generalized form of the unhinged loss \citep{van2015learning} as follows:
\begin{equation}
\label{the-unhinged-loss}
L_{\gamma}(\mW\vh+\vb,y)=-\vw_y^\top \vh-b_y+\gamma\sum_{j\neq y}(\vw_j^\top\vh+b_j),
\end{equation}
where $\gamma>0$ is the trade-off parameter and $y$ denotes the class label of the feature $\vh$. When $\gamma=\frac{1}{C-1}$, the loss $\frac{1}{C-1}\sum_{j\neq y}[\vw_y^\top\vh+b_y-(\vw_j^\top\vh+b_j)]$ can be regarded as the unhinged version of the hinge loss\footnote{{The multi-class hinge loss is formulated as $L_{\text{hinge}}(\vs,y;m)=\sum_{i\neq y}\max\{0, s_i-s_y+m\}$, where $m$ is the margin term.}}  that removes the maximum operator and the margin term. We also note that the sample margin $m(\vh,y)=\vw_y^\top\vh+b_y-\max_{j\neq y}(\vw_j^\top\vh+b_j)$ \citep{koltchinskii2002empirical, cao2019learning, zhou2022learning} is defined to measure the discriminativeness for a sample, which satisfies $m(\vh,y)\le \frac{1}{C-1}\sum_{j\neq y}[\vw_y^\top\vh+b_y-(\vw_j^\top\vh+b_j)]$, \textit{i.e.},  $L_{\gamma}(\mW\vh+\vb,y)$ with $\gamma=\frac{1}{C-1}$ averaging the margins over all non-target classes is the lower bound of $-m(\vh,y)$. Here, we replace $\frac{1}{C-1}$ with an additional parameter $\gamma$ that balances positive and negative logits to draw general conclusions. More clarification about the unhinged loss can be found in Appendix \ref{clarification-of-unhinged-loss}.

Intuitively, the {unhinged loss} in \cref{the-unhinged-loss} promotes the learned feature $\vh$ to increase the logit of the target class while decreasing the logits of the other classes. If we follow up the layer-peeled model \citep{fang2021exploring} to restrict the norms of both features and prototypes, the global minimizer of $\frac{1}{CN}\sum_{i=1}^N\sum_{c=1}^C L_{\gamma}(\mW\vh_{i,c},y_{i,c})$ (the bias term $\vb$ is omitted) will lead to \textit{Neural Collapse} \citep{papyan2020prevalence, han2022neural}:

\begin{lemma}[\textbf{Neural Collapse under the Unhinged Loss}]
\label{neural-collapse-of-unhinged} 
For norm-bounded prototypes and features, \textit{i.e.}, $\|\vw_c\|_2\le E_1$ and $\|\vh_{i,c}\|_2\le E_2$, $\forall i\in[N], \forall c\in[C]$, the global minimizer of $\frac{1}{CN}\sum_{i=1}^N\sum_{c=1}^C L_{\gamma}(\mW\vh_{i,c},y_{i,c})$ implies neural collapse when $p\ge C-1$. More specifically, the global minimizer is uniquely obtained at $\frac{\vw_i^\top\vw_j}{\|\vw_i
\|_2\|\vw_j\|_2}=-\frac{1}{C-1}$, $\forall i\neq j$, $\frac{\vw_{y_{i,c}}^\top\vh_{i,c}}{\|\vw_{y_{i,c}}\|_2\|\vh_{i,c}\|_2}=1$,  $\|\vw_c\|_2=E_1$, and $\|\vh_{i,c}\|_2=E_2$, $\forall i\in[N],$ $\forall c\in[C]$.
\end{lemma} 
This lemma shows that the Neural Collapse solution\footnote{The Neural Collapse solution exhibits three critical properties \citep{papyan2020prevalence}: (i) \textbf{Within-class variability collapse}, the within-class variation of features becomes negligible as these features collapse to their class means; (ii) \textbf{Convergence to self-duality}, the class means and classifier weights converge to each other, up to rescaling; (iii) \textbf{Convergence to a simplex equiangular tight frame (ETF)}, the vectors of class means converge to having equal length, forming equal-sized angles between any given pair, and being the maximally pairwise-distanced configuration. For instance, prototypes $\mW$ that forms a simplex ETF satisfies $\mW^\top\mW=\frac{C}{C-1}\mI-\frac{1}{C-1}\1\1^\top$ when $C<p$}. is the only global optimal solution to minimize $\frac{1}{CN}\sum_{i,c}L_{\gamma}(\mW\vh_{k,c},y_{i,c})$ in the norm-bounded case.
However, there exists an undesired direction to minimize the {unhinged loss} in unconstrained cases, since the norm of features and prototypes tends to grow to infinity. For example, we can directly scale up $\mW$ and $\vb$ to obtain a smaller loss if $L_{\gamma}(\mW\vh+\vb, y)<0$, which will happen analogously to CE \citep{liu2016large, wang2017normface, zhou2022learning}. In this paper, we will analytically characterize the direction in which features $\mH$ and prototypes $\mW$ diverge. Specifically, we show that the gradient flow or gradient descent with the unhinged loss will exhibit an implicit bias associated with the initialization of features and prototypes. We further investigate the prototype-anchored case, wherein class prototypes are anchored with some specific structures, and the unhinged loss can perform comparative to CE.

\section{Main Theoretical Results}
\label{main-results}
In this section, we begin with the commonly used assumption that treats last-layer features as free optimization variables. We then conduct a comprehensive analysis of the closed-form dynamics of last-layer features and prototypes under the unhinged loss in various scenarios, including unconstrained, regularized, \revise{prototype-anchored, spherical constrained, and NTK-invariant cases.} Additionally, we provide convergence analyses for each case, revealing a surprising result: all cases exhibit exponential convergence. \textbf{All proofs can be found in Appendix \ref{all-proofs}.}

\subsection{The Unconstrained Case}
We first consider the unconstrained case \citep{mixon2022neural, ji2022unconstrained} in which there is no constraint or regularization on features and prototypes, \textit{i.e.}, learning with the following objective
\begin{equation}
    \label{unconstrain-optimization}
    \gL= \frac{1}{CN}\sum_{i,c} \big[-\vw_{y_{i,c}}^\top\vh_{i,c}-b_{y_{i,c}}+\gamma\sum_{j\neq y_{i,c}}(\vw_j^\top \vh_{i,c}+b_j)\big],
\end{equation}
which can be reformulated as
\begin{equation}
\label{objective}
\gL =\Tr(\mY^\top\mW^\top\mH)+\tfrac{\gamma C-\gamma-1}{C}\1_C^\top\vb,
\end{equation}
where $\mH=[\vh_{1,1},\ldots,\vh_{1,C},\vh_{2,1},\ldots,\vh_{N,C}]\in\mathbb{R}^{p\times CN}$ is the matrix resulting from stacking together the feature vectors as columns, $\mY=\frac{1}{CN}(\gamma\1_C\1_{CN}^\top-(1+\gamma)(\mI_C\otimes \1^\top_{N}))$, $\mW=[\vw_1,\vw_2,\ldots,\vw_C]\in\mathbb{R}^{p\times C}$ is the matrix resulting from stacking class prototypes as columns, $\otimes$ denotes the Kronecker product, $\mI_C$ is the identity matrix,  $\1_C$, $\1_N$, and $\1_{CN}$ are the length-$C$, -$N$, and -$CN$ vectors of ones, respectively. For brevity, we represent the label set $\{y_{i,c}\}_{1\le i\le N, 1\le c\le C}$ as the columns of the matrix $\mI_C\otimes \1_N^\top$.

\paragraph{Remark} We follow the unconstrained features modeling perspective \citep{mixon2022neural, ji2022unconstrained} or the layer-peeled model \citep{fang2021exploring} that treats $\mH$ as a free optimization variable. Within this model, we analyze the continuous dynamics of features $\mH$, prototypes $\mW$ and biases $\vb$ with gradient flow where  \textit{time-of-training} is denoted by the variable $t$\footnote{Intuitively, we interpret $t=0$ as the initial state, that is $\mH(0)=\mH_0$, $\mW(0)=\mW_0$, and $\vb(0)=\vb_0$.}.

Taking the partial derivative with respect to $\mH$, $\mW$, and $\vb$, respectively, we have the following:
\begin{equation}
\begin{aligned}
    &\nabla_{\mH}\gL=\mW\mY,\\ 
    &\nabla_{\mW}\gL=\mH\mY^\top,\\
    &\nabla_{\vb}\gL=\tfrac{\gamma C-\gamma-1}{C}\1_C,
\end{aligned}
\end{equation}
and the corresponding learning dynamics with respect to the gradient flow is
\begin{equation}
    \label{gradient-flow}
\begin{aligned}
    &\mH'(t)=\eta_1(t) \mW(t)\mM,\\
    &\mW'(t)=\eta_2(t) \mH(t)\mM^\top, \\
    &\vb'(t)=-\eta_2(t)\tfrac{\gamma C-\gamma-1}{C}\1_C,
\end{aligned}
\end{equation}
where $\mM=-\mY=\tfrac{1}{CN}((1+\gamma)(\mI_C\otimes \1^\top_{N})-\gamma\1_C\1_{CN}^\top)$, $\eta_1(t)$ and $\eta_2(t)$ are the time-varying part\footnote{In this paper, we are interested in investigating the effect of changing the learning rate over time. Thus, $\eta(t)$ is not the complete rate, but rather the time varying part of the learning rate. Specifically, for the example gradient flow $x'_t=\frac{\mathrm{d} x_t}{\mathrm{d} t}=\eta(t) \nabla f(x_t)$ considered in this paper is a continuous-time approximation of gradient descent in the limit of an infinitesimally small time step $\Delta_t$. Specifically, the corresponding  of gradient descent in discrete time can be formulated as $x_{t+\Delta_t}=x_{t}-\Delta_t \cdot \eta(t)\nabla f(x_t)$, where $\Delta_t\cdot \eta(t)$ denotes the complete learning rate, and $\Delta_t$ is the time step.} of the learning rate of the features $\mH$ and weights $(\mW,\vb)$, respectively. 
The reason for introducing different learning rates is that the representation $\mH$ is a result of the interaction between a number of nonlinear layers rather than a completely free variable like network parameters. This implies that even if we use the same learning rate to optimize all network parameters, the feature $\mH$ assumed to be a free optimization variable is almost impossible to be optimized at this learning rate\footnote{In this paper, we mainly assume that $\eta_1(t_1)\eta_2(t_2)=\eta_1(t_2)\eta_2(t_1)$ for any pair of values of $t$, $t_1$ and $t_2$. This condition will be satisfied if and only if $\eta_1(t)$ is a scaled version of $\eta_2(t)$, \textit{i.e.}, $\eta_1(t)=s\cdot \eta_2(t)$}. Moreover, we consider time-varying rates $\eta_1=\eta_1(t)$ and $\eta_2=\eta_2(t)$ that are usually adopted in practical implementations, such as the cosine annealing decay \citep{loshchilov2017sgdr}. As can be seen, the dynamics of features $\mH$ and prototypes $\mW$ are independent of the bias term $\vb$, thus we can analyze the dynamics of $\mH$ and $\mW$ jointly, and analyze $\vb$ independently:

\begin{theorem}[\textbf{Dynamics of Features, Prototypes and Biases without Constraints}]
\label{closed-form-dynamics}
Consider the continual gradient flow in \Eqref{gradient-flow}, let $\mZ(t)=(\mH(t),\mW(t))$, if $\eta_1(t_1)\eta_2(t_2)=\eta_1(t_2)\eta_2(t_1)$ for any $t_1,t_2\ge0$, we have the following closed-form dynamics
\begin{equation}
    \label{closed-form-zt}
    \begin{aligned}
    \mZ(t)=&\Pi_1^+\mZ_0 \left(\alpha_1^+(t)\mC(t)+\beta_1^+(t)\mI_{C(N+1)}\right)\\
    +&\Pi_1^-\mZ_0 \left(\alpha_1^-(t)\mC(t)+\beta_1^-(t)\mI_{C(N+1)}\right)\\
    +&\Pi_2^+\mZ_0\left(\alpha_2^+(t)\mC(t)+\beta_2^+(t)\mI_{C(N+1)}\right)\\
    +&\Pi_2^-\mZ_0 \left(\alpha_2^-(t)\mC(t)+\beta_2^-(t)\mI_{C(N+1)}\right)+\Pi_3\mZ_0,
    \end{aligned}
\end{equation}
and 
\begin{equation}
    \vb(t)=\vb_0 +\frac{(1+\gamma-\gamma C)\zeta_2(t)}{C} \1_C,
\end{equation}
where $\alpha_1^\eps$, $\alpha_2^\eps$, $\beta_1^\eps$ and $\beta_2^\eps$ for $\eps\in\{\pm\}$ are the scalars that only depend on $C$, $N$, $\gamma$, $\eta_1(0)$ and $\eta_2(0)$ (where the detailed forms of these scalars can be seen in the appendix), $\mZ_0=(\mH_0,\mW_0)$, $\mC(t)=\left(\begin{smallmatrix}
    \zeta_1(t)\mI_{CN} & 0\\
    0 & \zeta_2(t)\mI_C\\
    \end{smallmatrix}\right)$, $\zeta_1(t)=\int_0^t\eta_1(\tau)\mathrm{d}\tau$, $\zeta_2(t)=\int_0^t\eta_2(\tau)\mathrm{d}\tau$, $\Pi_1^{\eps}$, $\Pi_2^{\eps}$ and $\Pi_3$ for $\eps\in\{\pm\}$ are orthogonal projection operators onto the following respective eigenspaces:
\begin{equation}
    \begin{aligned}
        & \gE_1^\eps:=\{(\mH,\mW):\mH=\eps\cdot\tfrac{1}{\sqrt{N}}(\mW\otimes \1_N^\top),\mW\1_C=0\},\\
        & \gE_2^\eps:=\{(\mH,\mW):\mH=\eps\cdot \tfrac{1}{\sqrt{N}}\vh\1^\top_{CN},\mW=\vh\1_C ^\top,\vh\in\mathbb{R}^p\},\\
        & \gE_3:=\{(\mH,\mW): \mH(\mI_C\otimes\1_N)=0,\mW=0\}.
    \end{aligned}
\end{equation}
\end{theorem}

\begin{figure}[!t]
    \centering
   \subfigure[Train Loss]{
    \label{fig:unconstrained-loss}
    \includegraphics[scale=0.8]{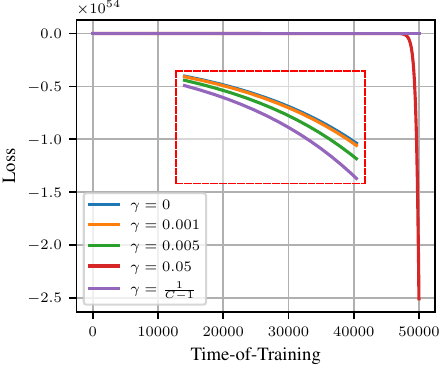}
   }
   \subfigure[Train Accuracy]{
    \label{fig:unconstrained-accuracy}
    \includegraphics[scale=0.8]{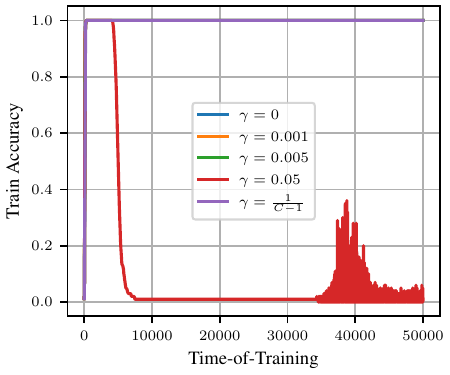}
   }
   \subfigure[$\|\hat{\mZ}(t)-\hat{\overline{\mZ}}\|_2$]{
    \label{fig:unconstrained-error-z}
    \includegraphics[scale=0.8]{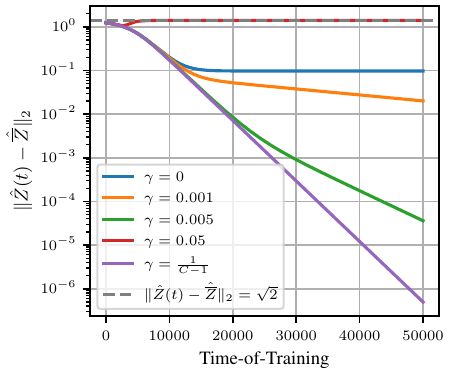}
   }
   \subfigure[$\|\mZ(t)\|_2$]{
    \label{fig:unconstrained-norm-z}
    \includegraphics[scale=0.8]{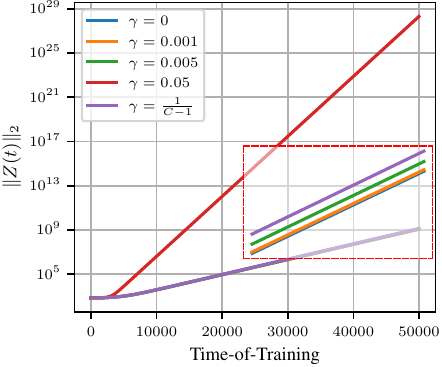}
   }
   \vskip-5pt
   \caption{Verification of the behavior of gradient descent iterates in \Cref{gradient-flow} with $\gamma\in\{0, 0.05, 0.005, 0.001, \frac{1}{C-1}\}$, where we set $p=512$, $C=100$, $N=10$, $\eta_1(t)=\eta_2(t)=0.1$ (\textit{i.e.}, $s=\frac{\eta_1(0)}{\eta_2(0)}=1$, thus $\overline{\mZ}=\Pi_1^+\mZ_0$ according to Corollary \ref{convergence-unconstrained-case}), and then randomly initialize $\mH_0$ and $\mW_0$.   {The red box in the figures represents the zoomed-in view of the last 2,000 epochs.} (a) The training loss. (b) The training accuracy with the prediction rule $\argmax_c \vw_c^\top\vh$. As expected, the features align to their corresponding prototypes when $\gamma<\frac{2}{C-2}$. (c) The distance between $\hat{\mZ}(t)=\mZ(t)/\|\mZ(t)\|_2$ and $\hat{\overline{\mZ}}=\overline{\mZ}/\|\overline{\mZ}\|_2$. As expected in \cref{convergence-rate-unconstrained-case}, the convergence rate is exponential when $0<\gamma<\frac{2}{C-2}$, and will be fastest if $\gamma=\frac{1}{C-1}$. (d) The norm of $\mZ(t)$, which increases exponentially. As can be noticed, $\hat{\mZ}(t)$ does not converge to $\hat{\overline{\mZ}}$ but tend to be orthogonal to $\hat{\overline{\mZ}}$ when $\gamma=0.05 >\frac{2}{C-2}$, that is, $\lim_{t\rightarrow\infty}\|\hat{\mZ}(t)-\hat{\overline{\mZ}}\|_2=\sqrt{2}$. Moreover, the exponential curve of $\|\boldsymbol Z(t)\|_2$ will be $O(e^{3.95 \times \sqrt{10}\times 10^{-4}t})$ when $\gamma=0.05$, which is faster than those curves of the rate roughly around $O(e^{\sqrt{10}\times 10^{-4}t})$ when $\gamma=0,0.001,0.005$ and $\frac{1}{C-1}=\frac{1}{99}$. This key distinction elucidates why the curves overlap when $\gamma=0, 0.001, 0.005$, and $\frac{1}{C-1}$.}
   \label{unconstrained-illustration}
   \vskip-5pt
\end{figure}

\paragraph{Remark} Note that $\gE^\eps_1$, $\gE^\eps_2$ and $\gE_3$ are orthogonal to each other,  $\gE^+_1$ (or $\gE^-_1$) denotes the subspace where all features are in the same (or opposite) direction of their corresponding prototypes while the mean of prototypes is zero, $\gE^\eps_2$ denotes the subspace where all features and prototypes collapse respectively into two scaled versions of the same unit vector, and $\gE_3$ denotes the subspace where the mean of all features from the same class is zero with all prototypes being zero. For classification tasks, we expect the features align to their corresponding prototypes with a cosine similarity of 1, \textit{i.e.}, the solution in $\gE=\{(\mH,\mW):\mH=k \mW\otimes\1_N^\top, \mW\in\sR^{p\times C}, k\in\sR^+\}$ that implies two manifestations of Neural Collapse: \textit{within-class variability collapse} and \textit{convergence to self-duality} \citep{papyan2020prevalence}. In the following, we prove that the normalized dynamics $\frac{\mZ(t)}{\|\mZ(t)\|}$ under the unhinged loss will converge to a solution in $\gE$:

\begin{corollary}[\textbf{Convergence in the Unconstrained Case}]
\label{convergence-unconstrained-case}
Under the conditions and notation of Theorem \ref{closed-form-dynamics}, let $s=\frac{\eta_1(0)}{\eta_2(0)}$. If $0<\gamma < \frac{2}{C-2}$ (where $C>2$) or $C=2$, and $\lim_{t\rightarrow\infty}\zeta_1(t)=\infty$, the gradient flow (as in \cref{gradient-flow}) will behave as:
\begin{equation}
    \label{dynamics-of-Z(t)}
    e^{-\frac{(1+\gamma)\sqrt{\zeta_1(t)\zeta_2(t)}}{C\sqrt{N}}}\mZ(t)=\overline{\mZ} +\bm{\Delta}(t),
\end{equation}
where $\overline{\mZ}=\left(\tfrac{1+\sqrt{s}}{2}\mH_1^++\tfrac{1-\sqrt{s}}{2}\mH_1^-, \tfrac{1+\sqrt{s}}{2\sqrt{s}}\mW_1^+-\tfrac{1-\sqrt{s}}{2\sqrt{s}}\mW_1^-\right)$,  $(\mH_1^+,\mW_1^+)=\Pi_1^+\mZ_0$, $(\mH_1^-,\mW_1^-)=\Pi_1^-\mZ_0$, and the residual term $\bm{\Delta}(t)$ decreases as $\|\bm{\Delta}(t)\|=O\left(e^{\frac{\sqrt{\zeta_1(t)\zeta_2(t)}}{C\sqrt{N}}\cdot\max\{-\gamma C,(C-2)\gamma-2\}}\right)$, and so the normalized $\mZ(t)$ converges to $\frac{\overline{\mZ}}{\|\overline{\mZ}\|}$ in
\begin{equation}
    \label{convergence-rate-unconstrained-case}
    \left\|\frac{\mZ(t)}{\|\mZ(t)\|}-\frac{\overline{\mZ}}{\|\overline{\mZ}\|}\right\|=O\left(e^{\frac{\sqrt{\zeta_1(t)\zeta_2(t)}}{C\sqrt{N}}\cdot\max\{-\gamma C,(C-2)\gamma-2\}}\right),
\end{equation}
which indicates $\lim_{t\rightarrow\infty}\frac{\mZ(t)}{\|\mZ(t)\|}=\frac{\overline{\mZ}}{\|\overline{\mZ}\|}\in\gE$.
Moreover, if $\gamma\neq \tfrac{1}{C-1}$, $\lim_{t\rightarrow \infty}\frac{\max_i b_i(t)}{\min_i b_i(t)}=1$.
\end{corollary}

The corollary above shows that even without any mandatory constraints, the gradient flow under the unhinged loss will converge to a solution $\frac{\overline{\mZ}}{\|\overline{\mZ}\|}$ that belongs to $\gE$ and can be determined by the initialization of $\mZ_0$ and the ratio $s$. It is worth noting, however, that this solution does not conform to the geometric structure of Neural Collapse. Neural collapse typically entails the formation of a simplex equiangular tight frame, as exemplified in the case of CE \citep{papyan2020prevalence}, MSE \citep{han2022neural}, margin-based losses \citep{zhou2022learning},  and other losses \citep{zhouall}. This result suggests that \textit{not all losses will lead to Neural Collapse solutions}. 

In addition, the rate of convergence is exponential as a function of the cumulative learning rates, \textit{i.e.}, $O\left(e^{\frac{\sqrt{\zeta_1(t)\zeta_2(t)}}{C\sqrt{N}}\cdot\max\{-\gamma C,(C-2)\gamma-2\}}\right)$, which indicates that the convergence of updating both features and prototypes by gradient descent is much faster than $O(1/\log t)$ that only updates prototypes (linear predictors) on linearly separable data \citep{soudry2018implicit}. In a sense, this convergence rate may help explain why training deep neural networks usually takes only several hundred or thousand epochs. Moreover, if $\gamma= \tfrac{1}{C-1}$, we can obtain the fastest convergence of \cref{convergence-rate-unconstrained-case}, that is, $\left\|\frac{\mZ(t)}{\|\mZ(t)\|}-\frac{\overline{\mZ}}{\|\overline{\mZ}\|}\right\|=O\left(e^{-\frac{\sqrt{\zeta_1(t)\zeta_2(t)}}{(C-1)\sqrt{N}}}\right)$. As shown in \cref{unconstrained-illustration}, if we set $\eta_1(t)=\eta$ as a constant learning rate, then $\zeta_1(t)=\eta t\rightarrow\infty$ as $t\rightarrow \infty$, and the gradient flow in \cref{gradient-flow} exhibits an exponential convergence rate of $\frac{\mZ(t)}{\|\mZ(t)\|}$ to $\frac{\overline{\mZ}}{\|\overline{\mZ}\|}$. However, this will lead to an exponential increase in the norm of features and prototypes, with the rate $e^{\frac{(1+\gamma)\sqrt{\zeta_1(t)\zeta_2(t)}}{C\sqrt{N}}}$. Such growth is almost unbearable for the practical training of DNNs. Therefore, in what follows, we consider to limit the excessive growth of these norms.

\subsection{The Regularized Case}
In this subsection, we focus on a regularized optimization problem that introduces $\ell_2$ regularization (also known as ``weight decay'') on features, prototypes, and biases. This regularization term helps to prevent the norms of these variables from growing too large during training. The optimization problem is defined as follows:
\begin{equation}
    \label{regularized-objective}
    \min_{\mW,\mH,\vb}\Tr(\mY^\top \mW^\top\mH)+\tfrac{\gamma C-\gamma-1}{C}\1_C^\top\vb + \tfrac{\lambda_1}{2}\|\mH\|_F^2+\tfrac{\lambda_2}{2}\|\mW\|_F^2+\tfrac{\lambda_2}{2}\|\vb\|_2^2,
\end{equation}
where $\|\cdot\|_F$ denotes the Frobenius norm, $\lambda_1$ and $\lambda_2$ are trade-off parameters\footnote{To draw general conclusions, we consider a different regularization strength for each component.} that control the regularization strength.
 
Taking the partial derivative with respect to $\mH$, $\mW$, and $\vb$, we have
\begin{equation}
    \notag
    \begin{aligned}
    &\nabla_{\mH}\gL=\mW\mY+\lambda_1\mH,\\ 
    &\nabla_{\mW}\gL=\mH\mY^\top+\lambda_2\mW,\\
    &\nabla_{\vb}\gL=\tfrac{\gamma C-\gamma-1}{C}\1_C+\lambda_2\vb,
    \end{aligned}
\end{equation}
and the corresponding learning dynamics following the gradient flow with time-vary learning rates, $\ell_2$ regularization can be formulated as
\begin{equation}
    \label{regularized-gradient-flow}
    \begin{aligned}
    &\mH'(t)=\eta_1(t) \mW(t)\mM-\lambda_1\eta_1(t)\mH(t),\\
    &\mW'(t)=\eta_2(t)\mH(t)\mM^\top-\lambda_2\eta_2(t)\mW(t), \\
    &\vb'(t)=-\eta_2(t)\tfrac{\gamma C-\gamma-1}{C}\1_C-\lambda_2\eta_2(t)\vb(t).
    \end{aligned}
\end{equation}

The dynamics of this regularized gradient flow can be proved as follows

\begin{theorem}[\textbf{Dynamics of Features, Prototypes, and Biases under Weight Decay}]
\label{closed-form-dynamics-l2}
Consider the continual gradient flow in \Eqref{regularized-gradient-flow}, let $\mZ(t)=(\mH(t),\mW(t))$. If $\eta_1(t_1)\eta_2(t_2)=\eta_1(t_2)\eta_2(t_1)$ for any $t_1,t_2\ge0$, we have the following closed-form dynamics:
\begin{equation}
    \begin{aligned}
    \mZ(t)=&\Pi_1^+\mZ_0\begin{pmatrix}a_1^+(t)\mI_{CN} & 0\\0 & b_1^+(t)\mI_{C}\end{pmatrix}+\Pi_1^-\mZ_0\begin{pmatrix}a_1^-(t)\mI_{CN} & 0\\0 & b_1^-(t)\mI_{C}\end{pmatrix}\\
    &+\Pi_2^+\mZ_0\begin{pmatrix}a_2^+(t)\mI_{CN} & 0\\0 & b_2^+(t)\mI_{C}\end{pmatrix}+\Pi_2^-\mZ_0\begin{pmatrix}a_2^-(t)\mI_{CN} & 0\\0 & b_2^-(t)\mI_{C}\end{pmatrix}\\
    &+\Pi_3\mZ_0\begin{pmatrix}a_3(t)\mI_{CN} & 0\\0 & b_3(t)\mI_{C}\end{pmatrix},
    \end{aligned}
\end{equation}
and 
\begin{equation}
\vb(t)=\phi(t)\left(\vb_0+\tfrac{1+\gamma-\gamma C}{C}\psi(t)\1_C\right),
\end{equation}
where $\Pi_1^+\mZ_0$, $\Pi_1^-\mZ_0$, $\Pi_1^+\mZ_0$, $\Pi_1^-\mZ_0$, and $\Pi_3\mZ_0$ follow the definition in Theorem \ref{closed-form-dynamics}, $a_1^\eps$, $a_2^\eps$, $b_1^\eps$, $b_2^\eps$, $a_3$, and $b_3$ for $\eps\in\{\pm\}$ are the scalars that depend only on $C$, $N$, $\gamma$, $\lambda_1$, $\lambda_2$, $\eta_1$, and $\eta_2$ (where the detailed forms can be seen in \ref{all-proofs}), $\phi(t)=\exp(-\lambda_2\int_0^t\eta_2(\tau)\mathrm{d}\tau)$, and $\psi(t)=\int_0^t\zeta_2(\tau)\exp(\lambda_2\int_0^\tau\eta_2(s)\mathrm{d}s)\mathrm{d}\tau$.
\end{theorem}

The convergence under the regularized case can also be derived as:

\begin{corollary}[\textbf{Convergence in the $\ell_2$ Regularized Case}]
\label{convergence-of-dynamics-l2}
Under the conditions and notation of Theorem \ref{closed-form-dynamics-l2}, let $s=\frac{\eta_1(0)}{\eta_2(0)}$. If $0<\gamma<\frac{2}{C-2}$ (where $C>2$) or $C=2$, $\lambda_1=\lambda_2=\lambda$, and $\lim_{t\rightarrow\infty}\zeta_1(t)=\infty$, then there exist constants $\pi_h^+,\pi_h^-,\pi_w^+$, $\pi_w^-$, and $\omega$ only depending on $\lambda$, $\gamma$, $s$, $C$, and $N$, such that the gradient flow (as in \cref{regularized-gradient-flow}) behaves as:
\begin{equation}
\label{convergence-of-Z-l2}
\left\|\frac{\mZ(t)}{\|\mZ(t)\|}-\frac{\mZ_\pi}{\|\mZ_\pi\|}\right\|=O\left(e^{-\omega\zeta_2(t)}\right),
\end{equation}
where $(\mH_1^+,\mW_1^+)=\Pi_1^+\mZ_0$, $(\mH_1^-,\mW_1^-)=\Pi_1^-\mZ_0$, and $\mZ_\pi=(\pi_h^+\mH_1^++\pi_h^-\mH_1^-,\pi_w^+\mW_1^++\pi_w^-\mW_1^-)$.

Furthermore, we have the following convergence results for $\mZ(t)$:
\begin{itemize}
    \item If $\lambda > \frac{1+\gamma}{C\sqrt{N}}$, then $\lim_{t\rightarrow\infty}\|\mZ(t)\|=0$;
    \item If $\lambda =\frac{1+\gamma}{C\sqrt{N}}$, then $\lim_{t\rightarrow\infty}\mZ(t)=\left(\mH_1^+ + \frac{1-s}{1+s}\mH_1^-, \mW_1^+ - \frac{1-s}{1+s}\mW_1^-\right)$;
    \item If $\lambda < \frac{1+\gamma}{C\sqrt{N}}$, then $\lim_{t\rightarrow\infty}\|\mZ(t)\|=\infty$.
\end{itemize}
\end{corollary}

As can be seen, the features and classifier weights under the unhinged loss converge to the solution depending on the initialization with the form $\mZ_\pi(t)=(\pi_h^+ \mH_1^+ +\pi_h^- \mH_1^-, \pi_w^+ \mW_1^+ +\pi_w^- \mW_1^-)$, where $\mH_1^+=\frac{1}{\sqrt{N}}(\mW^+\otimes \1_N^\top)$ and $\mH_1^-=-\frac{1}{\sqrt{N}}(\mW^-\otimes \1_N^\top)$. Thus, the solution implies the property of within-class variability collapse. To full encompass all properties of Neural Collapse (when the class number and the feature dimensionality $p$ satisfies $C\le p-1$), it is essential for $\pi_w^+ \mW_1^+ +\pi_w^- \mW_1^-$ to form a Simplex ETF, and for $\frac{\pi_h^+}{\pi_w^+}=-\frac{\pi_h^-}{\pi_w^-}$ to hold, inducing the property of convergence to self-duality (i.e., the class means and linear classifiers converge to each other, up to rescaling). For example, when $\lambda=\frac{1+\gamma}{C\sqrt{N}}$, Corollary demonstrates that the solution satisfies the property of convergence to self-duality.

Moreover, the results in Corollary \ref{convergence-of-dynamics-l2} suggest that adding an appropriate weight decay on both features and prototypes can avoid impractical effects, since the norm of $\mZ(t)$ shrinking to 0 or diverging toward infinity will significantly affect the training of DNNs. Several recent works \citep{zhu2021geometric, pmlr-v162-zhou22c} described that the features are implicitly penalized, but this implicit penalization may be fragile when using the unhinged loss (as depicted in \cref{fig:collapse-of-unhinged}). Consequently, \revise{we emphasize the importance of \textit{adding explicit regularization to features, rather than relying solely on the implicit penalization attached by other components like batch normalization and weight decay}}. Explicit feature regularization also plays a role in mitigating minority collapse \citep{fang2021exploring}. Firstly, minority collapse leads to features of the minority classes approaching $0$, since the minimization of the objective pays too much emphasis on enlarging the feature norms of the majority classes. In this context, explicit feature regularization can effectively restrain the excessive growth of feature norms of the majority classes. Secondly, explicit feature regularization in some sense intuitively reduces the energy of features in the optimization program \citep{fang2021exploring}, thereby shrinking the feasible domain of features. This further mitigates the imbalance of feature norms between majority and minority classes. In a nutshell, we need to limit the growth of feature norms.

\subsection{The Prototype-anchored Case}
We further consider the prototype-anchored case in which the class prototypes $\mW$ are fixed\footnote{That is, $\mW$ is not updated, which can be done by simply setting $W.\text{require\_grad} = False$ in PyTorch} into some specific structures (\textit{e.g.}, a simplex ETF) during the training process and the features $\mH$ are with $\ell_2$ regularization. Then, the dynamics of $\mH$ in \cref{regularized-objective} will be formulated as the first-order non-homogeneous linear difference equation:
\begin{equation}
    \label{eq:pal}
    \mH'(t)=\eta(t)\mW\mM-\lambda \eta(t)\mH(t),
\end{equation}
and the solution to the linear difference equation can be easily derived as follows

\begin{theorem}
\label{PAL-dynamics}
Consider the continual gradient flow (\cref{eq:pal}) in which the prototypes $\mW$ are fixed, we have the closed-form dynamics:
\begin{equation}
    \mH(t)=e^{-\lambda\int_0^t\eta(\tau)\mathrm{d}\tau}\mH(0)+\frac{1-e^{-\lambda\int_0^t\eta(\tau)\mathrm{d}\tau}}{\lambda} \mW\mM,
\end{equation}
which further indicates that $\left\|\mH(t)-\frac{1}{\lambda}\mW\mM\right\|=O\left(e^{-\lambda\int_0^t\eta(\tau)\mathrm{d}\tau}\right)$ when $\lim_{t\rightarrow \infty}e^{-\lambda\int_0^t\eta(\tau)\mathrm{d}\tau}=0$.
\end{theorem}

When the time-varying learning rate satisfies that $\lim_{t\rightarrow \infty}e^{-\lambda\int_0^t\eta(\tau)\mathrm{d}\tau}=0$, then $\mH(t)$ converges to $\frac{1}{\lambda}\mW\mM$. That is, the unhinged loss in the prototype-anchored case coincides with a feature alignment task such that each feature $\vh_{i,c}$ in class $c$ aligns to $\frac{1}{\lambda CN}(\vw_c -\gamma\sum_{j\neq c}\vw_j)$. It's worth noting that when $\lambda=0$, the behavior of $\mH(t)$ can be described as $\mH(t)=\mH(0)+\int_0^t\eta(\tau) \mathrm{d}\tau \mW\mM$, which is predominantly influenced by the term $\mW\mM$ when the integral $\int_0^t\eta(\tau) d\tau$ becomes significantly large.

As depicted in \cref{classification-results}, prototype-anchored learning demonstrates its effectiveness in mitigating training instability by transforming the classification problem into a feature alignment task. In this context, the unhinged loss with PAL yields results that are not only comparable but, in some cases, even superior to those achieved with the Cross-Entropy (CE) loss. This underscores the practicality of the unhinged loss as a valuable loss function for training classification models.

\subsection{The Spherical Constrained Case}
\label{sec:spherical-case}
We consider another constrained case in which features are restricted on the $p$-sphere $\sS^{p-1}=\{\vx:\|\vx\|_2=1,\vx\in\sR^p\}$ by explicitly performing $\ell_2$ normalization\footnote{$\ell_2$ normalization can also prevent arithmetic overflow or underflow occurring in the training of DNNs.}, and we fix the prototypes\footnote{The relevant studies are still few and often require some strict assumptions since the learning dynamics is very complicated when $\vw$ participates the optimization process with both feature and prototypes normalization. In this paper, we are going to try a more concise theoretical analysis with fixed prototypes.} to satisfy $\mW\1_C=0$\footnote{This aims to simplify \cref{the-unhinged-loss} as the objective of feature alignment, that is, $L_{\gamma}(\mW\hat{\vh},y)=-\vw_y^\top\hat{\vh}+\gamma\sum_{j\neq y}\vw_j^\top\hat{\vh}=\frac{(1+\gamma)\|\vw\|_2}{2}(\|\hat{\vh}-\hat{\vw}_y\|_2^2-2)$, and the global minimum will be obtained at $\hat{\vh}=\hat{\vw}_y$.}, then the optimization problem in \cref{unconstrain-optimization} can be reformulated as
\begin{equation}
    \label{spherical-objective}
    \min_{\mH}-\frac{1+\gamma}{CN}\Tr((\mI_C\otimes \1_N)\mW^\top\hat{\mH}),
\end{equation}
where $\hat{\mH}=(\hat{\vh}_{1,1},...,\hat{\vh}_{c,N})$ and $\hat{\vh}=\frac{\vh}{\|\vh\|_2}$ denotes the $\ell_2$-normalized vector. 

Take the partial derivative with respect to $\mH$, then the discrete dynamical system based on gradient descent is formulated as
\begin{equation}
    \label{normalized-discrete-dynamics}
    \mH(t+1)=\mH(t)+\tfrac{(1+\gamma)\eta(t)}{CN}\left(\tfrac{\partial \hat{\mH}}{\partial\mH}\big|_{\mH=\mH(t)}\right)^\top\mW(\mI_C\otimes \1_N^\top),
\end{equation}
where $(\frac{\partial \hat{\mH}}{\partial \mH})^\top$ is a vector-wise operator, and for any vector $\vh_{i,c}$ in $\mH'\in\sR^{p\times CN}$, we have $(\frac{\partial \hat{\mH}}{\partial \mH})^\top \vh_{i,c}=\left(\frac{1}{\|\vh_{i,c}\|_2}(\mI_p-\hat{\vh}_{i,c}\hat{\vh}_{i,c}^\top)\vh_{i,c}'\right)$.

Despite the fact that the optimization objective in \cref{spherical-objective} does not show convexity, Lipschitzness, and $\beta$-smoothness on $\mH$ due to the $\ell_2$ normalization operator, we prove that the normalized features which obey the gradient descent iterates in \cref{normalized-discrete-dynamics} can still converge to their corresponding normalized prototypes, \textit{i.e.}, achieve the global minimum of \cref{spherical-objective}:
\begin{theorem}[\textbf{Convergence in the Spherical Constrained Case}]
\label{convergence-of-spherical-case}
Considering the discrete dynamics in \cref{normalized-discrete-dynamics}, if  $\hat{\vw}_c^\top\hat{\vh}_{i,c}(0)>-1$ for any $i\in[N]$ and $c\in[C]$, the learning rate $\eta(t)$ satisfies that $\frac{\eta(t)}{\|\vh_{i,c}(t)\|_2}$ is non-increasing,  $\frac{\eta(0)(1+\gamma)}{CN\|\vh_{i,c}(0)\|_2}\le \frac{1}{\|\vw_c\|_2}$,  $\lim_{t\rightarrow\infty}\frac{\eta(t+1)}{\eta(t)}=1$, and there exists a constant $\varepsilon>0$, \textit{s.t.}, $\eta(t)>\varepsilon$, then we have 
\begin{equation}
    \lim_{t\rightarrow\infty}\left\|\hat{\mH}(t)-\hat{\mW}(\mI_C\otimes \1_N^\top)\right\|=0,
\end{equation} 
and further if $\lim_{t\rightarrow\infty}\|\mH(t)\|<\infty$, then there exists a constant $\mu>0$, such that the error above shows exponential decrease:
\begin{equation}
     \left\|\hat{\mH}(t)-\hat{\mW}(\mI_C\otimes \1_N^\top)\right\|= O(e^{-\mu t}).
\end{equation}
Moreover, if $\hat{\vw}_c^\top\hat{\vh}_{i,c}(0)=-1$, then $\vh_{i,c}(t)=\vh_{i,c}(0)$.
\end{theorem}

\paragraph{Remark} \revise{As shown in the theorem above, we prove that gradient descent exerts a bias that steers the normalized features toward the global minimizer $\hat{\mW}(\mI_C\otimes\1_N^\top)$, achieving exponential convergence under certain favorable conditions. This global solution shows two properties of Neural Collapse: within-class variability collapse and convergence to self-duality. If $\hat{W}$ also forms a simplex ETF, the global solution will exhibit all properties of Neural Collapse. Additionally}, we establish that if the inner product between $\hat{\vw}_c$ and $\hat{\vh}_{i,c}(0)$ (i.e., the cosine similarity between $\vw_c$ and $\vh_{i,c}(0)$) is $-1$, then $\vh_{i,c}(t)=\vh_{i,c}(0)$, as the gradient induced by $\ell_2$ normalization will be 0. More details can be found in the proof of Theorem  \ref{convergence-of-spherical-case}. Therefore, we just analyze the case where the inner product $\hat{\vw}_c$ and $\hat{\vh}_{i,c}(0)$ is strictly greater than $-1$.

\subsection{The NTK-Invariant Case}
We further perform some analysis in the case where the neural tangent kernel $\nabla_{\bm{\Theta}}\mH^\top\nabla_{\bm{\Theta}}\mH$ \citep{jacot2018neural, chizat2019lazy, yang2020feature, yang2021tensor} is assumed to be invariant during training. We prove that the unhinged loss has more potential to derive exact dynamics. Specifically, consider the last-layer feature $\vh_{i,c}=\vf_{\bm{\Theta}}(\vx_{i,c})$ extracted from the example $\vx_{i,c}$. When DNNs are trained using gradient descent to minimize the composition $\mathcal{L}\circ \boldsymbol{W}\circ \boldsymbol f_{\boldsymbol \Theta}$, we have

\begin{equation}
\begin{aligned}
\boldsymbol{\Theta}(t+1)-\boldsymbol{\Theta}(t)&=-\eta \nabla_{\boldsymbol{\Theta}}\mathcal{L}\big|_{\bm{\Theta}=\boldsymbol{\Theta}(t)},\\
\boldsymbol{W}(t+1)-\boldsymbol{W}(t)&=-\eta\nabla_{\boldsymbol{W}}\mathcal{L}\big|_{\bm{W}=\boldsymbol{W}(t)}.
\end{aligned}
\end{equation}
According to the first-order Taylor expansion, we obtain
\begin{equation}
\boldsymbol H(t+1)-\boldsymbol H(t)\approx \nabla_{\boldsymbol \Theta} \boldsymbol H^\top [\boldsymbol{\Theta}(t+1)-\boldsymbol{\Theta}(t)].
\end{equation}
For the unhinged loss, we can obtain the following closed-form dynamics
\begin{theorem}
\label{dynamics-under-NTK}
\revise{Assume that the neural tangent kernel $\nabla_{\boldsymbol \Theta}\boldsymbol H ^\top\nabla_{\boldsymbol \Theta}\boldsymbol H$ remains invariant during iterations.} Let $\vz(t)$ denote the row-first vectorization of $\begin{pmatrix}\boldsymbol{H}(t) & 0\\0 & \boldsymbol{W}(t)\end{pmatrix}$, $\boldsymbol B=\begin{pmatrix}0 & \boldsymbol M^\top\\\boldsymbol M & 0\end{pmatrix}$, and $\boldsymbol A=\begin{pmatrix}0 & \nabla_{\boldsymbol \Theta}\boldsymbol H ^\top\nabla_{\boldsymbol \Theta}\boldsymbol H\\\mathbf I_C & 0\end{pmatrix}$. Considering the eigendecomposition $\boldsymbol A=\boldsymbol U_{\boldsymbol A}\boldsymbol \Lambda_{\boldsymbol A} \boldsymbol U_{\boldsymbol A}^{-1}$ and $\boldsymbol B=\boldsymbol U_{\boldsymbol B}\boldsymbol \Lambda_{\boldsymbol B} \boldsymbol U_{\boldsymbol B}^{-1}$, we have
\begin{equation}
\mC\boldsymbol z(t)=\exp[(\boldsymbol \Lambda_{\boldsymbol A}\otimes \boldsymbol \Lambda_{\boldsymbol B})t]\mC\boldsymbol z(0),
\end{equation}
where $\mC=\boldsymbol U_{\boldsymbol B}^{-1}\boldsymbol U_{\boldsymbol A}^{-1}\otimes \mathbf I$.
\end{theorem}
Though Theorem \ref{dynamics-under-NTK} does not imply properties of Neural Collapse, it reveals a specific dynamics under the unhinged loss, which exhibits an exponential convergence rate, contingent upon the assumption that the neural tangent kernel remains constant. To demonstrate the mathematical simplicity of the unhinged loss, we consider the dynamics of the MSE loss $\mathcal L=\frac{1}{CN}\Vert\boldsymbol W^\top \boldsymbol H+\boldsymbol b \mathbf 1_{CN}^\top -\mathbf I\otimes \mathbf 1_N^\top \Vert_F^2$ for comparison, and similarly derive that
\begin{equation}
    \begin{aligned}
        \boldsymbol H'(t)&=\nabla_{\boldsymbol \Theta} \boldsymbol H^\top \nabla_{\boldsymbol \Theta} \boldsymbol H \boldsymbol W(t)\boldsymbol C(t),\quad
        \boldsymbol W'(t) &= \boldsymbol H(t) \boldsymbol C(t)^\top,\quad\boldsymbol{b}'(t)&=\boldsymbol C(t) \mathbf 1_{CN},
    \end{aligned}
\end{equation}
where $\boldsymbol C(t)=\frac{2\eta }{CN}(\boldsymbol W(t)^\top \boldsymbol H(t) + \boldsymbol b(t)\mathbf 1_{CN}^\top - \mathbf I_{C}\otimes \mathbf 1_N^\top)$. As can be seen, the above gradient flow is nonlinear and then intractable to analyze the exact dynamics though $\nabla_{\boldsymbol \Theta} \boldsymbol H^\top \nabla_{\boldsymbol \Theta} \boldsymbol H$ stays constant.

Through the above analysis, it is clear to see that the unhinged loss offers more mathematical opportunities than the MSE loss to analyze the closed-form dynamics under the assumption of an invariant neural tangent kernel.

\section{{Insights and Experiments}}
\label{insights}
In this section, we provide some insights into better training DNNs according to the conclusions in \cref{main-results}. We then corroborate our theoretical results and insights with extensive experiments. More specifically, in Section \ref{the-unhinged-loss-with-pal}, we propose to use prototype-anchored learning (PAL) as a means of resolving the instability issues that arise during training with the unhinged loss. In Section \ref{explicit-regularization}, we conduct experiments to highlight the benefits of explicit feature regularization on imbalanced learning and out-of-distribution (OOD) detection. In these experiments, we employ the unhinged loss with PAL and the CE loss as training objective. In Section \ref{rescaled-lr-for-the-spherical}, we propose the rescaling learning rates (RLR) with feature norms for the spherical case to address the problem of slow convergence resulting from feature normalization, which may also have implications for improving other methods of performing feature normalization. \textbf{More details and results can be found in Appendix \ref{appendix-experiments}.}

\subsection{The Unhinged Loss with Prototype-Anchored Learning}
\label{the-unhinged-loss-with-pal}
Since directly using the unhinged loss will lead to volatile effects, which is mainly reflected in the rapid increase of feature norms and the imbalance between class prototypes when training DNNs with the stochastic gradient method, as shown in \cref{fig:collapse-of-unhinged}. Inspired by recent works \citep{pmlr-v162-zhou22f, kasarla2022maximum, yang2022inducing} that use the Neural Collapse structure as an inductive bias (also called prototyping-anchored learning, PAL), we fix prototypes $\mW$ as a simplex ETF during training, \textit{i.e.}, $\mW^\top\mW=\frac{C}{C-1}\mI-\frac{1}{C-1}\1\1^\top$\footnote{These prototypes $\mW$ can be obtained using one of two methods: (i) By minimizing the objective $\sum_{i=1}^C\log \frac{\exp\left(s\hat{\vm}_i^\top\hat{\vh}_i\right)}{\sum_{j=1}^C\exp\left(s\hat{\vm}_j^\top\hat{\vh}_i\right)}$ and setting $\mW=\hat{\mM}$, where $s$ is a scale factor \citep{zhou2022learning}; or (ii) By deriving them from the standard simplex ETF $\mM=\sqrt{\frac{C}{C-1}}(\mI_C-\frac{1}{C}\1_C\1_C^\top)$, \textit{i.e.}, $\mW=s\mU\mM\in\sR^{p\times C}$, where $\mU\in\sR^{p\times C}$ ($p\ge C$) is a partial orthogonal matrix \citep{papyan2020prevalence}.}.

\paragraph{Experimental Results} We conduct experiments on widely-used classification datasets including CIFAR-10, CIFAR-100, and ImageNet-100. To mitigate training instability under the unhinged loss, we employ the Prototype-Anchored Learning (PAL) and Feature-Normalized PAL (FNPAL). As depicted in \cref{classification-results}, the results obtained by the unhinged loss with PAL and FNPAL variants demonstrate comparable or even better performance in comparison to CE. This substantiates the feasibility of utilizing the unhinged loss as a practical training objective for standard classification tasks.

\begin{figure}[t]
    \centering
    \subfigure[CIFAR-10]{
    \includegraphics[scale=0.55]{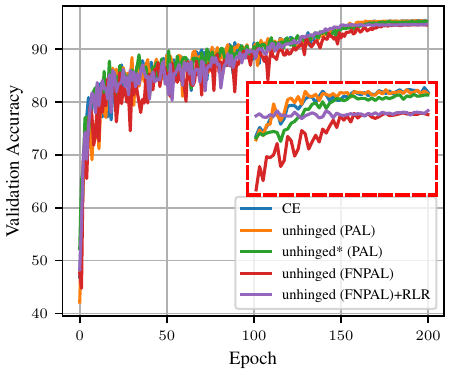}
    \label{fig:cifar-10-accuracy}
   }
   \subfigure[CIFAR-100]{
    \includegraphics[scale=0.55]{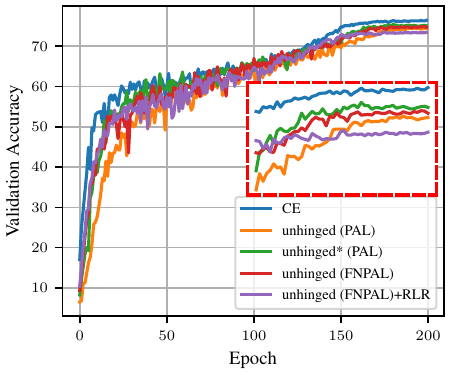}
    \label{fig:cifar-100-accuracy}
   }
   \subfigure[ImageNet-100]{
    \includegraphics[scale=0.55]{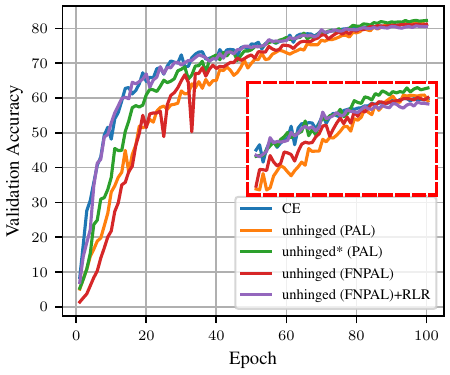}
    \label{fig:imagenet-100-accuracy}
   }
    \vskip-5pt
     \caption{Validation accuracy of different loss functions on CIFAR-10, CIFAR-100, and ImageNet, where $*$ denotes training with explicit feature regularization. PAL and FNPAL denote the model trained with prototype-anchored learning (PAL) and feature-normalized and prototype-anchored learning (FNPAL) \citep{pmlr-v162-zhou22f}. The curve in the red box represents the zoomed-in curve of the last 50 epochs. As can be seen, DNNs trained with the unhinged loss can achieve comparative or even better performance than those of CE.}
    \label{classification-results}
    \vskip-5pt
\end{figure}

\begin{table}[htbp]
    \centering
    \scriptsize
    \caption{Validation accuracies on long-tailed CIFAR-10/-100 with CE and different explicit feature regularization under Cross-Entropy (CE) and the unhinged loss (UL) with PAL. The imbalance ratio $\rho=\frac{\max_i n_i}{\min_i n_i}$ is the ratio between sample sizes of the most frequent and least frequent classes, and $\rho=1$ denotes the original CIFAR-10/-100. $\lambda=0$ denotes the model training with CE. All values are percentages. \textbf{Bold} numbers indicate the results that are better than the baseline vanilla CE or UL. The best results are \underline{underlined}.}
    \label{tab:lt-cifar}
    \begin{tabular}{c|c|c|c|c|c|c|c|c|c|c}
    \toprule
     \textbf{Dataset} & \multicolumn{5}{c|}{Long-tailed CIFAR-10} & \multicolumn{5}{c}{Long-tailed CIFAR-100}\\
     \midrule
     \textbf{Imbalance Ratio} &  100 & 50 & 20 & 10 & 1 & 100 & 50 & 20 & 10 & 1 \\
     \midrule
     Vanilla CE & 67.81 & 72.93 & 83.97 & 88.37 & 95.28 & 33.37 & 39.40 & 42.96 & 56.38 & 75.42\\
     CE ($\lambda=5e-6$) & \textbf{67.84} & 72.85 & 83.17 & \textbf{89.06} & 95.27 & \textbf{36.00} & \textbf{41.92} & \textbf{50.75} & \textbf{60.13} & \textbf{76.48} \\
     CE ($\lambda=1e-5$) & 67.74 & \textbf{76.14} & \textbf{84.17 }& \underline{\textbf{89.19}} & 95.23  & \underline{\textbf{36.61}} & \textbf{42.36} & \textbf{49.21} & \textbf{58.91} & \underline{\textbf{77.34}} \\
     CE ($\lambda=5e-5$) & \underline{\textbf{69.74}} & \underline{\textbf{77.29}} & \underline{\textbf{84.92}} & \textbf{88.64} & \underline{\textbf{95.39}} & \textbf{34.88} & \underline{\textbf{42.74}} & \underline{\textbf{54.72}} & \underline{\textbf{60.84}} & \textbf{76.19} \\
     \midrule
     Vanilla UL & 72.60 & 77.20 & 85.46 & 89.29 & 95.23 & 39.20 & 45.81 & 54.66 & 59.92 & 75.37 \\
     UL ($\lambda=1e-7$) & 71.85 & \textbf{78.80} & \textbf{85.76} & \textbf{89.68} & \textbf{95.33} & \textbf{41.39} & \textbf{46.13} & \textbf{55.25} & \textbf{61.15} & 75.18  \\
     UL ($\lambda=5e-7$) & \textbf{73.47} & \underline{\textbf{79.54}} & \textbf{86.28} & \textbf{89.36} & 95.14 & \textbf{41.85} & \textbf{47.58} & 54.00 & \textbf{61.31} & \textbf{75.51}  \\
     UL ($\lambda=1e-6$) & \underline{\textbf{73.71}} & \textbf{79.40} & \textbf{86.12} & \textbf{89.43} & \textbf{95.50} & \textbf{40.85} & \textbf{47.07} & \textbf{55.71} & \textbf{61.47} & \textbf{75.84}\\
     UL ($\lambda=5e-6$) & 71.13 & \textbf{78.90} & 85.33 & 89.20 & 95.10 & \textbf{41.87} & \textbf{45.83} & \textbf{55.92} & \textbf{60.87} & \textbf{77.20}\\
     \bottomrule
    \end{tabular}
\end{table}

\subsection{Explicit Feature Regularization}
\label{explicit-regularization}
In this paper, we directly consider explicit feature regularization to avoid excessive growth of feature norms, \textit{i.e.}, adding the regularization term $\lambda \sum_{\vx\in\gB}\|\vf_{\boldsymbol{\Theta}}(\vx)\|_2^2$ in the objective. Explicit regularization on features can significantly remedy over-confidence and even improve generalization. Moreover, adding explicit feature regularization can speed up the convergence of $\hat{\mH}(t)$ to $\hat{\mW\mM}$ according to Theorem \ref{PAL-dynamics}, as verified in \cref{pal-toy}.

\paragraph{Experimental Results} To validate the role of explicit feature regularization, we conduct experiments on two tasks: (i) Long-tailed recognition on benchmarks CIFAR-10-LT and CIFAR-100-LT with artificially created long-tailed settings; (ii) Out-of-distribution (OOD) detection between SVHN and CIFAR-10/-100. For long-tailed classification, we follow the controllable class imbalance strategy in \citep{cao2019learning} by reducing the number of training examples per class and keeping the validation set unchanged. As shown in Table \ref{tab:lt-cifar}, explicit feature regularization effectively improves the performance on long-tailed classification in most cases, even for normal classification. For ODD detection, we train ResNet-18 and ResNet-34 on in-distribution datasets CIFAR-10 and CIFAR-100, respectively, and then use SVHN as the OOD dataset to evaluate the performance.

\begin{figure}[t]
    \centering
    \subfigure[FPR95: 43.04]{
    \includegraphics[scale=0.42]{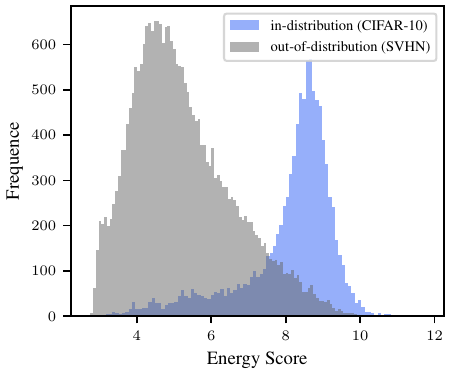}
    \label{fig:ood-cifar-10-0.0-energy}
    }
    \subfigure[FPR95: 27.87]{
    \includegraphics[scale=0.42]{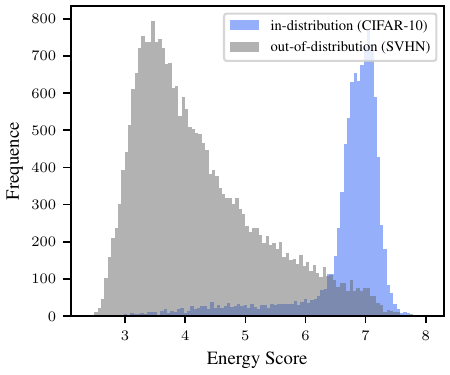}
    \label{fig:ood-cifar-10-1e-5-energy}
    }
    \subfigure[FPR95: 89.84]{
    \includegraphics[scale=0.42]{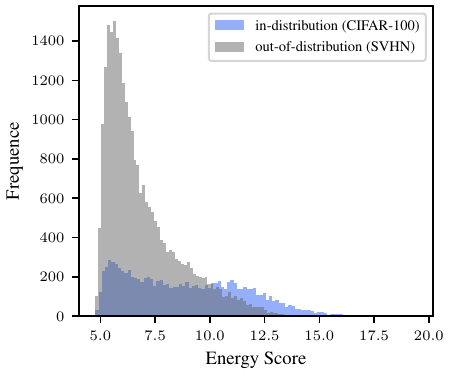}
    \label{fig:ood-cifar-100-0.0-energy}
    }
    \subfigure[FPR95: 81.41]{
    \includegraphics[scale=0.42]{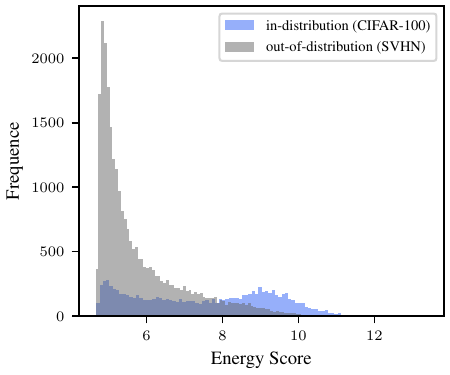}
    \label{fig:ood-cifar-100-1e-5-energy}
    }
    \subfigure[FPR95: 52.10]{
    \includegraphics[scale=0.42]{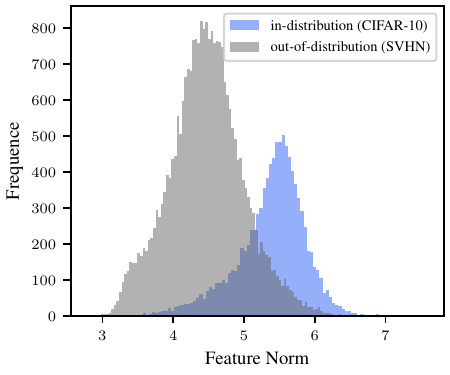}
    \label{fig:ood-cifar-10-0.0-norm}
    }
    \subfigure[FPR95: 24.94]{
    \includegraphics[scale=0.42]{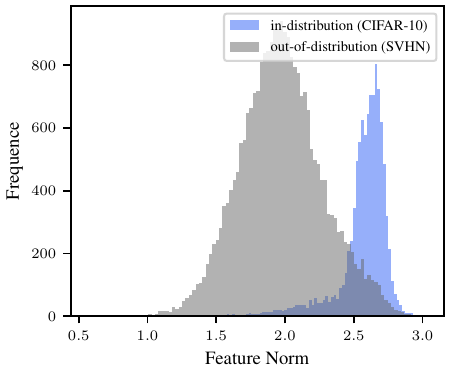}
    \label{fig:ood-cifar-10-1e-5-norm}
    }
    \subfigure[FPR95: 95.54]{
    \includegraphics[scale=0.42]{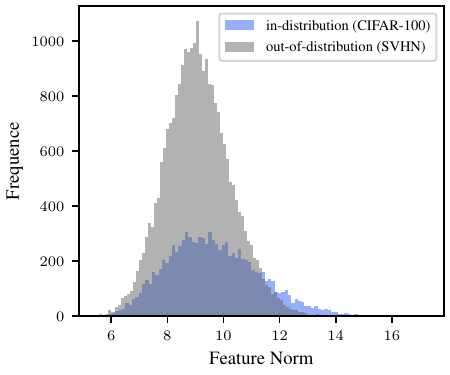}
    \label{fig:ood-cifar-100-0.0-norm}
    }
    \subfigure[FPR95: 87.83]{
    \includegraphics[scale=0.42]{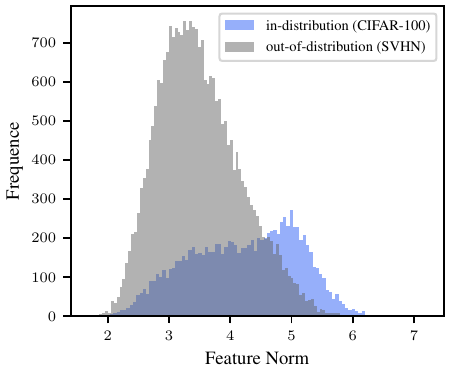}
    \label{fig:ood-cifar-100-1e-5-norm}
    }
    \vskip-5pt
    \caption{Distribution of energy scores (a-d) \citep{liu2020energy}  and feature norms (e-h) from classification models trained without (a \& c \& e \& g) or with (b \& d \& f \& h) explicit feature regularization (EFR) ($\lambda=1e-5$). (a \& b \& e \& f) and (c \& d \& g \& h) are from ResNet-18 \citep{he2016deep} trained on CIFAR-10 and from ResNet-34 trained on CIFAR-100, respectively. As can be seen, EFR can improve the performance of OOD detection by alleviating the over-confidence of OOD samples and making the energy scores of ID samples more concentrated. More intuitively, comparing (f) to (e) and (h) to (g), EFR effectively limits the growth of feature norms and significantly improves the distinction between ID samples and OOD samples in feature norm.}
   \label{fig:cifar10-ood-energy-score}
   \vskip-5pt
\end{figure}

\begin{table}[htbp]
    \scriptsize
    \centering
    \caption{OOD detection performance using softmax-based \citep{hendrycks2016baseline}, energy-based \citep{liu2020energy}, and feature norm-based approaches while model training with explicit feature regularization (EFL) (where $\lambda=1e-5$ for CE and $\lambda=5e-7$ for the unhinged loss (UL)). We use ResNet-18 and ResNet-34 to train on the in-distribution datasets CIFAR-10 and CIFAR-100, respectively. We then use SVHN \citep{Netzer2011reading} as the OOD dataset to evaluate the performance of OOD detection. All values are percentages. $\uparrow$ indicates large values are better, and $\downarrow$ indicates smaller values are better. The best results are \underline{underlined}.}
    \label{tab:ood-cifar-unhinged}
    \begin{tabular}{c|c|c|c|c}
    \toprule
    \textbf{Dataset} $\gD_{\text{in}}^{\text{test}}$ & \textbf{Method} & \textbf{FPR95} $\downarrow$  & \textbf{AUROC} $\uparrow$ & \textbf{AUPR} $\uparrow$ \\
    \midrule
    
    \multirow{5}*{CIFAR-10} & \multicolumn{4}{c}{Softmax-based / Energy-based / Feature Norm-based}\\
    ~ & CE & 52.09 / 43.04 / 52.10 & 91.67 / 91.94 / 89.54 & 84.11 / 82.80 / 77.06 \\
    ~ & CE + EFL & \textbf{37.39} / \textbf{27.87} / \underline{\textbf{24.94}} & \textbf{93.90} / {\textbf{94.60}} / \textbf{94.17} & {\textbf{85.48}} / \textbf{85.34} / \textbf{83.15} \\
    ~ & UL & 53.29  / 53.33 / 52.34 & 87.72 / 87.73 / 89.45 & 74.86 / 74.83 / 79.05 \\
    ~ & UL+EFL & \textbf{27.27}  / \textbf{27.24} / {\textbf{26.62}} & \textbf{94.27} / \textbf{94.46} / \underline{\textbf{95.12}} & \textbf{86.51} / \textbf{86.60} / \underline{\textbf{88.46}} \\
    \midrule
     \multirow{5}*{CIFAR-100} & \multicolumn{4}{c}{Softmax-based / Energy-based / Feature Norm-based}\\
    ~ & CE    & 87.75 / 89.84 / 95.54 & 71.01 / 71.94 / 59.54 & 55.42 / 56.69 / 43.21 \\
    ~ & CE+EFL & \textbf{81.48} / \textbf{81.41} / \textbf{87.83} & \textbf{77.02} / \underline{\textbf{78.03}} / \textbf{73.91} & \textbf{62.92} / \underline{\textbf{63.66}} / \textbf{58.81} \\
    ~ & UL & 80.43  / 76.86 / 76.23 & 66.97 / 74.83 / 76.19 & 37.83 / 45.07 / 47.51 \\
    ~ & UL+EFL & \textbf{75.74}  / \textbf{74.64} / \underline{\textbf{73.59}} & \textbf{75.21} / \textbf{75.36} / \textbf{76.80} & \textbf{45.16} / \textbf{45.27} / \textbf{48.34} \\
    \bottomrule
    \end{tabular}
\end{table}

\subsection{Rescaled Learning Rate for the Spherical Case}
\label{rescaled-lr-for-the-spherical}
In this subsection, we propose a strategy of rescaling the learning rate for the spherical case in which the features are performed $\ell_2$ normalization. 

As shown in Theorem \ref{convergence-of-spherical-case} and its proof, the convergence behavior in the spherical constrained case is characterized by the term $\frac{\eta(t)}{\|\vh_{i,c}(t)\|_2}$, when $\|\vh_{i,c}\|_2$ is monotonically increasing and thus leads to a slower convergence rate. This result underscores the influence of the feature norm on the gradient dynamics in spherical case.  Also, we observe that the training plots in the spherical case exhibit a slower rate in comparison to the CE loss in Figure 2. To bridge the performance of the unhinged loss to that of CE, we propose a straightforward approach---rescaling the learning rates based on feature norms, that is, scaling up the learning rate $\eta(t)$ with the feature norm $\|\vh_{i,c}\|_2$ for each example\footnote{We can also implement this strategy by rescaling the loss, but the multiplicative feature norm stops gradients.}. Though we change the learning rates, we still guarantee the convergence analysis in Theorem \ref{convergence-of-spherical-case} if $\eta(t)$ is non-increasing and satisfies $\frac{\eta(0)(1+\gamma)}{CN}\le \frac{1}{\|\vw_c\|_2}$.

\paragraph{Experimental Results} To confirm the efficacy of the rescaled learning rate in accelerating convergence for the spherical case, we experiment with the unhinged loss (FNPAL). As illustrated in Figure \ref{classification-results}, the utilization of the unhinged loss (FNPAL) with the rescaled learning rate (RLR) clearly demonstrates accelerated convergence in comparison to its counterpart without RLR.

\section{Conclusion}
In this paper, we introduced the unhinged loss as a surrogate to analyze the behavior of last-layer features and prototypes. Due to the conciseness of the unhinged loss, we derived exact dynamics under gradient descent in various scenarios, including unconstrained, regularized, and spherical constrained cases, as well as the case with an invariant neural tangent kernel. Furthermore, we demonstrated that these dynamics converge exponentially to a specific solution depending on the initialization. Inspired by these results, we proposed additional insights for improvements, such as using prototype-anchored learning with the unhinged loss to bridge its performance to that of cross-entropy loss, employing explicit feature regularization to address over-confidence, and implementing a rescaled learning rate to accelerate convergence in the spherical case. Finally, we validated these theoretical results and insights through extensive experiments, covering numerical analysis, visual classification, imbalanced learning, and out-of-distribution detection. We anticipate that the unhinged loss will serve as a valuable tool for the community to gain a deeper understanding of the behavior of deep neural networks, extending beyond the scope of this paper.

\section*{Acknowledgements}
This work was supported in part by National Natural Science Foundation of China under Grants 92270116 and 62071155, and in part by the Fundamental Research Funds for the Central Universities (Grant No.HIT.OCEF.2023043 and HIT.DZJJ.2023075). Xiong Zhou thanks the support from MEGVII Technology and Qiyuan Lab.

\newpage

\newpage

\begin{center}
    \Large \bf Appendix for ``On the Dynamics under the Unhinged Loss and Beyond''
\end{center}

\appendix

\section{Clarification for the Unhinged Loss}
\label{clarification-of-unhinged-loss}
In this section, we provide some clarification for the unhinged loss as follows.
\begin{itemize}
    \item \textbf{About the contribution.} As stated in the body of this paper, the unhinged loss serves as an alternative object of study, which has the advantage of being linear in the outputs of the model. Thus, it can provide more opportunities to analyze more concrete behaviors and explore the shortcomings of the existing modeling. The main contribution of this work is to glimpse into deep learning through a concise form under the unhinged loss, which enjoys the great advantage to simplify theoretical research very much. Moreover, it is not just a tool for theoretical analysis. To make the unhinged loss practical, we further provide several insights theoretically and empirically, including using prototype-anchored learning with the unhinged loss (Sec. \ref{the-unhinged-loss-with-pal}), applying explicit regularization on features (Sec. \ref{explicit-regularization}), rescaling the learning rates for the spherical constrained case (Sec. \ref{rescaled-lr-for-the-spherical}), and other potential insights (Sec. \ref{potential-insights}).
    \item \textbf{About the criterion of correct classification.} The value of the unhinged loss in \cref{the-unhinged-loss} cannot be directly used as the criterion to judge whether a sample is classified correctly since the unhinged loss can be arbitrary negative in unconstrained cases, since the gradient of the unhinged loss if $\frac{\partial L_{\gamma}}{\partial \vh}=-(\vw_y-\gamma\sum_{i\neq y}\vw_i)$ cannot converge to zero unless $\vw_y-\gamma\sum_{i\neq y}\vw_i=0$ for all $y$. This is also the main reason why the unhinged loss cannot be used directly for practical training. To overcome this drawback, we turn to limit the rapid growth of norms. For instance, in the spherical constrained case, the features are constrained to the unit sphere, and class prototypes are fixed as a simplex equiangular tight frame (ETF). Moreover, under this case, the unhinged loss would be lower bounded by $-(1+\gamma)$, and thus can be used in practical training (as shown in \cref{classification-results}). We can also derive that $L_{\gamma}<-(1+\gamma)\sqrt{\frac{C-2}{2(C-1}}\Rightarrow \vw_y^\top \vh>\max_{i\neq y}\vw_i^\top\vh$ when these prototypes are fixed as a simplex ETF, so the sample can be correctly classified if the value of the unhinged loss is less than $(1+\gamma)\sqrt{\frac{C-2}{2(C-1}}$.
    \item{\textbf{About other losses.}}
    
    1) For the CE and multi-binary CE loss, we have the following inequality.
    $$
    \begin{aligned}
    &L_{\gamma}(\mW\vh+\vb,y)\\
    \le&\min\bigg\{\log (1+\exp(-\vw_y^\top \vh-b_y))+\gamma \sum_{j\neq y}\log (1+\exp(\vw_j^\top \vh+b_j)),\\
    &(\gamma(C-1)-1)(\vw_y^\top+b_y) \vh-\gamma(C-1)\log \frac{\exp(\vw_y^\top \vh+b_y)}{\sum_{i=1}^C\exp(\vw_i^\top \vh+b_i)}\bigg\}.
    \end{aligned}
    $$
        
    2) For the Wasserstein loss in \citet{wgan}, we can find that the empirical form of the unhinged loss $\frac{1}{N}\sum_{i=1}^N(\vw_y^\top\vh+b_y-\gamma\sum_{i\neq y}(\vw_i^\top\vh+b_i))$ resembles the Wasserstein loss $\E_{x\sim P_r}[f_w(x)]-\E_{x\sim P_\theta}[f_w(x)]$ if we regard positive logits $\vw_y^\top\vh+b_y$ as scores from $P_r$ and negative logits $\vw_i^\top\vh+b_i$ ($i\neq y$) as scores from $P_{\theta}$. The main difference is that in the unhinged loss the two logits $\vw_y^\top\vh+b_y$ and $\vw_i^\top\vh+b_i$ ($i\neq y$) are computed from $\vh$ (\textit{i.e.}, the same input), rather than being sampled separably from two distributions as done in Wasserstein loss. Moreover, the unhinged loss does not constrain the values of these scores, in contrast the Wasserstein loss requires $f_w$ to be Lipschitz continuous.
\end{itemize}

\section{Proofs for Lemmas, Theorems, Propositions and Corollaries}
\label{all-proofs}

\subsection{Proof of Lemma \ref{neural-collapse-of-unhinged}}

\textbf{Lemma \ref{neural-collapse-of-unhinged} (Neural Collapse of The Unhinged Loss).}
\textit{
For norm-bounded prototypes and features, \textit{i.e.}, $\|\vw_c\|_2\le E_1$ and $\|\vh_{i,c}\|_2\le E_2$, $\forall i\in[N], \forall c\in[C]$, the global minimizer of $\frac{1}{CN}\sum_{i=1}^N\sum_{c=1}^C L_{\gamma}(\mW\vh_{i,c},y_{i,c})$ implies neural collapse when $p\ge C-1$. More specifically, the global minimizer is uniquely obtained at $\frac{\vw_i^\top\vw_j}{\|\vw_i
\|_2\|\vw_j\|_2}=-\frac{1}{C-1}$, $\forall i\neq j$, $\frac{\vw_{y_{i,c}}^\top\vh_{i,c}}{\|\vw_{y_{i,c}}\|_2\|\vh_{i,c}\|_2}=1$,  $\|\vw_c\|_2=E_1$, and $\|\vh_{i,c}\|_2=E_2$, $\forall i\in[N],$ $\forall c\in[C]$.
}
\begin{proof}
The proof is based on lower bounding the objective $\frac{1}{CN}\sum_{i=1}^N\sum_{c=1}^C L_{\gamma}(\mW\vh_{i,c},y_{i,c})$ by a sequence of inequalities that holds if and only if the solution forms Neural Collapse \citep{papyan2020prevalence}. Let $\hat \vw=\frac{1}{C}\sum_{c=1}^C\vw_c$, according to the definition of $L_{SM}$, we have
\begin{equation}
    \notag
    \begin{aligned}
    &\frac{1}{CN}\sum_{i=1}^N\sum_{c=1}^C L_{SM}(\mW\vh_{i,c},y_{i,c})\\
    =& \frac{1}{CN}\sum_{i=1}^N\sum_{c=1}^C(-\vw_{y_{i,c}}^\top\vh_{i,c}+\gamma\sum_{j\neq y_{i,c}}\vw_j^\top\vh_{i,c})\\
    \ge& \frac{1}{CN}\sum_{i=1}^N\sum_{c=1}^C(-E_1E_2+\gamma(C\hat{\vw}-\vw_{y_{i,c}})^\top\vh_{i,c})\\
    \ge& -\frac{\gamma E_2}{CN}\sum_{i=1}^N\sum_{c=1}^C\|C\hat{\vw}-\vw_{y_{i,c}}\|_2-E_1 E_2\\
    \ge& -\gamma E_2\sqrt{\frac{1}{CN}\sum_{i=1}^N\sum_{c=1}^C\|C\hat{\vw}-\vw_{y_{i,c}}\|_2^2}-E_1 E_2\\
    =&-\gamma E_2\sqrt{\frac{1}{C}\sum_{c=1}^C \|\vw_{c}\|_2^2-C^2\|\hat{\vw}\|_2^2}-E_1E_2\\
    \ge& -(1+\gamma) E_1 E_2
    \end{aligned},
\end{equation}
where  the first and second inequalities are based on the facts that $\vw_{y_{i,c}}^\top\vh_{i,c}\le E_1E_2$ and $(C\hat{\vw}-\vw_{y_{i,c}})^\top\vh_{i,c}\ge -E_2\|C\hat{\vw}-\vw_{y_{i,c}}\|_2$, respectively. In the third equality, we used the Cauchy-Schwarz inequality, and the last inequality we use the facts that $\|\vw_c\|_2\le E_1$ and $\|\hat{\vw}\|_2\ge 0$.

According the above derivation, the equality holds if and only if $\forall i\in [N]$, $\forall c\in [C]$, $\vw_{y_{i,c}}^\top\vh_{i,c}= E_1E_2$, $(C\hat{\vw}-\vw_{y_{i,c}})^\top\vh_{i,c}= -E_2\|C\hat{\vw}-\vw_{y_{i,c}}\|_2$, $\|C\hat{\vw}-\vw_{c}\|_2=\|C\hat{\vw}-\vw_{C}\|_2$, $\|\vw_c\|_2= E_1$, and $\|\hat{\vw}\|_2= 0$. These equations can be simplified as $\frac{\vw_i^\top\vw_j}{\|\vw_i
|_2\|\vw_j\|_2}=-\frac{1}{C-1}$, $\forall i\neq j$, $\frac{\vw_{y_{i,c}}^\top\vh_{i,c}}{\|\vw_{y_{i,c}}\|_2\|\vh_{i,c}\|_2}=1$,  $\|\vw_c\|_2=E_1$, and $\|\vh_{i,c}\|_2=E_2$, $\forall i\in[N],$ $\forall c\in[C]$, which also implies neural collapse.
\end{proof}

\subsection{Proof of Theorem \ref{closed-form-dynamics}}
In this section, we will provide the proof of Theorem \ref{closed-form-dynamics}. Our analysis will actually rely on the eigenvalues and eigenspaces of five subspaces $\gE_1^+$, $\gE_1^-$, $\gE_2^+$, $\gE_2^-$ and $\gE_3$ in Theorem \ref{closed-form-dynamics}. Their concrete projection operator can be found in \cref{projection-of-subspaces}. In the following, we show that these five subspaces are orthogonal:
\begin{lemma}
\label{five-orthogonal-subspaces}
The following five subspaces are orthogonal to each other and satisfy $\sR^{p\times C(N+1)}=\gE_1^+\oplus\gE_1^-\oplus\gE_2^+\oplus\gE_2^-\oplus \gE_3$:
\begin{equation}
   \begin{aligned}
        & \gE_1^\eps:=\{(\mH,\mW):\mH=\eps\cdot\tfrac{1}{\sqrt{N}}(\mW\otimes \1_N^\top),\mW\1_C=0, \mW\in\sR^{p\times C}\},\\
        & \gE_2^\eps:=\{(\mH,\mW):\mH=\eps\cdot \tfrac{1}{\sqrt{N}}\vh\1^\top_{CN},\mW=\vh\1_C ^\top,\vh\in\mathbb{R}^p\},\\
        & \gE_3:=\{(\mH,\mW): \mH(\mI_C\otimes\1_N)=0,\mW=0, \mH\in\sR^{p\times CN}\}.
    \end{aligned}
\end{equation}
where $\eps\in\{\pm 1\}$, and $k\neq 0$.
\begin{proof}
For $(\mH_1,\mW_1)=(\frac{1}{\sqrt{N}}(\mW_1\otimes \1_N^\top), \mW_1)\in\gE_1^+$ and $(\mH_2,\mW_2)=(-\frac{1}{\sqrt{N}}(\mW_2\otimes \1_N^\top), \mW_2)\in\gE_1^-$, we have
\begin{equation}
    \notag
    \mH_1\mH_2^\top + \mW_1\mW_2^\top=-\frac{1}{N}(\mW_1\otimes \1_N^\top)(\mW_2^\top\otimes \1_N) + \mW_1\mW_2^\top=0.
\end{equation}

For $(\mH_1,\mW_1)=(\frac{1}{\sqrt{N}}\vh_1\1_{CN}^\top, \vh_1\1_C^\top)\in\gE_2^+$ and $(\mH_2,\mW_2)=(-\frac{1}{\sqrt{N}}\vh_2\1_{CN}^\top, \vh_2\1_C^\top)\in\gE_2^-$, we have
\begin{equation}
    \notag
    \mH_1\mH_2^\top + \mW_1\mW_2^\top=-\frac{1}{N}\vh_1\1_{CN}^\top \1_{CN}\vh_2^\top + \vh_1\1_{C}^\top \1_{C}\vh_2^\top=0.
\end{equation}

For $(\mH_1,\mW_1)=(\eps_1\frac{1}{\sqrt{N}}(\mW_1\otimes \1_N^\top), \mW_1)\in\gE_1^\eps$ and $(\mH_2,\mW_2)=(\eps_2\frac{1}{\sqrt{N}}\vh_2\1_{CN}^\top, \vh_2\1_C^\top)\in\gE_2^\eps$, since $\mW_1\1_C=0$ and $(\mW_1\otimes\1_N^\top)\1_{CN}=N\mW_1\1_C=0$, we have
\begin{equation}
    \notag
    \mH_1\mH_2^\top + \mW_1\mW_2^\top=\frac{\eps_1\eps_2}{N}(\mW_1\otimes\1_N^\top)\1_{CN}\vh_2^\top+\mW_1\1_C\vh_2^\top =0.
\end{equation}

For $(\mH_1,\mW_1)=(\eps\frac{1}{\sqrt{N}}(\mW_1\otimes \1_N^\top), \mW_1)\in\gE_1^\eps$ and $(\mH_2,\mW_2)=(\mH_2,0)\in\gE_3$, we have
\begin{equation}
    \notag
    \mH_1\mH_2^\top + \mW_1\mW_2^\top=\frac{\eps}{\sqrt{N}}(\mW_1\otimes \1_N^\top)\mH_2^\top=\frac{\eps}{\sqrt{N}}\mW_1(\mH_2(\mI_C\otimes\1_N))^\top = 0.
\end{equation}

For $(\mH_1,\mW_1)=(\eps\frac{1}{\sqrt{N}}\vh_1\1_{CN}^\top,\vh_1\1_C^\top)\in\gE_2^\eps$ and $(\mH_2,\mW_2)=(\mH_2,0)\in\gE_3$, since $\mH_2(\mI_C\otimes\1_N)=0$, we have
\begin{equation}
    \notag
    \mH_1\mH_2^\top + \mW_1\mW_2^\top=\frac{\eps}{\sqrt{N}}\vh_1\1_{CN}^\top\mH_2^\top=\frac{\eps}{\sqrt{N}}\vh_1(\mH_2(\mI_C\otimes\1_N)
    \1_C^\top)^\top = 0
\end{equation}

To sum up, we prove that the five subspaces $\gE_1^+,\gE_1^-,\gE_2^+,\gE_2^-,\gE_3$ are orthogonal to each other. Moreover, we have
$$
\mathop{\text{dim}} \gE_1^+=\mathop{\text{dim}} \gE_1^-=p(C-1), \quad \mathop{\text{dim}}\gE_2^+= \mathop{\text{dim}}\gE_2^-=p,\quad \mathop{\text{dim}}\gE_3=pC(N-1).
$$
Since these dimensions sum to $pC(N+1)=\mathop{\text{dim}}(\mathbb{R}^{p\times CN}\oplus\mathbb{R}^{p\times C})$, then $\sR^{p\times C(N+1)}=\gE_1^+\oplus\gE_1^-\oplus\gE_2^+\oplus\gE_2^-\oplus \gE_3$.
\end{proof}

\end{lemma}

\noindent\textbf{Theorem \ref{closed-form-dynamics} (Dynamics of Features, Class Prototypes and Biases without Constraints).}
\textit{
Consider the continual gradient flow in \Eqref{gradient-flow}, let $\mZ(t)=(\mH(t),\mW(t))$, if $\eta_1(t_1)\eta_2(t_2)=\eta_1(t_2)\eta_2(t_1)$ for any $t_1,t_2\ge0$, we have the following closed-form dynamics
\begin{equation}
    \begin{aligned}
    \mZ(t)=&\Pi_1^+\mZ_0 \left(\alpha_1^+(t)\mC(t)+\beta_1^+(t)\mI_{C(N+1)}\right)\\
    +&\Pi_1^-\mZ_0 \left(\alpha_1^-(t)\mC(t)+\beta_1^-(t)\mI_{C(N+1)}\right)\\
    +&\Pi_2^+\mZ_0\left(\alpha_2^+(t)\mC(t)+\beta_2^+(t)\mI_{C(N+1)}\right)\\
    +&\Pi_2^-\mZ_0 \left(\alpha_2^-(t)\mC(t)+\beta_2^-(t)\mI_{C(N+1)}\right)+\Pi_3\mZ_0,
    \end{aligned}
\end{equation}
and 
\begin{equation}
    \vb(t)=\vb_0 +\frac{(1+\gamma-\gamma C)\zeta_2(t)}{C} \1_C,
\end{equation}
where $\alpha_1^\eps$, $\alpha_2^\eps$, $\beta_1^\eps$ and $\beta_2^\eps$ for $\eps\in\{\pm\}$ are the scalars that only depend on $C$, $N$, $\gamma$, $\eta_1$ and $\eta_2$ (where the detailed forms of these scalars can be seen in the appendix), $\mZ_0=(\mH_0,\mW_0)$, $\mC(t)=\left(\begin{smallmatrix}
    \zeta_1(t)\mI_{CN} & 0\\
    0 & \zeta_2(t)\mI_C\\
    \end{smallmatrix}\right)$, $\zeta_1(t)=\int_0^t\eta_1(\tau)\mathrm{d}\tau$, $\zeta_2(t)=\int_0^t\eta_2(\tau)\mathrm{d}\tau$, $\Pi_1^{\eps}$, $\Pi_2^{\eps}$ and $\Pi_3$ for $\eps\in\{\pm\}$ are orthogonal projection operators onto the following respective eigenspaces:
\begin{equation}
    \begin{aligned}
        & \gE_1^\eps:=\{(\mH,\mW):\mH=\eps\cdot\tfrac{1}{\sqrt{N}}(\mW\otimes \1_N^\top),\mW\1_C=0\},\\
        & \gE_2^\eps:=\{(\mH,\mW):\mH=\eps\cdot \tfrac{1}{\sqrt{N}}\vh\1^\top_{CN},\mW=\vh\1_C ^\top,\vh\in\mathbb{R}^p\},\\
        & \gE_3:=\{(\mH,\mW): \mH(\mI_C\otimes\1_N)=0,\mW=0\}.
    \end{aligned}
\end{equation}
}

\begin{proof}
\label{proof-of-closed-form-dynamics}
Writing $\mZ(t)=(\mH(t),\mW(t))$, then the unsolved portion of the system is given by
\begin{equation}
    \mZ'(t)=\mZ(t)\begin{pmatrix}0 & \mM^\top\\\mM & 0\\\end{pmatrix}\begin{pmatrix}
    \eta_1(t)\mI_{CN} & 0\\
    0 & \eta_2(t)\mI_C\\
    \end{pmatrix},
\end{equation}
where $\mM=\frac{1}{CN}((1+\gamma)(\mI_C\otimes \1^\top_{N})-\gamma\1_C\1_{CN}^\top)$.

Let $\mA(t)=\begin{pmatrix}0 & \mM^\top\\\mM & 0\\\end{pmatrix} \begin{pmatrix}\eta_1(t)\mI_{CN} & 0\\0 & \eta_2(t)\mI_C\\\end{pmatrix}=\begin{pmatrix}0 & \eta_2(t)\mM^\top\\\eta_1(t)\mM & 0\\\end{pmatrix}$, then the equation above can be reformulated as the initial-value problem associated with the linear ordinary differential equation:
\begin{equation}
    \mZ'(t)=\mZ(t)\mA(t),\quad \mZ(0)=\mZ_0.
\end{equation}

For any $t_1,t_2$, we have the matrix commutator of $\mA(t_1)$ and $\mA(t_2)$
$$
    \begin{aligned}
    &[\mA(t_1),\mA(t_2)]\\
    =&\mA(t_1)\mA(t_2)-\mA(t_2)\mA(t_1)\\
    =&\begin{pmatrix}0 & \eta_2(t_1)\mM^\top\\\eta_1(t_1)\mM & 0\end{pmatrix}\begin{pmatrix}0 & \eta_2(t_2)\mM^\top\\\eta_1(t_2)\mM & 0\end{pmatrix}-\mA(t_2)\mA(t_1)\\
    =&\begin{pmatrix}\eta_2(t_1)\eta_1(t_2)\mM^\top\mM & 0\\ 0 & \eta_1(t_1)\eta_2(t_2)\mM\mM^\top\end{pmatrix}-\mA(t_2)\mA(t_1)\\
    =&\begin{pmatrix}(\eta_2(t_1)\eta_1(t_2)-\eta_2(t_2)\eta_1(t_1))\mM^\top\mM & 0\\ 0 & (\eta_1(t_1)\eta_2(t_2)-\eta_2(t_1)\eta_1(t_2))\mM\mM^\top\end{pmatrix}\\
    =&0
    \end{aligned}
$$
where the last equality is based on the fact that $\eta_2(t_1)\eta_1(t_2)=\eta_2(t_2)\eta_1(t_1)$. Therefore, according to Magnus approach, we have
\begin{equation}
    \mZ(t)=\mZ_0\exp\left(\int_0^t \mA(\tau)\mathrm{d}\tau\right)=\mZ_0\exp\begin{pmatrix}
    0 & \int_0^t\eta_2(\tau) \mathrm{d}\tau\mM^\top\\
    \int_0^t\eta_1(\tau) \mathrm{d}\tau\mM & 0\\
    \end{pmatrix}.
\end{equation}
Let $\zeta_1(t)=\int_0^t\eta_1(\tau)\mathrm{d}\tau$, $\zeta_2(t)=\int_0^t\eta_2(\tau)\mathrm{d}\tau$, $\mB=\begin{pmatrix}
    0 & \mM^\top\\
    \mM & 0\\
    \end{pmatrix}$, $\mC(t)=\begin{pmatrix}
    \zeta_1(t)\mI_{CN} & 0\\
    0 & \zeta_2(t)\mI_C\\
    \end{pmatrix}$, and $\mL(t)=\mB\mC(t)$,  we have
\begin{equation}
    \mZ(t)=\mZ_0\exp(\mL(t))=\mZ_0\sum_{k=0}^\infty \frac{(\mL(t))^k}{k!}.
\end{equation}
Moreover, we have
\begin{equation}
    \begin{aligned}
    (\mL(t))^2&=\begin{pmatrix}
    0 & \zeta_2(t)\mM^\top\\
    \zeta_1(t)\mM & 0\\
    \end{pmatrix}\begin{pmatrix}
    0 & \zeta_2(t)\mM^\top\\
    \zeta_1(t)\mM & 0\\
    \end{pmatrix}\\
    &=\begin{pmatrix}
    \zeta_1(t)\zeta_2(t)\mM^\top\mM & 0\\
    0 & \zeta_1(t)\zeta_2(t)\mM\mM^\top\\
    \end{pmatrix}\\
    &=\zeta_1(t)\zeta_2(t)\mB^2,
    \end{aligned}
\end{equation}
thus we obtain
\begin{equation}
    \begin{aligned}
    \mZ(t)=&\mZ_0\left(\sum_{k=0}^\infty \frac{(\mL(t))^{2k+1}}{(2k+1)!}+\sum_{k=0}^\infty \frac{(\mL(t))^{2k}}{(2k)!}\right)\\
    =&\mZ_0\left(\sum_{k=0}^\infty \frac{(\zeta_1(t)\zeta_2(t))^k\mB^{2k+1}\mC(t)}{(2k+1)!}+\sum_{k=0}^\infty \frac{(\zeta_1(t)\zeta_2(t))^k\mB^{2k}}{(2k)!}\right)\\
    =&\mZ_0\sum_{k=0}^\infty \frac{(\zeta_1(t)\zeta_2(t))^k\mB^{2k+1}\mC(t)}{(2k+1)!}+\mZ_0\sum_{k=0}^\infty \frac{(\zeta_1(t)\zeta_2(t))^k\mB^{2k}}{(2k)!}
    \end{aligned}
\end{equation}
Looking at the above equation, we just need to analyze the eigenspaces and eigenvalues of $\mB$.

Considering the following five subspaces:
\begin{equation}
    \begin{aligned}
        & \gE_1^\eps:=\{(\mH,\mW):\mH=\eps\cdot\tfrac{1}{\sqrt{N}}(\mW\otimes \1_N^\top),\mW\1_C=0\},\\
        & \gE_2^\eps:=\{(\mH,\mW):\mH=\eps\cdot \tfrac{1}{\sqrt{N}}\vh\1^\top_{CN},\mW=\vh\1_C ^\top,\vh\in\mathbb{R}^p\},\\
        & \gE_3:=\{(\mH,\mW): \mH(\mI_C\otimes\1_N)=0,\mW=0\}.
    \end{aligned}
\end{equation}
where $\eps\in\{\pm\}$. According to Lemma \ref{five-orthogonal-subspaces}, these five subspaces are orthogonal to each other and satisfy $\sR^{p\times C(N+1)}=\gE_1^+\oplus\gE_1^-\oplus\gE_2^+\oplus\gE_2^-\oplus \gE_3$. 

In the following, we will prove that $\gE_{1}^\eps$, $\gE_{2}^\eps$ and $\gE_3$ are five eigenspaces of $\mB$. More specifically, each nonzero member of each claimed eigenspace is an eigenvector, and the claimed eigenspaces have distinct eigenvalues.

Note that for $(\mH,\mW)\in \sR^{p\times CN}\oplus\sR^{p\times C}$, we have $(\mH,\mW)\mB=(\mW\mM^\top,\mH\mW)$.

For $(\mH,\mW)\in\gE_1^\eps$, we have $\mH=\frac{\eps}{\sqrt{N}}\mW\otimes \1_N^\top$ and  $\mW\1_C=0$, thus
\begin{equation}
\notag
    \begin{aligned}
        \mW\mM&=\tfrac{1}{CN}\mW((1+\gamma)(\mI_C\otimes \1^\top_{N})-\gamma\1_C\1_{CN}^\top)\\
        &=\tfrac{(1+\gamma)}{CN}\mW\otimes \1_N^\top\\
        &=\tfrac{\eps(1+\gamma)}{C\sqrt{N}} \mH,\\
        \mH\mM^\top&=\tfrac{1}{CN}[\eps\cdot\tfrac{1}{\sqrt{N}}(\mW\otimes \1_N^\top)][(1+\gamma)(\mI_C\otimes\1^\top_{N})-\gamma\1_C\1_{CN}^\top]^\top\\
        &=\tfrac{\eps}{CN\sqrt{N}}[(\mW\otimes \1_N^\top)((1+\gamma)(\mI_C\otimes \1_N)-\gamma(\1_C\otimes \1_N)\1_C^\top)]\\
        &=\tfrac{\eps(1+\gamma)}{C\sqrt{N}} \mW,
\end{aligned}
\end{equation}
\textit{i.e.}, $(\mH,\mW)$ is an eigenvector of $\mB$ with eigenvalue $\tfrac{\eps(1+\gamma)}{C\sqrt{N}}$.

For $(\mH,\mW)\in\gE_2^\eps$, we have $\mH=\frac{\eps}{\sqrt{N}}\vh\1_{CN}^\top$ and $\mW=\vh\1_C^\top$, thus
\begin{equation}
\notag
\begin{aligned}
\mW\mM&=\tfrac{1}{CN}\vh\1_C^\top ((1+\gamma)(\mI_C\otimes\1^\top _{N})-\gamma\1_C\1_{CN}^\top )\\
&=\tfrac{1}{CN}\vh((1+\gamma)\1_C(\mI_C \otimes \1_N^\top )-\gamma C\1_{CN}^\top )\\
&=\tfrac{(1+\gamma-\gamma C)}{CN}\vh \1_{CN}^\top \\
&=\tfrac{\eps(1+\gamma-\gamma C)}{C\sqrt{N}} \mH,\\
\mH\mM^\top &=\tfrac{1}{CN\sqrt{N}}(\eps\cdot \vh \1_{CN}^\top )((1+\gamma)(\mI_C\otimes\1^\top _{N})-\gamma\1_C\1_{CN}^\top)^\top \\
&=\tfrac{\eps}{CN}\vh((1+\gamma)\1_{CN}^\top(\mI_C\otimes \1_N)-\gamma\1_{CN\sqrt{N}}^\top\1_{CN}\1_C^\top)\\
&=\tfrac{\eps}{CN\sqrt{N}}\vh((1+\gamma)N\1_{C}^\top-\gamma CN\1_C^\top) \\
&=\tfrac{\eps(1+\gamma-\gamma C)}{C\sqrt{N}}\vh\1_C ^\top \\
&=\tfrac{\eps(1+\gamma-\gamma C)}{C\sqrt{N}}\mW,
\end{aligned}
\end{equation}
\textit{i.e.}, $(\mH,\mW)$ is an eigenvector of $\mB$ with eigenvalue $\tfrac{\eps(1+\gamma-\gamma C)}{C\sqrt{N}}$.

For $(\mH,\mW)\in\gE_3$, we have $\mH(\mI_C\otimes\1_N)=0$ and $\mW=0$, thus
\begin{equation}
\notag
\begin{aligned}
\mW\mM&=\tfrac{1}{CN}\cdot0((1+\gamma)(\mI_C\otimes\1^\top _{N})-\gamma\1_C\1_{CN}^\top)=0,\\
\mH\mM^\top &=\tfrac{1}{CN}\mH((1+\gamma)(\mI_C\otimes\1^\top_{N})-\gamma\1_C\1_{CN}^\top )^\top \\
&=\tfrac{1}{CN}\mH((1+\gamma)(\mI_C\otimes\1_{N})-\gamma\1_{CN}\1_C^\top ) \\
&=-\tfrac{\gamma}{CN}\mH (\mI_C\otimes\1_{N})\1_C\1_C^\top \\
&=0
\end{aligned}
\end{equation}
\textit{i.e.}, $(\mH,\mW)$ is an eigenvector of $\mB$ with eigenvalue $0$.

Overall, letting $\Pi_i^\eps$ denote orthogonal projection onto $\gE_i^\eps$, we have the spectral decomposition
\begin{equation}
    \mB=\tfrac{1}{C\sqrt{N}}\left[(1+\gamma)(\Pi_1^+-\Pi_1^-)+(1+\gamma-\gamma C)(\Pi_2^+-\Pi_2^-)\right].
\end{equation}
We then provide the concrete formulation of $\mZ(t)=\mZ_0\exp(\mL(t))$ by the orthogonal projection of $\mZ_0$ onto each eigenspace of $\mB$, \textit{i.e.},
$$
\mZ_0=\Pi_1^+\mZ_0+\Pi_1^-\mZ_0+\Pi_2^+\mZ_0+\Pi_2^-\mZ_0+\Pi_3\mZ_0.
$$

\paragraph{Decomposition along $\Pi_1^\eps \mZ_0$.}
First, $\Pi_1^\eps \mZ_0 \mB=\frac{\eps(1+\gamma)}{C\sqrt{N}}\Pi_1^\eps\mZ_0$, so $ \Pi_1^\eps \mZ_0 \mB^k = \left(\frac{\eps(1+\gamma)}{C\sqrt{N}}\right)^k\Pi_1^\eps\mZ_0$ for $k\ge0$, then
\begin{equation}
    \begin{aligned}
    &\Pi_1^\eps\mZ_0\sum_{k=0}^\infty \frac{(\zeta_1(t)\zeta_2(t))^k\mB^{2k+1}\mC(t)}{(2k+1)!}\\
    =&\Pi_1^\eps\mZ_0\sum_{k=0}^\infty \frac{(\zeta_1(t)\zeta_2(t))^k\left(\frac{\eps(1+\gamma)}{C\sqrt{N}}\right)^{2k+1}\mC(t)}{(2k+1)!}\\
    =&\frac{\Pi_1^\eps\mZ_0 \mC(t)}{\sqrt{\zeta_1(t)\zeta_2(t)}}\sum_{k=0}^\infty \frac{\left(\frac{\eps(1+\gamma)\sqrt{\zeta_1(t)\zeta_2(t)}}{C\sqrt{N}}\right)^{2k+1}}{(2k+1)!}\\
    =&\frac{\Pi_1^\eps\mZ_0 \mC(t)}{2\sqrt{\zeta_1(t)\zeta_2(t)}}\left(e^{\frac{\eps(1+\gamma)\sqrt{\zeta_1(t)\zeta_2(t)}}{C\sqrt{N}}}-e^{-\frac{\eps(1+\gamma)\sqrt{\zeta_1(t)\zeta_2(t)}}{C\sqrt{N}}}\right),
    \end{aligned}
\end{equation}
and
\begin{equation}
    \begin{aligned}
        & \Pi_1^\eps\mZ_0\sum_{k=0}^\infty \frac{(\zeta_1(t)\zeta_2(t))^k\mB^{2k}}{(2k)!}\\
        =&\Pi_1^\eps\mZ_0\sum_{k=0}^\infty \frac{(\zeta_1(t)\zeta_2(t))^k\left(\frac{\eps(1+\gamma)}{C\sqrt{N}}\right)^{2k}}{(2k)!}\\
        =&\frac{\Pi_1^\eps\mZ_0}{2}\left(e^{\frac{\eps(1+\gamma)\sqrt{\zeta_1(t)\zeta_2(t)}}{C\sqrt{N}}}+e^{-\frac{\eps(1+\gamma)\sqrt{\zeta_1(t)\zeta_2(t)}}{C\sqrt{N}}}\right),
    \end{aligned}
\end{equation}
which is based on the facts that $\frac{e^x-e^{-x}}{2}=\sum_{k=0}^\infty \frac{x^{2k+1}}{(2k+1)!}$ and $\frac{e^x+e^{-x}}{2}=\sum_{k=0}^\infty \frac{x^{2k}}{(2k)!}$. Thus we have
\begin{equation}
    \Pi_1^\eps\mZ_0 \exp(\mL(t))=\Pi_1^\eps\mZ_0(\alpha_1^\eps\mC(t)+\beta_1^\eps\mI_{C(N+1)}),
\end{equation}
with
\begin{equation}
\begin{aligned}
&\alpha_1^\eps(t)=\frac{\exp\left(\frac{\eps(1+\gamma)\sqrt{\zeta_1(t)\zeta_2(t)}}{C\sqrt{N}}\right)-\exp\left(-\frac{\eps(1+\gamma)\sqrt{\zeta_1(t)\zeta_2(t)}}{C\sqrt{N}}\right)}{2\sqrt{\zeta_1(t)\zeta_2(t)}},\\ &\beta_1^\eps(t)=\frac{\exp\left(\frac{\eps(1+\gamma)\sqrt{\zeta_1(t)\zeta_2(t)}}{C\sqrt{N}}\right)+\exp\left(-\frac{\eps(1+\gamma)\sqrt{\zeta_1(t)\zeta_2(t)}}{C\sqrt{N}}\right)}{2}.
\end{aligned}
\end{equation}

\paragraph{Decomposition along $\Pi_2^\eps \mZ_0$.} Similarly, for $\Pi_2^\eps\mZ_0 \mB=\frac{\eps(1+\gamma-\gamma C)}{C\sqrt{N}}\Pi_2^\eps$, we have
$$
    \Pi_2^\eps\mZ_0\sum_{k=0}^\infty \frac{(\zeta_1(t)\zeta_2(t))^k\mB^{2k+1}\mC(t)}{(2k+1)!}=\frac{\Pi_2^\eps\mZ_0 \mC(t)}{2\sqrt{\zeta_1(t)\zeta_2(t)}}\left(e^{\frac{\eps(1+\gamma-\gamma C)\sqrt{\zeta_1(t)\zeta_2(t)}}{C\sqrt{N}}}-e^{-\frac{\eps(1+\gamma-\gamma C)\sqrt{\zeta_1(t)\zeta_2(t)}}{C\sqrt{N}}}\right),
$$
and 
$$
\Pi_2^\eps\mZ_0\sum_{k=0}^\infty \frac{(\zeta_1(t)\zeta_2(t))^k\mB^{2k}}{(2k)!}=\frac{\Pi_2^\eps\mZ_0}{2}\left(e^{\frac{\eps(1+\gamma-\gamma C)\sqrt{\zeta_1(t)\zeta_2(t)}}{C\sqrt{N}}}+e^{-\frac{\eps(1+\gamma-\gamma C)\sqrt{\zeta_1(t)\zeta_2(t)}}{C\sqrt{N}}}\right).
$$
Thus we have
\begin{equation}
    \Pi_2^\eps\mZ_0 \exp(\mL(t))=\Pi_2^\eps\mZ_0(\alpha_2^\eps\mC(t)+\beta_2^\eps\mI_{C(N+1)}),
\end{equation}
with 
\begin{equation}
\begin{aligned}
&\alpha_2^\eps(t)=\frac{\exp\left(\frac{\eps(1+\gamma-\gamma C)\sqrt{\zeta_1(t)\zeta_2(t)}}{C\sqrt{N}}\right)-\exp\left(-\frac{\eps(1+\gamma-\gamma C)\sqrt{\zeta_1(t)\zeta_2(t)}}{C\sqrt{N}}\right)}{2\sqrt{\zeta_1(t)\zeta_2(t)}},\\ 
&\beta_2^\eps(t)=\frac{\exp\left(\frac{\eps(1+\gamma-\gamma C)\sqrt{\zeta_1(t)\zeta_2(t)}}{C\sqrt{N}}\right)+\exp\left(-\frac{\eps(1+\gamma-\gamma C)\sqrt{\zeta_1(t)\zeta_2(t)}}{C\sqrt{N}}\right)}{2}.
\end{aligned}
\end{equation}

\paragraph{Decomposition along $\Pi_3 \mZ_0$.} Since each vector in $\gE_3$ is a eigenvector of $\mB$ with eigenvalue 0, then we have
\begin{equation}
    \Pi_3\mZ_0\left(\sum_{k=0}^\infty \frac{(\zeta_1(t)\zeta_2(t))^k\mB^{2k+1}\mC(t)}{(2k+1)!}+\sum_{k=0}^\infty \frac{(\zeta_1(t)\zeta_2(t))^k\mB^{2k}}{(2k)!}\right)=\Pi_3\mZ_0
\end{equation}

Note that $\gE_1^+$, $\gE_1^-$, $\gE_2^+$, $\gE_2^-$ and $\gE_3$ are orthogonal subspace of $\sR^{p\times CN}\oplus\sR^{C\times p}$, thus 
\begin{equation}
    \begin{aligned}
    \mZ(t)=&\mZ_0\sum_{k=0}^\infty \frac{(\zeta_1(t)\zeta_2(t))^k\mB^{2k+1}\mC(t)}{(2k+1)!}+\mZ_0\sum_{k=0}^\infty \frac{(\zeta_1(t)\zeta_2(t))^k\mB^{2k}}{(2k)!}\\
    =&(\Pi_1^+\mZ_0+\Pi_1^-\mZ_0+\Pi_2^+\mZ_0+\Pi_2^-\mZ_0+\Pi_3\mZ_0)\exp(\mL(t))\\
    =&\Pi_1^+\mZ_0 \left(\alpha_1^+(t)\mC(t)+\beta_1^+(t)\mI_{C(N+1)}\right)+\Pi_1^-\mZ_0 \left(\alpha_1^-(t)\mC(t)+\beta_1^-(t)\mI_{C(N+1)}\right)+\Pi_3\mZ_0\\
    &+\Pi_2^+\mZ_0\left(\alpha_2^+(t)\mC(t)+\beta_2^+(t)\mI_{C(N+1)}\right)+\Pi_2^-\mZ_0 \left(\alpha_2^-(t)\mC(t)+\beta_2^-(t)\mI_{C(N+1)}\right)
    \end{aligned}
\end{equation}

Moreover, since $\vb'(t)=-\eta_2(t)\tfrac{\gamma C-\gamma-1}{C}\1_C$, we obtain
\begin{equation}
     \vb(t)=\vb(0)+\int_0^t-\eta_2(\tau)\frac{\gamma C-\gamma-1}{C} \mathrm{d}\tau \1_C=\vb_0+\frac{(1+\gamma-\gamma C)\zeta_2(t)}{C} \1_C,
\end{equation}
with $\zeta_2(t)=\int_0^t\eta_2(\tau)\mathrm{d}\tau$.
\end{proof}

\subsection{Proof of Corollary \ref{convergence-unconstrained-case}}
\textbf{Corollary \ref{convergence-unconstrained-case}}
\textit{
Under the conditions and notation of Theorem \ref{closed-form-dynamics}, let $s=\frac{\eta_1(0)}{\eta_2(0)}$, if $0<\gamma < \frac{2}{C-2}$ (where $C>2$) or $C=2$, and $\lim_{t\rightarrow\infty}\zeta_1(t)=\infty$, then the gradient flow (as in \cref{gradient-flow}) will behave as:
\begin{equation}
    e^{-\frac{(1+\gamma)\sqrt{\zeta_1(t)\zeta_2(t)}}{C\sqrt{N}}}\mZ(t)=\overline{\mZ} +\bm{\Delta}(t),
\end{equation}
where $\overline{\mZ}=\left(\tfrac{1+\sqrt{s}}{2}\mH_1^++\tfrac{1-\sqrt{s}}{2}\mH_1^-, \tfrac{1+\sqrt{s}}{2\sqrt{s}}\mW_1^+-\tfrac{1-\sqrt{s}}{2\sqrt{s}}\mW_1^-\right)$,  $(\mH_1^+,\mW_1^+)=\Pi_1^+\mZ_0$, $(\mH_1^-,\mW_1^-)=\Pi_1^-\mZ_0$, and the residual term $\bm{\Delta}(t)$ decreases as $\|\bm{\Delta}(t)\|=O\left(e^{\frac{\sqrt{\zeta_1(t)\zeta_2(t)}}{C\sqrt{N}}\cdot\max\{-\gamma C,(C-2)\gamma-2\}}\right)$, and so the normalized $\mZ(t)$ converges to $\frac{\overline{\mZ}}{\|\overline{\mZ}\|}$ in
\begin{equation}
    \left\|\frac{\mZ(t)}{\|\mZ(t)\|}-\frac{\overline{\mZ}}{\|\overline{\mZ}\|}\right\|=O\left(e^{\frac{\sqrt{\zeta_1(t)\zeta_2(t)}}{C\sqrt{N}}\cdot\max\{-\gamma C,(C-2)\gamma-2\}}\right),
\end{equation}
which further indicates $\lim_{t\rightarrow\infty}\frac{\mZ(t)}{\|\mZ(t)\|}\in\gE$.
Moreover, if $\gamma\neq \tfrac{1}{C-1}$, then $\lim_{t\rightarrow \infty}\frac{\max_i b_i(t)}{\min_i b_i(t)}=1$.
}
\begin{proof}
Let $(\mH_1^\eps,\mW_1^\eps)=\Pi_1^\eps\mZ_0$, $(\mH_2^\eps,\mW_2^\eps)=\Pi_2^\eps\mZ_0$ and $(\mH_3,\mW_3)=\Pi_3\mZ_0$ for $\eps\in\{\pm1\}$, according to Theorem \ref{closed-form-dynamics}, we have
\begin{equation}
    \label{decomposition-H-W}
    \begin{aligned}
    &\mH(t)=\sum_{\substack{i\in\{1,2\}\\\eps\in\{\pm\}}} (\alpha_i^\eps(t)\zeta_1(t) + \beta_i^\eps(t))\mH_i^\eps + \mH_3,\\
    &\mW(t)=\sum_{\substack{i\in\{1,2\}\\\eps\in\{\pm\}}} (\alpha_i^\eps(t)\zeta_2(t) + \beta_i^\eps(t))\mW_i^\eps + \mW_3.
    \end{aligned}
\end{equation}
Since $\eta_1(t_1)\eta_2(t_2)=\eta_1(t_2)\eta_2(t_1)$, then $\eta_1(t)=\frac{\eta_1(0)}{\eta_2(0)}\eta_2(t)$. Let $s=\frac{\eta_1(0)}{\eta_2(0)}$, $p(t)=\frac{(1+\gamma)\sqrt{\zeta_1(t)\zeta_2(t)}}{C\sqrt{N}}$, and $q(t)=\frac{(1+\gamma-\gamma C)\sqrt{\zeta_1(t)\zeta_2(t)}}{C\sqrt{N}}$, we have $\zeta_1(t)=\int_0^t\eta_1(\tau)\mathrm{d}\tau=s\int_0^t\eta_2(\tau)\mathrm{d}\tau= s\zeta_2(t)$, and then
\begin{equation}
    \notag
    \begin{aligned}
    &\alpha_1^\eps(t)\zeta_1(t) + \beta_1^\eps(t)
    =\tfrac{1+\sqrt{s}}{2}e^{\eps p(t)}+\tfrac{1-\sqrt{s}}{2}e^{-\eps p(t)}=\tfrac{1+\eps\sqrt{s}}{2}e^{p(t)}+O(e^{-p(t)}),\\
    &\alpha_1^\eps(t)\zeta_2(t) + \beta_1^\eps(t)
    =\tfrac{1+\sqrt{s}}{2\sqrt{s}}e^{\eps p(t)} -\tfrac{1-\sqrt{s}}{2\sqrt{s}}e^{-\eps p(t)}=\tfrac{\eps+\sqrt{s}}{2\sqrt{s}}e^{p(t)}+O(e^{-p(t)}),\\
    &\alpha_2^\eps(t)\zeta_1(t) + \beta_2^\eps(t)
    =\tfrac{1+\sqrt{s}}{2}e^{\eps q(t)}+\tfrac{1-\sqrt{s}}{2}e^{-\eps q(t)}=\tfrac{1+\eps\sqrt{s}}{2}e^{q(t)}+O(e^{-q(t)}),\\
    &\alpha_2^\eps(t)\zeta_2(t) + \beta_2^\eps(t)
    =\tfrac{1+\sqrt{s}}{2\sqrt{s}}e^{\eps q(t)} -\tfrac{1-\sqrt{s}}{2\sqrt{s}}e^{-\eps q(t)}=\tfrac{\eps+\sqrt{s}}{2\sqrt{s}}e^{q(t)}+O(e^{-q(t)}).\\
    \end{aligned}
\end{equation}
Since $0<\gamma<\frac{2}{C-2}$ ( where $C>2$) or $C=2$, then we have $p(t)-q(t)=\frac{\gamma C\sqrt{\zeta_1(t)\zeta_2(t)}}{C\sqrt{N}}>0$, $p(t)+q(t)=\frac{(2+2\gamma-\gamma C)\sqrt{\zeta_1(t)\zeta_2(t)}}{C\sqrt{N}}>0$, and substitute these results into \Cref{decomposition-H-W} to obtain
\begin{equation}
    \begin{aligned}
    &e^{-p(t)}\mH(t)=\frac{1+\sqrt{s}}{2}\mH_1^++\frac{1-\sqrt{s}}{2}\mH_1^-+\bm{\Delta}_1(t),\\
    &e^{-p(t)}\mW(t)=\frac{1+\sqrt{s}}{2\sqrt{s}}\mW_1^+-\frac{1-\sqrt{s}}{2\sqrt{s}}\mW_1^-+\bm{\Delta}_2(t),\\
    \end{aligned}
\end{equation}
where $\|\bm{\Delta}_1(t)\|=O(e^{\max\{q(t)-p(t),-q(t)-p(t)\}}$ and $\|\bm{\Delta}_2(t)\|=O(e^{\max\{q(t)-p(t),-q(t)-p(t)\}}$. Therefore, we have
\begin{equation}
    e^{-p(t)}\mZ(t)=\left(\tfrac{1+\sqrt{s}}{2}\mH_1^++\tfrac{1-\sqrt{s}}{2}\mH_1^-, \tfrac{1+\sqrt{s}}{2\sqrt{s}}\mW_1^+-\tfrac{1-\sqrt{s}}{2\sqrt{s}}\mW_1^-\right) +\bm{\Delta}(t),
\end{equation}
where $\bm{\Delta}(t)=(\bm{\Delta}_1(t),\bm{\Delta}_2(t))$ and $\|\bm{\Delta}(t)\|\le \|\bm{\Delta}_1(t)\|+\|\bm{\Delta}_2(t)\|=O(e^{\max\{q(t)-p(t),-q(t)-p(t)\}}$.

Let $\overline{\mZ}=\left(\tfrac{1+\sqrt{s}}{2}\mH_1^++\tfrac{1-\sqrt{s}}{2}\mH_1^-, \tfrac{1+\sqrt{s}}{2\sqrt{s}}\mW_1^+-\tfrac{1-\sqrt{s}}{2\sqrt{s}}\mW_1^-\right)$, we have
\begin{equation}
    \begin{aligned}
    \left\|\frac{\mZ(t)}{\|\mZ(t)\|}-\frac{\overline{\mZ}}{\|\overline{\mZ}\|}\right\|=\left\|\frac{\overline{\mZ}+\bm{\Delta}(t)}{\|\overline{\mZ}+\bm{\Delta}(t)\|}-\frac{\overline{\mZ}}{\|\overline{\mZ}\|}\right\|\le\frac{2\|\overline{\mZ}\|\|\bm{\Delta}(t)\|}{\|\overline{\mZ}+\bm{\Delta}(t)\|\|\overline{\mZ}\|}= \frac{2\|\bm{\Delta}(t)\|}{\|\overline{\mZ}+\bm{\Delta}(t)\|}, 
    \end{aligned}
\end{equation}
thus $\left\|\frac{\mZ(t)}{\|\mZ(t)\|}-\frac{\overline{\mZ}}{\|\overline{\mZ}\|}\right\|=O(e^{\max\{q(t)-p(t),-q(t)-p(t)\}}$, and further $\lim_{t\rightarrow\infty}\frac{\mZ(t)}{\|\mZ(t)\|}=\frac{\overline{\mZ}}{\|\overline{\mZ}\|}$ when $\lim_{t\rightarrow \infty}\zeta_1(t)=\infty$. 

According to the definition of $\gE_1$ and $\gE_2$, we have
\begin{equation}
    \begin{aligned}
    \overline{\mZ}=\left(\tfrac{\sqrt{s}}{\sqrt{N}}\left(\tfrac{1+\sqrt{s}}{2\sqrt{s}}\mW_1^+-\tfrac{1-\sqrt{s}}{2\sqrt{s}}\mW_1^-\right)\otimes \1_N^\top, \tfrac{1+\sqrt{s}}{2\sqrt{s}}\mW_1^+-\tfrac{1-\sqrt{s}}{2\sqrt{s}}\mW_1^- \right),
    \end{aligned}
\end{equation}
thus we have $\overline{\mZ}\in\gE$, then $\lim_{t\rightarrow\infty}\frac{\mZ(t)}{\|\mZ(t)\|}=\frac{\overline{\mZ}}{\|\overline{\mZ}\|}\in\gE$.

Moreover, we have $\vb(t)=\vb_0 +\frac{(1+\gamma-\gamma C)\zeta_2(t)}{C} \1_C$, then $\forall i,j$,
\begin{equation}
    \notag
    \lim_{t\rightarrow \infty}\frac{b_i(t)}{b_j(t)}=\lim_{t\rightarrow \infty}\frac{b_{i}(0)+\frac{1+\gamma -\gamma C}{C}\zeta_2(t)}{b_{j}(0)+\frac{1+\gamma -\gamma C}{C}\zeta_2(t)}=1,
\end{equation}
thus $\lim_{t\rightarrow \infty}\frac{\max_i b_i(t)}{\max_i b_i(t)}$=1.
\end{proof}

\subsection{Proof of Theorem \ref{closed-form-dynamics-l2}}
\noindent\textbf{Theorem \ref{closed-form-dynamics-l2} (Dynamics of Features and Prototypes Under Weight Decay)}
\textit{
Consider the continual gradient flow in \Eqref{regularized-gradient-flow}, let $\mZ(t)=(\mH(t),\mW(t))$. If $\eta_1(t_1)\eta_2(t_2)=\eta_1(t_2)\eta_2(t_1)$ for any $t_1,t_2\ge0$, we have the following closed-form dynamics:
\begin{equation}
    \begin{aligned}
    \mZ(t)=&\Pi_1^+\mZ_0\begin{pmatrix}a_1^+(t)\mI_{CN} & 0\\0 & b_1^+(t)\mI_{C}\end{pmatrix}+\Pi_1^-\mZ_0\begin{pmatrix}a_1^-(t)\mI_{CN} & 0\\0 & b_1^-(t)\mI_{C}\end{pmatrix}+\\
    &\Pi_2^+\mZ_0\begin{pmatrix}a_2^+(t)\mI_{CN} & 0\\0 & b_2^+(t)\mI_{C}\end{pmatrix}+\Pi_2^-\mZ_0\begin{pmatrix}a_2^-(t)\mI_{CN} & 0\\0 & b_2^-(t)\mI_{C}\end{pmatrix}+\\
    &\Pi_3\mZ_0\begin{pmatrix}a_3(t)\mI_{CN} & 0\\0 & b_3(t)\mI_{C}\end{pmatrix}
    \end{aligned}
\end{equation}
and
\begin{equation}
    \vb(t)=\phi(t)\left(\vb_0+\tfrac{1+\gamma -\gamma C}{C}\psi(t)\1_C\right),
\end{equation}
where $\Pi_1^+\mZ_0$, $\Pi_1^-\mZ_0$, $\Pi_1^+\mZ_0$, $\Pi_1^-\mZ_0$, and $\Pi_3\mZ_0$ follow the definition in Theorem \ref{closed-form-dynamics}, $a_1^\eps$, $a_2^\eps$, $b_1^\eps$, $b_2^\eps$, $a_3$, and $b_3$ for $\eps\in\{\pm\}$ are the scalars that depend only on $C$, $N$, $\gamma$, $\lambda_1$, $\lambda_2$, $\eta_1$, and $\eta_2$ (where the detailed forms can be seen in \ref{all-proofs}), $\phi(t)=\exp(-\lambda\int_0^t\eta_2(\tau)\mathrm{d}\tau)$, and $\psi(t)=\int_0^t\zeta_2(\tau)\exp(\lambda\int_0^\tau\eta_2(s)\mathrm{d}s)\mathrm{d}\tau$.
}
\begin{proof}
According to the gradient flow in \cref{regularized-gradient-flow} and the notations in the proof of Theorem \ref{closed-form-dynamics}, we have:
\begin{equation}
    \mZ'(t)=\mZ(t)\mA(t)-\mZ(t)\begin{pmatrix}
    \lambda_1\eta_1(t)\mI_{CN} & 0\\
    0 & \lambda_2\eta_2(t)\mI_{C}
    \end{pmatrix},
\end{equation}
\textit{i.e.}, $\mZ'(t)=\mZ(t)\mA_\lambda(t)$, where $\mA_\lambda(t)=\mA(t)-\bm{\Lambda}(t)$ and $\bm{\Lambda}(t)=\left(\begin{smallmatrix}
    \lambda_1\eta_1(t)\mI_{CN} & 0\\
    0 & \lambda_2\eta_2(t)\mI_{C}
    \end{smallmatrix}\right)$.

For any $t_1,t_2$, we have the matrix commutator of $\mA_\lambda(t_1)$ and $\mA_\lambda(t_2)$ 
\begin{equation}
    \begin{aligned}
    &[\mA_\lambda(t_1),\mA_\lambda(t_2)]\\
    =&\mA_\lambda(t_1)\mA_\lambda(t_2)-\mA_\lambda(t_2)\mA_\lambda(t_1)\\
    =&[\mA(t_1)-\bm{\Lambda}(t_1)][\mA(t_2)-\bm{\Lambda}(t_2)]-[\mA(t_2)-\bm{\Lambda}(t_2)][\mA(t_1)-\bm{\Lambda}(t_1)]\\
    =& \bm{\Lambda}(t_2)\mA(t_1)-\bm{\Lambda}(t_1)\mA(t_2)+\mA(t_2)\bm{\Lambda}(t_1)-\mA(t_1)\bm{\Lambda}(t_2) \\
    =&\begin{pmatrix}
    0 & \lambda_1[\eta_1(t_2)\eta_2(t_1)-\eta_1(t_1)\eta_2(t_2)]\mM^\top\\
    \lambda_2[\eta_2(t_2)\eta_1(t_2)-\eta_2(t_1)\eta_1(t_2)]\mM & 0
    \end{pmatrix}\\
    =&0
    \end{aligned}
\end{equation}

where the last equality is based on the fact that $\eta_2(t_1)\eta_1(t_2)=\eta_2(t_2)\eta_1(t_1)$. Therefore, according to Magnus approach, we have
\begin{equation}
    \mZ(t)=\mZ_0\exp\left(\int_0^t \mA_\lambda(\tau)\mathrm{d}\tau\right)=\mZ_0\exp\begin{pmatrix}
    -\lambda_1\zeta_1(t)\mI_{CN} & \zeta_2(t)\mM^\top\\
    \zeta_1(t)\mM & -\lambda_2\zeta_2(t)\mI_{C}\\
    \end{pmatrix},
\end{equation}
where $\zeta_1(t)=\int_0^t\eta(\tau)\mathrm{d}\tau$ and $\zeta_2(t)=\int_0^t\eta(\tau)\mathrm{d}\tau$.

Let $\mB_{\lambda} = \mB-\left(\begin{smallmatrix}\lambda_1\mI_{CN} & 0\\0 & \lambda_2\mI_{C}\end{smallmatrix}\right)$, we have
\begin{equation}
    \mZ(t)=\mZ_0\exp(\mB_\lambda\mC(t))=\mZ_0\sum_{k=0}^\infty \frac{(\mB_\lambda\mC(t))^k}{k!}.
\end{equation}

We again consider the orthogonal decomposition of $\mZ_0$, \textit{i.e.}, $\mZ_0=(\Pi_1^++\Pi_1^-+\Pi_2^++\Pi_2^-+\Pi_3)\mZ_0$. As mentioned in the proof of Theorem \ref{closed-form-dynamics}, we have
\begin{equation}
    \begin{aligned}
    \Pi_1^\eps \mZ_0 \mB&=\tfrac{\eps(1+\gamma)}{C\sqrt{N}}\Pi_1^\eps\mZ_0,\\
    \Pi_2^\eps\mZ_0 \mB&=\tfrac{\eps(1+\gamma-\gamma C)}{C\sqrt{N}}\Pi_2^\eps\mZ_0, \text{ and }\\
    \Pi_3\mZ_0\mB&=0.
    \end{aligned}
\end{equation}
Therefore, for any $ \mD=(\mH,\mW)\in\{\Pi_1^\eps\mZ_0,\Pi_2^\eps\mZ_0,\Pi_3\mZ_0\}$ (where $\mH\in\sR^{p\times CN}$ and $\mW\in\sR^{p\times C}$) and the corresponding eigenvalue $\sigma\in\{\frac{\eps(1+\gamma)}{C\sqrt{N}}, \frac{\eps(1+\gamma-\gamma C)}{C\sqrt{N}}, 0\}$, we have
\begin{equation}
    \mD\mB=\sigma\mD,\ \mH\mM^\top = \sigma\mW,\text{ and } \mW\mM = \sigma\mH.
\end{equation}
In the following, we will prove that there exist two scalars $a(t)$ and $b(t)$, such that $\mD\exp(\mB_\lambda\mC(t))=\mD\begin{pmatrix}
a(t)\mI_{CN} & 0\\ 0 & b(t)\mI_C
\end{pmatrix}$ for any $\mD=(\mH,\mW)\in\{\Pi_1^\eps\mZ_0,\Pi_2^\eps\mZ_0,\Pi_3\mZ_0\}$.

First, we prove that $\mD(\mB_\lambda\mC(t))^k$ can be represented as $\mD(\mB_\lambda\mC(t))^k=(a_k(t)\mH,b_k(t)\mW)$ by induction, where $a_k(t),b_k(t)\in\sR$.

For $k=0$, we have $\mD(\mB_\lambda\mC(t))^0=(\mH,\mW)$, \textit{i.e.}, $a_0=b_0=1$. Assume that $\mD(\mB_\lambda\mC(t))^n=(a_n(t)\mH,b_n(t)\mW)$ for $k=n$. Then for $k=n+1$, we have
\begin{equation}
    \begin{aligned}
    &\mD(\mB_\lambda\mC(t))^{n+1}\\
    =&\mD(\mB_\lambda\mC(t))^{n}(\mB_\lambda\mC(t))\\
    =&(a_n(t)\mH,b_n(t)\mW)(\mB_\lambda\mC(t))\\
    =&(a_n(t)\mH,b_n(t)\mW)\begin{pmatrix}-\lambda_1\mI_{CN} & \mM^\top\\ \mM & -\lambda_2\mI_{C}\end{pmatrix} \begin{pmatrix}\zeta_1(t)\mI_{CN} & 0\\ 0 & \zeta_2(t)\mI_{C}\\\end{pmatrix}\\
    =& (b_n(t)\mW\mM-\lambda_1 a_n(t)\mH, a_n(t)\mH\mM^\top-\lambda_2 b_n(t)\mW)\begin{pmatrix}\zeta_1(t)\mI_{CN} & 0\\ 0 & \zeta_2(t)\mI_{C}\\\end{pmatrix}\\
    =& (\zeta_1(t)(\sigma b_n(t)-\lambda_1 a_n(t))\mH, \zeta_2(t)(\sigma a_n(t)-\lambda_2 b_n(t))\mW)
    \end{aligned},
\end{equation}
thus $a_{n+1}(t)=\zeta_1(t)(\sigma b_n(t)-\lambda_1 a_n(t))$ and $b_{n+1}(t)=\zeta_2(t)(\sigma a_n(t)-\lambda_2 b_n(t))$. 

To sum up, we have shown by induction that $\mD(\mB_\lambda\mC(t))^k$ can be represented as $\mD(\mB_\lambda\mC(t))^k=(a_k(t)\mH,b_k(t)\mW)=\mD\begin{pmatrix}a_k(t)\mI_{CN}&0\\0 & b_k(t)\mI_C\end{pmatrix}$, and $\begin{pmatrix}
a_k(t) \\ b_k(t)
\end{pmatrix}$ satisfies
\begin{equation}
    \begin{pmatrix}a_{k}(t) \\ b_{k}(t)\end{pmatrix}=\begin{pmatrix}-\lambda_1\zeta_1(t) & \sigma\zeta_1(t) \\ \sigma\zeta_2(t) & -\lambda_2\zeta_2(t)\end{pmatrix}\begin{pmatrix}a_{k-1}(t) \\ b_{k-1}(t)\end{pmatrix} \quad \text{with}\quad \begin{pmatrix}a_{0} \\ b_{0}\end{pmatrix}=\begin{pmatrix}1\\1\end{pmatrix},
\end{equation}
\textit{i.e.}, $\begin{pmatrix}a_{k}(t) \\ b_{k}(t)\end{pmatrix} = \left(\mS(\sigma,\lambda_1,\lambda_2,\zeta_1(t),\zeta_2(t))\right)^k\begin{pmatrix}1\\1\end{pmatrix}$, where $\mS(\sigma,\lambda_1,\lambda_2,\zeta_1,\zeta_2)=\begin{pmatrix}-\lambda_1\zeta_1 & \sigma\zeta_1 \\ \sigma\zeta_2 & -\lambda_2\zeta_2\end{pmatrix}$.

Therefore, we have
\begin{equation}
\mD\exp(\mB_\lambda\mC(t))=\mD\sum_{k=0}^\infty \frac{(\mB_\lambda\mC(t))^k}{k!}=\mD \begin{pmatrix}
    a(t)\mI_{CN} & 0\\
    0 & b(t)\mI_{C}\\
    \end{pmatrix},
\end{equation}
with $a(t)=\sum_{k=0}^\infty\frac{a_k(t)}{k!}$ and $b(t)=\sum_{k=0}^\infty\frac{b_k(t)}{k!}$, \textit{i.e.},
\begin{equation}
    \begin{aligned}
    \begin{pmatrix}a(t) \\ b(t)\end{pmatrix}&=\sum_{k=0}^\infty \frac{1}{k!}\begin{pmatrix}a_k(t) \\ b_k(t)\end{pmatrix}=\sum_{k=0}^\infty \frac{(\mS(\sigma,\lambda_1,\lambda_2,\zeta_1(t),\zeta_2(t)))^k}{k!}\begin{pmatrix}1\\1\end{pmatrix}\\
    &=\exp(\mS(\sigma,\lambda_1,\lambda_2,\zeta_1(t),\zeta_2(t)))\begin{pmatrix}1\\1\end{pmatrix}.
\end{aligned}
\end{equation}

Next, we are going to derive a concrete expression of $\exp(\mS(\sigma,\lambda_1,\lambda_2,\zeta_1(t),\zeta_2(t)))$. Let the determinant $|\mS(\sigma,\lambda_1,\lambda_2,\zeta_1(t),\zeta_2(t)-\theta\mI|=0$, we can derive that the eigenvalues of $\mS(\sigma,\lambda_1,\lambda_2,\zeta_1(t),\zeta_2(t))$ are
\begin{equation}
    \begin{aligned}
        &\theta_1= \tfrac{-(\lambda_1\zeta_1(t)+\lambda_2\zeta_2(t))-\sqrt{(\lambda_1\zeta_1(t)-\lambda_2\zeta_2(t))^2+4\sigma^2\zeta_1(t)\zeta_2(t)}}{2}<0\text{ and } \\
        &\theta_2=\tfrac{-(\lambda_1\zeta_1(t)+\lambda_2\zeta_2(t))+\sqrt{(\lambda_1\zeta_1(t)-\lambda_2\zeta_2(t))^2+4\sigma^2\zeta_1(t)\zeta_2(t)}}{2},
    \end{aligned}
\end{equation}

and the corresponding eigenvectors are
$$
\vv_1=\begin{pmatrix}1\\ \frac{\lambda_1\zeta_1(t)+\theta_1}{\sigma\zeta_1(t)}\end{pmatrix}, \text{ and } \vv_2=\begin{pmatrix}1\\ \frac{\lambda_1\zeta_1(t)+\theta_2}{\sigma\zeta_1(t)}\end{pmatrix}.
$$
Let $\mP=(\vv_1,\vv_2)$, we have
\begin{equation}
    \mS(\sigma,\lambda_1,\lambda_2,\zeta_1(t),\zeta_2(t)) = \mP\begin{pmatrix}
    \theta_1 & 0\\
    0 & \theta_2\\
    \end{pmatrix}\mP^{-1},\  \mP^{-1}=\tfrac{\sigma\zeta_1(t)}{\theta_2-\theta_1}\begin{pmatrix}
    \frac{\lambda_1\zeta_1(t)+\theta_2}{\sigma\zeta_1(t)} & -1\\
    -\frac{\lambda_1\zeta_1(t)+\theta_1}{\sigma\zeta_1(t)} & 1
    \end{pmatrix},
\end{equation}
and
\begin{equation}
    \begin{aligned}
    &\exp\left(\mS(\sigma,\lambda,\zeta_1(t),\zeta_2(t))\right)
    =\mP\begin{pmatrix}
    e^{\theta_1} & 0\\
    0 & e^{\theta_2}\\
    \end{pmatrix}\mP^{-1}\\
    =&\tfrac{\sigma\zeta_1(t)}{\theta_2-\theta_1}\begin{pmatrix}
    1 & 1\\
    \frac{\lambda_1\zeta_1(t)+\theta_1}{\sigma\zeta_1(t)} & \frac{\lambda_1\zeta_1(t)+\theta_2}{\sigma\zeta_1(t)}\\
    \end{pmatrix}
    \begin{pmatrix}
    e^{\theta_1} & 0\\
    0 & e^{\theta_2}\\
    \end{pmatrix}
    \begin{pmatrix}
    \frac{\lambda_1\eta_1(t)+\theta_2}{\sigma\zeta_1(t)} & -1\\
    -\frac{\lambda_1\zeta_1(t)+\theta_1}{\sigma\zeta_1(t)} & 1
    \end{pmatrix}
    \\
    =&\tfrac{1}{\theta_2-\theta_1}\begin{pmatrix}
    (\lambda_1\zeta_1(t)+\theta_2)e^{\theta_1}-(\lambda_1\zeta_1(t)+\theta_1)e^{\theta_2} & \sigma \zeta_1(t)e^{\theta_2}-\sigma\zeta_1(t)e^{\theta_1}\\
    (\lambda_1\zeta_1(t)+\theta_2)e^{\theta_2}-(\lambda_1\zeta_1(t)+\theta_1)e^{\theta_1} & \sigma\zeta_2(t)e^{\theta_2}-\sigma\theta_2(t)e^{\theta_1}
    \end{pmatrix}
    \end{aligned},
\end{equation}
thus 
\begin{equation}
    \begin{aligned}
    & a(t)=\frac{e^{\theta_2}}{\theta_2-\theta_1}\left[(\sigma\zeta_1(t)-\lambda_1\zeta_1(t)-\theta_1)-(\sigma\zeta_1(t)-\lambda_1\zeta_1(t)-\theta_2)e^{\theta_1-\theta_2}\right],\\
    & b(t)=\frac{e^{\theta_2}}{\theta_2-\theta_1}\left[(\lambda_1\zeta_1(t)+\theta_2+\sigma\zeta_2(t))+(\lambda_1\zeta_1(t)+\theta_1-\sigma\zeta_2(t))e^{\theta_1-\theta_2}\right].\\
    \end{aligned}
\end{equation}
Finally, since $\gE_1^+$, $\gE_1^-$, $\gE_2^+$, $\gE_2^-$ and $\gE_3$ are orthogonal subspace of $\sR^{p\times CN}\oplus\sR^{C\times p}$, we obtain that
\begin{equation}
    \begin{aligned}
    \mZ(t)=&(\Pi_1^+\mZ_0+\Pi_1^-\mZ_0+\Pi_2^+\mZ_0+\Pi_2^-\mZ_0+\Pi_3\mZ_0)\exp(\mB_\lambda \mC(t))\\
    =&\Pi_1^+\mZ_0\begin{pmatrix}a_1^+(t)\mI_{CN} & 0\\0 & b_1^+(t)\mI_{C}\end{pmatrix}+\Pi_1^-\mZ_0\begin{pmatrix}a_1^-(t)\mI_{CN} & 0\\0 & b_1^-(t)\mI_{C}\end{pmatrix}+\\
    &\Pi_2^+\mZ_0\begin{pmatrix}a_2^+(t)\mI_{CN} & 0\\0 & b_2^+(t)\mI_{C}\end{pmatrix}+\Pi_2^-\mZ_0\begin{pmatrix}a_2^-(t)\mI_{CN} & 0\\0 & b_2^-(t)\mI_{C}\end{pmatrix}+\\
    &\Pi_3\mZ_0\begin{pmatrix}a_3(t)\mI_{CN} & 0\\0 & b_3(t)\mI_{C}\end{pmatrix},
    \end{aligned}
\end{equation}
where $a_1^\eps(t)$, $b_1^\eps(t)$, $a_2^\eps(t)$, $b_2^\eps(t)$, $a_3(t)$ and $b_3(t)$ satisfy
\begin{equation}
    \begin{aligned}
    &\begin{pmatrix}a_1^\eps(t) \\ b_1^\eps(t)\end{pmatrix}=\exp\left(\mS\left(\tfrac{\eps(1+\gamma)}{C\sqrt{N}},\lambda_1, \lambda_2, \zeta_1(t),\zeta_2(t)\right)\right)\begin{pmatrix}1\\ 1\end{pmatrix},\\
    &\begin{pmatrix}a_2^\eps(t) \\ b_2^\eps(t)\end{pmatrix}=\exp\left(\mS\left(\tfrac{\eps(1+\gamma-\gamma C)}{C\sqrt{N}},\lambda_1, \lambda_2, \zeta_1(t),\zeta_2(t)\right)\right)\begin{pmatrix}1\\ 1\end{pmatrix},\\
    & \begin{pmatrix}a_3(t) \\ b_3(t)\end{pmatrix}=\exp\left(\mS\left(0,\lambda_1, \lambda_2, \zeta_1(t),\zeta_2(t)\right)\right)\begin{pmatrix}1\\ 1\end{pmatrix}.\\
    \end{aligned}
\end{equation}
Moreover, since $\vb'(t)=\eta_2(t)\frac{1+\gamma -\gamma C}{C}\1_C-\lambda_2\eta_2(t)\vb(t)$ is a first-order linear differential equation, then we have
\begin{equation}
    \vb(t)=\phi(t)\left(\vb_0+\tfrac{1+\gamma-\gamma C}{C}\psi(t)\1_C\right),
\end{equation}
where $\phi(t)=\exp(-\lambda_2\int_0^t\eta_2(\tau)\mathrm{d}\tau)$ and $\psi(t)=\int_0^t\zeta_2(\tau)\exp(\lambda_2\int_0^\tau\eta_2(s)\mathrm{d}s)\mathrm{d}\tau$.
\end{proof}

\subsection{Proof of Corollary \ref{convergence-of-dynamics-l2}}
\noindent\textbf{Corollary \ref{convergence-of-dynamics-l2}.}
\textit{
Under the conditions and notation of Theorem \ref{closed-form-dynamics-l2}, let $s=\frac{\eta_1(0)}{\eta_2(0)}$. If $0<\gamma<\frac{2}{C-2}$ (where $C>2$) or $C=2$, $\lambda_1=\lambda_2=\lambda$, and $\lim_{t\rightarrow\infty}\zeta_1(t)=\infty$, then there exist constants $\pi_h^+,\pi_h^-,\pi_w^+$, $\pi_w^-$, and $\omega$ only depending on $\lambda$, $\gamma$, $s$, $C$, and $N$, such that the gradient flow (as in \cref{regularized-gradient-flow}) behaves as:
\begin{equation}
    \left\|\frac{\mH(t)}{\|\mH(t)\|}-\frac{\pi_h^+\mH_1^++\pi_h^-\mH_1^-}{\|\pi_h^+\mH_1^++\pi_h^-\mH_1^-\|}\right\|+\left\|\frac{\mW(t)}{\|\mW(t)\|}-\frac{\pi_w^+\mW_1^++\pi_w^-\mW_1^-}{\|\pi_w^+\mH_1^++\pi_w^-\mH_1^-\|}\right\|=O(e^{-\omega\zeta_2(t)}),
\end{equation}
where $(\mH_1^+,\mW_1^+)=\Pi_1^+\mZ_0$, $(\mH_1^-,\mW_1^-)=\Pi_1^-\mZ_0$.
Furthermore, we have the following results:
\begin{itemize}
    \item If $\lambda > \frac{1+\gamma}{C\sqrt{N}}$, then $\lim_{t\rightarrow\infty}\|\mZ(t)\|=0$;
    \item If $\lambda =\frac{1+\gamma}{C\sqrt{N}}$, then $\lim_{t\rightarrow \infty}\mH(t)=\mH_1^+ + \frac{1-s}{1+s}\mH_1^-, \lim_{t\rightarrow \infty}\mW(t)=\mW_1^+ - \frac{1-s}{1+s}\mW_1^-$;
    \item If $\lambda < \frac{1+\gamma}{C\sqrt{N}}$, then $\lim_{t\rightarrow\infty}\|\mZ(t)\|=\infty$.
\end{itemize}
}
\begin{proof}
Since $\zeta_1(t)=s\zeta_2(t)$, then the eigenvalues of $\mS(\sigma,\lambda,\zeta_1(t),\zeta_2(t))$ satisfy 
\begin{equation}
\begin{aligned}
\theta_1 &=-\frac{\lambda(s+1)+\sqrt{\lambda^2(s-1)^2+4s\sigma^2}}{2}\zeta_2(t)\rightarrow -\infty \text{ as } t\rightarrow\infty, \text{ and } \\
\theta_2 &=\frac{(\sqrt{\lambda^2(s-1)^2+4s\sigma^2}-\lambda (s+1))}{2}\zeta_2(t).
\end{aligned}
\end{equation}

Let $\omega_1(\sigma,\lambda,s)=-\tfrac{\lambda(s+1)+\sqrt{\lambda^2(s-1)^2+4s\sigma^2}}{2}<0$ and $\omega_2(\sigma,\lambda,s)=\tfrac{(\sqrt{\lambda^2(s-1)^2+4s\sigma^2}-\lambda (s+1))}{2}$. For brevity, let $\omega_1$ and $\omega_2$ denote $\omega_1(\sigma,\lambda,s)$ and $\omega_2(\sigma,\lambda,s)$, respectively, and then we can reformulate $a(t)$ and $b(t)$ as
\begin{equation}
    \begin{aligned}
    a(t)&=\frac{e^{\omega_2\zeta_2(t)}}{\omega_2-\omega_1}\left[(s\sigma-s\lambda-\omega_1)-(s\sigma-s\lambda-\omega_2)e^{(\omega_1-\omega_2)\zeta_2(t)}\right]\\
    &=\tfrac{s\sigma-s\lambda-\omega_1}{\omega_2-\omega_1}e^{\omega_2\zeta_2(t)} + O(e^{(\omega_1)\zeta_2(t)}),\\
    b(t)&=\frac{e^{\omega_2\zeta_2(t)}}{\omega_2-\omega_1}\left[(s\lambda+\omega_2+\sigma)+(s\lambda+\omega_1-\sigma)e^{(\omega_1-\omega_2)\zeta_2(t)}\right]\\
    &=\tfrac{s\lambda+\omega_2+\sigma}{\omega_2-\omega_1}e^{\omega_2\zeta_2(t)} + O(e^{\omega_1\zeta_2(t)}).
    \end{aligned}
\end{equation}

Moreover, according to Theorem \ref{closed-form-dynamics-l2}, we have
\begin{equation}
    \begin{aligned}
    &\mH(t)=\sum_{\substack{\eps\in\{\pm\}\\
    i\in\{1,2\}}}a_i^\eps(t)\mH_i^\eps + a_3(t)\mH_3,\quad
    \mW(t)=\sum_{\substack{\eps\in\{\pm\}\\
    i\in\{1,2\}}}b_i^\eps(t)\mW_i^\eps + b_3(t)\mW_3,\\
    \end{aligned}
\end{equation}
with
\begin{equation}
    \begin{aligned}
    &\begin{pmatrix}a_1^\eps(t) \\ b_1^\eps(t)\end{pmatrix}=\exp\left(\mS\left(\tfrac{\eps(1+\gamma)}{C\sqrt{N}},\lambda, \zeta_1(t),\zeta_2(t)\right)\right)\begin{pmatrix}1\\ 1\end{pmatrix},\\
    &\begin{pmatrix}a_2^\eps(t) \\ b_2^\eps(t)\end{pmatrix}=\exp\left(\mS\left(\tfrac{\eps(1+\gamma-\gamma C)}{C\sqrt{N}},\lambda, \zeta_1(t),\zeta_2(t)\right)\right)\begin{pmatrix}1\\ 1\end{pmatrix},\\
    & \begin{pmatrix}a_3(t) \\ b_3(t)\end{pmatrix}=\exp\left(\mS\left(0,\lambda, \zeta_1(t),\zeta_2(t)\right)\right)\begin{pmatrix}1\\ 1\end{pmatrix}.\\
    \end{aligned}
\end{equation}
Since $0<\gamma<\frac{2}{C-2}$ (where $C>2$) or $C=2$, thus we have $\frac{1+\gamma }{C\sqrt{N}} > \frac{|1+\gamma -\gamma C|}{C\sqrt{N}}$, and then $\omega_2\left(\frac{1+\gamma}{C\sqrt{N}},\lambda,s\right)>\omega_2\left(\frac{|1+\gamma-\gamma C|}{C\sqrt{N}},\lambda,s\right)$. When $t\rightarrow \infty$, the dominant terms in $\mH(t)$ and $\mW(t)$ are the ones whose coefficient contains $\exp\left(\omega_2\left(\frac{1+\gamma}{C\sqrt{N}},\lambda,s\right)\zeta_2(t)\right)$, \textit{i.e.}, $a_1^+(t)$, $a_1^-(t)$, $b_1^+(t)$ and $b_1^-(t)$. Let $(\mH_1^\eps,\mW_1^\eps)=\Pi_1^\eps\mZ_0$, $(\mH_2^\eps,\mW_2^\eps)=\Pi_2^\eps\mZ_0$ and $(\mH_3,\mW_3)=\Pi_3\mZ_0$ for $\eps\in\{\pm1\}$, thus we have
\begin{equation}
    \begin{aligned}
    &e^{-\omega_2\left(\tfrac{1+\gamma}{C\sqrt{N}},\lambda,s\right)\zeta_2(t)}\mH(t)= \pi_h^+(\lambda, \gamma, s,C,N)\mH_1^+ + \pi_h^-(\lambda, \gamma, s,C,N)\mH_1^- + \bm{\Delta}_1,\\
    &e^{-\omega_2\left(\tfrac{1+\gamma}{C\sqrt{N}},\lambda,s\right)\zeta_2(t)}\mW(t)= \pi_w^+(\lambda, \gamma, s,C,N)\mW_1^+ + \pi_w^-(\lambda, \gamma, s,C,N)\mW_1^- + \bm{\Delta}_2,
    \end{aligned}
\end{equation}
where $\bm{\Delta}_1$ and $\bm{\Delta}_2$ decrease to zero as least as $O\left(e^{\frac{\left(\omega_2\left(\frac{|1+\gamma -\gamma C|}{C\sqrt{N}}, \lambda,s\right)-\omega_2\left(\frac{1+\gamma}{C\sqrt{N}}, \lambda,s\right)\right)\zeta_2(t)}{C\sqrt{N}}}\right)$, and
\begin{equation}
    \begin{aligned}
    &\pi_h^+(\lambda, \gamma, s,C,N)=\frac{s\frac{1+\gamma}{C\sqrt{N}}-s\lambda-\omega_1\left(\tfrac{1+\gamma}{C\sqrt{N}},\lambda,s\right)}{\omega_2\left(\tfrac{1+\gamma}{C\sqrt{N}},\lambda,s\right)-\omega_1\left(\tfrac{1+\gamma}{C\sqrt{N}},\lambda,s\right)},\\
    &\pi_h^-(\lambda, \gamma, s,C,N) = \frac{-s\frac{1+\gamma}{C\sqrt{N}}-s\lambda-\omega_1\left(\tfrac{1+\gamma}{C\sqrt{N}},\lambda,s\right)}{\omega_2\left(\tfrac{1+\gamma}{C\sqrt{N}},\lambda,s\right)-\omega_1\left(\tfrac{1+\gamma}{C\sqrt{N}},\lambda,s\right)},\\
    &\pi_w^+(\lambda, \gamma, s,C,N)=\frac{s\lambda+\omega_2\left(\tfrac{1+\gamma}{C\sqrt{N}},\lambda,s\right)+\tfrac{1+\gamma}{C\sqrt{N}}}{\omega_2\left(\tfrac{1+\gamma}{C\sqrt{N}},\lambda,s\right)-\omega_1\left(\tfrac{1+\gamma}{C\sqrt{N}},\lambda,s\right)},\\
    &\pi_w^-(\lambda, \gamma, s,C,N) = \frac{s\lambda+\omega_2\left(\tfrac{1+\gamma}{C\sqrt{N}},\lambda,s\right)-\tfrac{1+\gamma}{C\sqrt{N}}}{\omega_2\left(\tfrac{1+\gamma}{C\sqrt{N}},\lambda,s\right)-\omega_1\left(\tfrac{1+\gamma}{C\sqrt{N}},\lambda,s\right)}.\\
    \end{aligned}
\end{equation}
Therefore, we have
\begin{equation}
   \begin{aligned}
    &\lim_{t\rightarrow \infty}\frac{\mH(t)}{\|\mH(t)\|} = \frac{\pi_h^+(\lambda, \gamma, s,C,N)\mH_1^+ + \pi_h^-(\lambda, \gamma, s,C,N)\mH_1^-}{\|\pi_h^+(\lambda, \gamma, s,C,N)\mH_1^+ + \pi_h^-(\lambda, \gamma, s,C,N)\mH_1^-\|},\\
    &\lim_{t\rightarrow \infty}\frac{\mW(t)}{\|\mW(t)\|} = \frac{\pi_w^+(\lambda, \gamma, s,C,N)\mW_1^+ + \pi_w^-(\lambda, \gamma, s,C,N)\mW_1^-}{\|\pi_w^+(\lambda, \gamma, s,C,N)\mW_1^+ + \pi_w^-(\lambda, \gamma, s,C,N)\mW_1^-\|},
   \end{aligned}
\end{equation}
and the rate of convergence is $O\left(e^{\frac{\left(\omega_2\left(\frac{|1+\gamma -\gamma C|}{C\sqrt{N}}, \lambda,s\right)-\omega_2\left(\frac{1+\gamma}{C\sqrt{N}}, \lambda,s\right)\right)\zeta_2(t)}{C\sqrt{N}}}\right)$.

Moreover, we have the following conclusions:
\begin{itemize}
    \item If $\lambda=\frac{1+\gamma}{C\sqrt{N}}$, we have $\omega_2=0$, $\omega_1=-\lambda(s+1)$, $\pi_h^+=1$, $\pi_h^-=\frac{1-s}{1+s}$, $\pi_w^+=1$ and $\pi_w^-=-\frac{1-s}{1+s}$. Since $\lim_{t\rightarrow \infty}\zeta_2(t)=\infty$, we have
    \begin{equation}
        \lim_{t\rightarrow \infty}\mH(t)=\mH_1^+ + \frac{1-s}{1+s}\mH_1^-, \lim_{t\rightarrow \infty}\mW(t)=\mW_1^+ - \frac{1-s}{1+s}\mW_1^-.
    \end{equation}
    
    \item If $\lambda > \frac{1+\gamma}{C\sqrt{N}}$, we have $\omega_2 > 0$, and then $\lim_{t\rightarrow \infty}\|\mZ(t)\|=0$ since $\lim_{t\rightarrow \infty}\zeta_2(t)=\infty$.
    
    \item If $\lambda <\frac{1+\gamma}{C\sqrt{N}}$, we have $\omega_2 < 0$, and then $\lim_{t\rightarrow \infty}\|\mZ(t)\|=\infty$ since $\lim_{t\rightarrow \infty}\zeta_2(t)=\infty$.
\end{itemize}
So far the proof has been completed.
\end{proof}

\subsection{Proof of Theorem \ref{convergence-of-spherical-case}}

\begin{lemma}
\label{convergence-of-sperical-case-for-vector}
For $\vh(t), \vw\in\sR^p$, $\eta(t)>0$, let $\hat{\vv}=\frac{\hat{\vv}}{\|\vv\|_2}$ denote the $\ell_2$-normalized vector of $\vv$, considering the discrete dynamical system $\vh(t+1)=\vh(t)+\frac{\eta(t)}{\|\vh(t)\|_2}\left(\mI_p-\hat{\vh}(t)\hat{\vh}^\top(t)\right)\vw$, if $\hat{\vw}^\top\hat{\vh}(0)>-1$, the learning rate $\eta(t)$ satisfies that $\lim_{t\rightarrow\infty}\frac{\eta(t+1)}{\eta(t)}=1$, $\frac{\eta(t)}{\|\vh(t)\|_2}$ is non-increasing with $\frac{\eta(0)}{\|\vh(0)\|_2}<\frac{1}{\|\vw\|_2}$, and there exists a constant $\varepsilon>0$, \textit{s.t.}, $\eta(t)>\varepsilon$, \textit{s.t.}, $\eta(t)>\varepsilon$, then we have
\begin{equation}
    \lim_{t\rightarrow \infty}\left\|\frac{\vh(t)}{\|\vh(t)\|_2}-\frac{\vw}{\|\vw\|_2}\right\| = 0.
\end{equation}
\end{lemma}
\begin{proof}
For brevity, let $\alpha_t=\frac{\eta(t)\|\vw\|_2}{\|\vh(t)\|_2^2}$, $\xi_t=\frac{\eta(t+1)}{\eta(t)}$, and $\beta_t=\hat{\vw}^\top\hat{\vh}(t)$ denote the cosine similarity between $\hat{\vw}$ and $\hat{\vh}(t)$, then we can easily derive that $\alpha_t>0$ and $\beta_{t}>-1$ for all $t\ge 0$. 

We will show that \textit{$\alpha_t$ is monotonically decreasing and $\beta_t$ is monotonically increasing}. Note that $\vh(t)$ is orthogonal with $\left(\mI_p-\hat{\vh}(t)\hat{\vh}^\top(t)\right)\vw$, thus we have
\begin{equation}
    \notag
    \|\vh(t+1)\|_2^2=\|\vh(t)\|_2^2 + \frac{\eta^2(t)}{\|\vh(t)\|_2^2} \left\|\left(\mI_p-\hat{\vh}(t)\hat{\vh}^\top(t)\right)\vw\right\|_2^2\ge \|\vh(t)\|_2^2,
\end{equation}
which indicates $\|\vh(t)\|_2$ is monotonically increasing as a function of $t$, and then \textit{$\alpha_t$ is monotonically decreasing} since $\frac{\eta(t)\|\vw\|_2}{\|\vh(t)\|_2}$ is non-increasing.

Moreover, we can rearrange the discrete dynamics and formulate $\vh(t+1)$ as a positive combination of $\vh(t)$ and $\vw$:
\begin{equation}
    \label{discrete-dynamics-of-H}
    \vh(t+1)=\left(1-\frac{\eta(t)\vw^\top\hat{\vh}(t)}{\|\vh(t)\|_2}\right)\vh(t)+\frac{\eta(t)}{\|\vh(t)\|_2}\vw,
\end{equation}
so $\hat{\vw}^\top\hat{\vh}(t+1)\ge\hat{\vw}^\top\hat{\vh}(t)$, \textit{i.e.}, \textit{$\beta_{t}$ is monotonically increasing}, which is based on the facts that $1-\frac{\eta(t)\vw^\top\hat{\vh}(t)}{\|\vh(t)\|_2}\ge 1-\frac{\eta(t)\|\vw\|_2}{\|\vh(t)\|_2}\ge 1-\frac{\eta(0)\|\vw\|_2}{\|\vh(0)\|_2}>0$, $\frac{\eta(t)}{\|\vh(t)\|_2}>0$, and $\frac{\vy^\top(\vx+k\vy )}{\|\vx+k\vy\|_2}\ge \frac{\vy^\top\vx}{\|\vx\|_2}$ holds for all $k>0$ and $\vx,\vy\neq 0$.

We can formulate the discrete iterations of $\alpha_t$ and $\beta_t$ from the \cref{discrete-dynamics-of-H} as follows:
\begin{equation}
    \label{discrete_dynamical-system}
    \begin{aligned}
    \beta_{t+1}&=\frac{{\vw}^\top{\vh}(t+1)}{\|\vw\|_2\|\vh(t+1)\|_2}\\
    &=\frac{\vw^\top\left(\vh(t)+\frac{\eta(t)}{\|\vh(t)\|_2}\left(\mI_p-\hat{\vh}(t)\hat{\vh}^\top(t)\right)\vw\right)}{\|\vw\|_2\sqrt{\|\vh(t)\|_2^2 + \frac{\eta^2(t)}{\|\vh(t)\|_2^2} \left\|\left(\mI_p-\hat{\vh}(t)\hat{\vh}^\top(t)\right)\vw\right\|_2^2}}\\
    &=\frac{\beta_t+\frac{\eta(t)\|\vw\|_2}{\|\vh(t)\|_2^2}(1-\beta_t^2)}{\sqrt{1+\frac{\eta^2(t)\|\vw\|_2^2}{\|\vh(t)\|_2^4}(1-\beta_t^2)}}\\
    &=\frac{\beta_t+\alpha_t(1-\beta_t^2)}{\sqrt{1+\alpha_t^2(1-\beta_t^2)}},\\
    \alpha_{t+1}&= \frac{\eta(t+1)\|\vw\|_2}{\|\vh(t+1)\|_2^2}\\
    &=\frac{\xi_t\eta(t)\|\vw\|_2}{\|\vh(t)\|_2^2 + \frac{\eta^2(t)}{\|\vh(t)\|_2^2} \left\|\left(\mI_p-\hat{\vh}(t)\hat{\vh}^\top(t)\right)\vw\right\|_2^2}\\
    &=\frac{\xi_t\alpha_t}{1+\alpha_t^2(1-\beta_t^2)},
    \end{aligned}
\end{equation}
with $\beta_0=\hat{\vw}^\top\hat{\vh}(0)>-1$, $\alpha_0=\frac{\eta(0)\|\vw\|_2}{\|\vh(0)\|_2^2}>0$, $\xi_t\le 1$ and $\lim_{t\rightarrow\infty}\xi_t=1$.

To prove $\lim_{t\rightarrow \infty}\left\|\frac{\vh(t)}{\|\vh(t)\|_2}-\frac{\vw}{\|\vw\|_2}\right\| = 0$, we just need prove $\lim_{t\rightarrow \infty} \beta_t=1$. Note that $\alpha_t$ is monotonic decreasing and lower bounded by $0$, then the sequence $(\alpha_t)$ is convergent. Similarly, the sequence $(\beta_t)$ is convergent. Let $a = \lim_{t\rightarrow} \alpha_t$ and $b=\lim_{t\rightarrow \beta_t}$, we obtain
\begin{equation}
    \lim_{t\rightarrow\infty}\alpha_{t+1}=\lim_{t\rightarrow\infty} \frac{\xi_t\alpha_t}{1+\alpha_t^2(1-\beta_t^2)}\quad \Rightarrow\quad a=\frac{a}{1+\lim_{t\rightarrow\infty}\alpha_t^2(1-\beta_t^2)},
\end{equation}
thus $a=0$ or $\lim_{t\rightarrow \infty}\alpha_t^2(1-\beta_t^2)=0$, \textit{i.e.}, $a=0$ or $b=1$. Therefore, the limits of $\alpha_t$ and $\beta_t$ exist if and only if $\lim_{t\rightarrow \infty} \alpha_t=0$ or $\lim_{t\rightarrow \infty} \beta_t=1$. In the following, we will prove that the limit of $\beta_t$ must be equal to 1.

Firstly, we prove a simpler result when $\beta_0>0$:
\begin{lemma}
\label{lim-with-beta-ge-0}
For the discrete dynamics in \cref{discrete_dynamical-system}, if $\beta_0\ge0$, then $\lim_{t\rightarrow \infty} \beta_t=1$.
\end{lemma}
\begin{proof}
As aforementioned, due to the existence of the limit of $\alpha_t$, we have  $\lim_{t\rightarrow\infty}\alpha_t=0$ or $\lim_{t\rightarrow\infty}\beta_t=1$. Thus, we just need to prove that $\lim_{t\rightarrow\infty}\beta_t=1$ as $\lim_{t\rightarrow\infty}\alpha_t=0$.

When $\lim_{t\rightarrow \infty}\alpha_t=0$, then there exists $\tau$, such that $\forall t>\tau$, $\alpha_t\le 1$.

According to the iterations in \cref{discrete_dynamical-system}, we can derive that
\begin{equation}
    \notag
    \begin{aligned}
    \frac{1-\beta_{t+1}^2}{\alpha_{t+1}}
    &=\frac{1+\alpha_t^2-\alpha_t^2\beta_t^2-(\alpha_t+\beta_t-\alpha_t\beta_t^2)^2}{\xi_t\alpha_t}\\
    &=\frac{1+\alpha_t^2-\alpha_t^2\beta_t^2-\alpha_t^2-\beta_t^2-\alpha_t^2\beta_t^4-2\alpha_t\beta_t+2\alpha_t^2\beta_t^2+2\alpha_t\beta_t^3}{\xi_t\alpha_t}\\
    &=\frac{1-\beta_t^2+\alpha_t^2(\beta_t^2-\beta_t^4)-2\alpha_t(\beta_t-\beta_t^3)}{\xi_t\alpha_t}\\
    &=\frac{1-\beta_t^2}{\xi_t\alpha_t}\cdot (1+\alpha_t^2\beta_t^2-2\alpha_t\beta_t)\\
    &=\frac{1-\beta_t^2}{\alpha_t}\cdot \frac{(1-\alpha_t\beta_t)^2}{\xi_t},
    \end{aligned}
\end{equation}
then $\forall t>\tau$,
\begin{equation}
    \begin{aligned}
    1-\beta_{t+1}^2&=\alpha_{t+1}\cdot\frac{1-\beta_0^2}{\alpha_0}\prod_{i=0}^t\frac{(1-\alpha_i\beta_i)^2}{\xi_i}\\
    &=  \alpha_{t+1}\cdot \frac{\eta(0)(1-\beta_0^2)}{\alpha_0\eta(t+1)}\cdot \prod_{i=0}^\tau (1-\alpha_i\beta_i)^2 \cdot\prod_{i=\tau+1}^t (1-\alpha_i\beta_i)^2\\
    &\le\alpha_{t+1}\cdot \frac{\eta(0)(1-\beta_0^2)}{\varepsilon\alpha_0}\cdot \prod_{i=0}^\tau (1-\alpha_i\beta_i)^2,
    \end{aligned}
\end{equation}
where the inequality is based on the fact that $1-\alpha_i\beta_i\in(0,1]$ when $0\le\beta_0\le\beta_i\le1$, $\eta(t+1)\ge \varepsilon$, and $\alpha_i\le 1$ for $i>\tau$. Since $\frac{\eta(0)(1-\beta_0^2)}{\varepsilon\alpha_0}\cdot \prod_{i=0}^\tau (1-\alpha_i\beta_i)^2$ is a constant, we obtain
\begin{equation}
    \notag
    \lim_{t\rightarrow \infty} 1-\beta_{t+1}^2\le \lim_{t\rightarrow \infty} \alpha_{t+1}\cdot \frac{\eta(0)(1-\beta_0^2)}{\varepsilon\alpha_0}\cdot \prod_{i=0}^\tau (1-\alpha_i\beta_i)^2=0,
\end{equation}
as $\lim_{t\rightarrow\infty}\alpha_{t+1}=0$. This reveals $\lim_{t\rightarrow \infty}\beta_t^2=1$. Furthermore, since $\beta_t\ge 0$, we then have $\lim_{t\rightarrow \infty}\beta_t=1$.
\end{proof}

Next, we are going to prove $\lim_{t\rightarrow\infty}\beta_t=1$ when $-1<\beta_0< 0$. According to Lemma \ref{lim-with-beta-ge-0}, we just need prove that $\exists \tau>0$, \textit{s.t.}, $\beta_\tau\ge 0$.

For the sake of contradiction, suppose that $\beta_t< 0$ for all $t>0$, we then have $\lim_{t\rightarrow\infty}\alpha_t=0$. As a consequence, we obtain
\begin{equation}
    \alpha_t+\beta_t-\alpha_t\beta_t^2< 0,\ \forall t\ge 0,
\end{equation}
and we know that $\exists t'>0$, such that 
\begin{equation}
    \label{contradiction}
    \alpha_t<\frac{\varepsilon}{\eta(0)}\alpha_0(1-\beta_0^2),\ \forall t\ge t'.
\end{equation}

According to the iterations in \cref{discrete_dynamical-system}, we can derive that
\begin{equation}
    \notag
    \begin{aligned}
    \alpha_{t+1}(1-\beta_{t+1}^2)&=\frac{\xi_t\alpha_t}{1+\alpha_t^2-\alpha_t^2\beta_t^2}\left(1-\frac{(\alpha_t+\beta_t-\alpha_t\beta_t^2)^2}{1+\alpha_t^2-\alpha_t^2\beta_t^2}\right)\\
    &=\xi_t\alpha_t\left(\frac{1+\alpha_t^2-\alpha_t^2\beta_t^2-(\alpha_t+\beta_t-\alpha_t\beta_t^2)^2}{(1+\alpha_t^2-\alpha_t^2\beta_t^2)^2}\right)\\
    &=\xi_t\alpha_t(1-\beta_t^2)\left(\frac{1-\alpha_t\beta_t}{1+\alpha_t^2-\alpha_t^2\beta_t^2}\right)^2\\
    & =\xi_t \alpha_t(1-\beta_t^2)\left(1-\frac{\alpha_t(\alpha_t+\beta_t-\alpha_t\beta_t^2)}{1+\alpha_t^2-\alpha_t^2\beta_t^2}\right)^2,
    \end{aligned}
\end{equation}
Since $1-\frac{\alpha_t(\alpha_t+\beta_t-\alpha_t\beta_t^2)}{1+\alpha_t^2-\alpha_t^2\beta_t^2}\ge1$, then for $t\ge t'$,
\begin{equation}
    \alpha_t\ge\alpha_{t}(1-\beta_{t}^2)\ge \xi_{t-1}\alpha_{t-1}(1-\beta_{t-1}^2)\ge \ldots\ge \alpha_0(1-\beta_0^2)\prod_{i=0}^{t-1}\xi_i \ge\frac{\varepsilon}{\eta(0)}\alpha_0(1-\beta_0^2),
\end{equation}
which contradicts the fact in \cref{contradiction}. Thus, $\exists \tau>0$, s.t. $\beta_\tau\ge 0$. Consider the dynamical system with an initial time $\tau$, we have $\lim_{t\rightarrow\infty} \beta_{t}=1$ according to Lemma \ref{lim-with-beta-ge-0}.

To sum up, we have proven that $\lim_{t\rightarrow\infty}\beta_t=1$ when $\beta_0>-1$ and $\alpha_0>0$.
\end{proof}

\noindent\textbf{Theorem \ref{convergence-of-spherical-case}} (\textbf{Convergence in the Spherical Constrained Case})
\textit{
Considering the discrete dynamics in \cref{normalized-discrete-dynamics}, if $\forall i\in[N],c\in[C]$, $\hat{\vw}_c^\top\hat{\vh}_{i,c}(0)>-1$, the learning rate $\eta(t)$ satisfies that $\frac{\eta(t)}{\|\vh_{i,c}(t)\|_2}$ is non-increasing,  $\frac{\eta(0)(1+\gamma)}{CN\|\vh_{i,c}(0)\|_2}\le \frac{1}{\|\vw_c\|_2}$,  $\lim_{t\rightarrow\infty}\frac{\eta(t+1)}{\eta(t)}=1$, and there exists a constant $\varepsilon>0$, \textit{s.t.}, $\eta(t)>\varepsilon$, then we have 
\begin{equation}
    \lim_{t\rightarrow\infty}\left\|\hat{\mH}(t)-\hat{\mW}(\mI_C\otimes \1_N^\top)\right\|=0,
\end{equation} 
and further if $\lim_{t\rightarrow\infty}\|\mH(t)\|<\infty$, then there exists a constant $\mu>0$, such that the error above shows exponential convergence:
\begin{equation}
     \left\|\hat{\mH}(t)-\hat{\mW}(\mI_C\otimes \1_N^\top)\right\|\le O(e^{-\mu t}).
\end{equation}
Moreover, if $\hat{\vw}_c^\top\hat{\vh}_{i,c}(0)=-1$, then $\vh_{i,c}(t)=\vh_{i,c}(0)$.
}
\begin{proof}
Since $\mH(t+1)=\mH(t)+\frac{(1+\gamma)\eta(t)}{CN}\left(\frac{\partial \hat{\mH}}{\partial \mH}\big|_{\mH=\mH(t)}\right)^\top\mW(\mI_C\otimes \1_N^\top)$, then for $i\in[N], c\in[C]$,
\begin{equation}
    \vh_{i,c}(t+1)=\vh_{i,c}(t)+\frac{(1+\gamma)\eta(t)}{CN\|\vh_{i,c}(t)\|_2}\left(\mI_p-\hat{\vh}_{i,c}(t)\hat{\vh}_{i,c}^\top(t)\right)\vw_{c}.
\end{equation}
According to Lemma \ref{convergence-of-sperical-case-for-vector}, when $\hat{\vw}_c^\top\hat{\vh}_{i,c}(0)>-1$ and $\frac{\eta(0)(1+\gamma)}{CN}<\frac{\|\vh_{i,c}(0)\|_2^2}{\|\vw_c\|_2}$, we have $\lim_{t\rightarrow \infty}\|\hat{\vh}_{i,c}-\hat{\vw}_c\|=0$, then $\lim_{t\rightarrow\infty}\left\|\hat{\mH}(t)-\hat{\mW}(\mI_C\otimes \1_N^\top)\right\|=0$.

If further $\lim_{t\rightarrow\infty}\|\mH(t)\|_2<\infty$, let $L=\sup_{i,c,t}\|\vh_{i,c}(t)\|$. According to the proof of Lemma \ref{convergence-of-sperical-case-for-vector}, $\forall i,c$, we have $\lim_{t\rightarrow\infty}\hat{\vw}_c^\top\hat{\vh}_{i,c}(t)=1$, then for a given constant $\delta>0$, $\exists \tau>0$, \textit{s.t.}, $\forall t> \tau$, $\hat{\vw}_c^\top\hat{\vh}_{i,c}(\tau)\ge \delta$. Consider $t>\tau$, we have
\begin{equation}
    \notag
    \begin{aligned}
    &1-\left(\hat{\vw}_c^\top\hat{\vh}_{i,c}(t+1)\right)^2\\
    =&\tfrac{\|\vh_{i,c}(t)\|_2^2}{\|\vh_{i,c}(t+1)\|_2^2}\left(1-\left(\hat{\vw}_c^\top\hat{\vh}_{i,c}(t)\right)^2\right)\left(1-\tfrac{(1+\gamma)\eta(t)\|\vw_c\|_2}{CN\|\vh_{i,c}(t)\|_2^2}\cdot \hat{\vw}_c^\top\hat{\vh}_{i,c}(t)\right)^2\\
    \le &\left(1-\left(\hat{\vw}_c^\top\hat{\vh}_{i,c}(0)\right)^2\right)\prod_{j=0}^t \left(1-\tfrac{(1+\gamma)\eta(j)\|\vw_c\|_2}{CN\|\vh_{i,c}(j)\|_2^2}\cdot \hat{\vw}_c^\top\hat{\vh}_{i,c}(j)\right)^2\\
    \le &\left(1-\left(\hat{\vw}_c^\top\hat{\vh}_{i,c}(0)\right)^2\right)\prod_{j=0}^\tau \left(1-\tfrac{(1+\gamma)\eta(j)\|\vw_c\|_2}{CN\|\vh_{i,c}(j)\|_2^2}\cdot \hat{\vw}_c^\top\hat{\vh}_{i,c}(j)\right)^2\prod_{j=\tau+1}^t \left(1-\tfrac{(1+\gamma)\varepsilon\delta\|\vw_c\|_2}{CN\|\vh_{i,c}(j)\|_2^2}\right)^2\\
    \le & \left(1-\left(\hat{\vw}_c^\top\hat{\vh}_{i,c}(0)\right)^2\right)\prod_{j=0}^\tau \left(1-\tfrac{(1+\gamma)\eta(j)\|\vw_c\|_2}{CN\|\vh_{i,c}(j)\|_2^2}\cdot \hat{\vw}_c^\top\hat{\vh}_{i,c}(j)\right)^2 \left(1-\tfrac{(1+\gamma)\varepsilon\delta\|\vw_c\|_2}{CNL^2}\right)^{2(t-\tau)}.
    \end{aligned}
\end{equation}
where the first, the second, and the third inequalities are based on the facts that $\frac{\|\vh_{i,c}(t)\|_2^2}{\|\vh_{i,c}(t+1)\|_2^2}\le 1$, $1-\frac{(1+\gamma)\eta(j+1)\|\vw_c\|_2}{CN\|\vh_{i,c}(j)\|_2^2}\cdot \hat{\vw}_c^\top\hat{\vh}_{i,c}(j)\le 1-\frac{(1+\gamma)\varepsilon\delta\|\vw_c\|_2}{CN\|\vh_{i,c}(j)\|_2^2}$ for $t>\tau$, and $\|\vh_{i,c}(j)\|_2\le L$, respectively.

Let $c_1=\max_{i,c}(1-(\hat{\vw}_c^\top\hat{\vh}_{i,c}(0))^2)\prod_{j=0}^\tau \left(1-\frac{(1+\gamma)\eta(j)\|\vw_c\|_2}{CN\|\vh_{i,c}(j)\|_2^2}\cdot \hat{\vw}_c^\top\hat{\vh}_{i,c}(j)\right)^2 \left(1-\frac{(1+\gamma)\varepsilon\delta\|\vw_c\|_2}{CNL^2}\right)^{-2\tau}$, and $\mu=\min_c -2\log \left(1-\frac{(1+\gamma)\varepsilon\delta\|\vw_c\|_2}{CNL^2}\right)$, then $1-\hat{\vw}_c^\top\hat{\vh}_{i,c}(t+1)\le \frac{c_1e^{-\mu t}}{1+\delta}$.

Therefore, we have
\begin{equation}
    \left\|\hat{\mH}(t)-\hat{\mW}(\mI_C\otimes \1_N^\top)\right\|_2^2=2\sum_{i,c}\big(1-\hat{\vw}_c^\top\hat{\vh}_{i,c}(t+1)\big)\le \frac{2c_1CNe^{-\mu t}}{(1+\delta)},
\end{equation}
\textit{i.e.}, $\left\|\hat{\mH}(t)-\hat{\mW}(\mI_C\otimes \1_N^\top)\right\|=O(e^{-\mu t})$.

Moreover, if $\hat{\vw}_c^\top\hat{\vh}_{i,c}(0)=-1$, we have
\begin{equation}
    \vh_{i,c}(t+1)=\vh_{i,c}(t)+\frac{(1+\gamma)\eta(t)}{CN\|\vh_{i,c}(t)\|_2}\left(\mI_p-\hat{\vh}_{i,c}(t)\hat{\vh}_{i,c}^\top(t)\right)\vw_{c}=\vh_{i,c}(t),
\end{equation}
thus $\vh_{i,c}(t)=\vh_{i,c}(0)$.
\end{proof}
\subsection{Proof of Theorem \ref{dynamics-under-NTK}}
\begin{theorem}
Let $\vz(t)$ denote the row-first vectorization of $\begin{pmatrix}\boldsymbol{H}(t) & 0\\0 & \boldsymbol{W}(t)\end{pmatrix}$, $\boldsymbol B=\begin{pmatrix}0 & \boldsymbol M^\top\\\boldsymbol M & 0\end{pmatrix}$, and $\boldsymbol A=\begin{pmatrix}0 & \nabla_{\boldsymbol \Theta}\boldsymbol H ^\top\nabla_{\boldsymbol \Theta}\boldsymbol H\\\mathbf I_C & 0\end{pmatrix}$. Considering the eigendecomposition $\boldsymbol A=\boldsymbol U_{\boldsymbol A}\boldsymbol \Lambda_{\boldsymbol A} \boldsymbol U_{\boldsymbol A}^{-1}$ and $\boldsymbol B=\boldsymbol U_{\boldsymbol B}\boldsymbol \Lambda_{\boldsymbol B} \boldsymbol U_{\boldsymbol B}^{-1}$, we have
\begin{equation}
\mC\boldsymbol z(t)=\exp[(\boldsymbol \Lambda_{\boldsymbol A}\otimes \boldsymbol \Lambda_{\boldsymbol B})t]\mC\boldsymbol z(0),
\end{equation}
where $\mC=\boldsymbol U_{\boldsymbol B}^{-1}\boldsymbol U_{\boldsymbol A}^{-1}\otimes \mathbf I$.
\end{theorem}
\begin{proof}
Since $\boldsymbol z(t)$ denote the row-first vectorization of $\begin{pmatrix}\boldsymbol{H}(t) & 0\\0 & \boldsymbol{W}(t)\end{pmatrix}$, $\boldsymbol A=\begin{pmatrix}0 & \nabla_{\boldsymbol \Theta}\boldsymbol H ^\top\nabla_{\boldsymbol \Theta}\boldsymbol H\\\mathbf I & 0\end{pmatrix}$, and $\boldsymbol B=\begin{pmatrix}0 & \boldsymbol M^\top\\\boldsymbol M & 0\end{pmatrix}$, we have
$$
\boldsymbol z'(t)=(\boldsymbol A\otimes \boldsymbol{B}^\top) \boldsymbol z(t)=(\boldsymbol A\otimes \boldsymbol{B}) \boldsymbol z(t),
$$
and then $\boldsymbol z(t)=e^{(\boldsymbol A\otimes \boldsymbol{B})t}\boldsymbol z(0)$. Considering the eigendecomposition $\boldsymbol A=\boldsymbol U_{\boldsymbol A}\boldsymbol \Lambda_{\boldsymbol A} \boldsymbol U_{\boldsymbol A}^{-1}$ and $\boldsymbol B=\boldsymbol U_{\boldsymbol B}\boldsymbol \Lambda_{\boldsymbol B} \boldsymbol U_{\boldsymbol B}^{-1}$, where $\boldsymbol \Lambda_{\boldsymbol A}=\text{diag}(\lambda_1^{\boldsymbol A},...,\lambda_n^{\boldsymbol A})$ and $\boldsymbol \Lambda_{\boldsymbol B}=\text{diag}(\lambda_1^{\boldsymbol B},...,\lambda_n^{\boldsymbol B})$, thus 
$$
\boldsymbol A\otimes \boldsymbol B= (\boldsymbol U_{\boldsymbol A}\boldsymbol \Lambda_{\boldsymbol A} \boldsymbol U_{\boldsymbol A}^{-1})\otimes \boldsymbol B=(\boldsymbol U_{\boldsymbol A}\otimes \mathbf I)(\boldsymbol \Lambda_{\boldsymbol A} \otimes \boldsymbol B)(\boldsymbol U_{\boldsymbol A}^{-1}\otimes \mathbf I),
$$
where $\otimes$ denotes Kronecker product. Moreover, we have
$$
\begin{aligned}
\exp[(\boldsymbol A\otimes \boldsymbol{B})t]&=(\boldsymbol U_{\boldsymbol A}\otimes \mathbf I)\exp[(\boldsymbol \Lambda_{\boldsymbol A}\otimes \boldsymbol{B})t](\boldsymbol U_{\boldsymbol A}^{-1}\otimes \mathbf I)\\
&=(\boldsymbol U_{\boldsymbol A}\boldsymbol U_{\boldsymbol B}\otimes \mathbf I)\exp[(\boldsymbol \Lambda_{\boldsymbol A}\otimes \boldsymbol \Lambda_{\boldsymbol B})t](\boldsymbol U_{\boldsymbol B}^{-1}\boldsymbol U_{\boldsymbol A}^{-1}\otimes \mathbf I).\\
\end{aligned}
$$
We know that $\exp[(\boldsymbol \Lambda_{\boldsymbol A}\otimes \boldsymbol \Lambda_{\boldsymbol B})t]$ is a diagonal matrix. Therefore, we have $(\boldsymbol U_{\boldsymbol B}^{-1}\boldsymbol U_{\boldsymbol A}^{-1}\otimes \mathbf I)\boldsymbol z(t)=\exp[(\boldsymbol \Lambda_{\boldsymbol A}\otimes \boldsymbol \Lambda_{\boldsymbol B})t](\boldsymbol U_{\boldsymbol B}^{-1}\boldsymbol U_{\boldsymbol A}^{-1}\otimes \mathbf I)\boldsymbol z(0)$, which leads to a concise closed-form dynamics on $(\boldsymbol U_{\boldsymbol B}^{-1}\boldsymbol U_{\boldsymbol A}^{-1}\otimes \mathbf I)\boldsymbol z(t)$.
\end{proof}
\subsection{Proof of Theorem \ref{PAL-dynamics}}

\noindent\textbf{Theorem \ref{PAL-dynamics}}\textit{
Consider the continual gradient flow (\Cref{eq:pal}) in which the prototypes $\mW$ is fixed, we have the closed-form dynamics:
\begin{equation}
    \mH(t)=e^{-\lambda\int_0^t\eta(\tau)\mathrm{d}\tau}\mH(0)+\frac{1-e^{-\lambda\int_0^t\eta(\tau)\mathrm{d}\tau}}{\lambda} \mW\mM,
\end{equation}
which further indicates that $\left\|\mH(t)-\frac{1}{\lambda}\mW\mM\right\|=O\left(e^{-\lambda\int_0^t\eta(\tau)\mathrm{d}\tau}\right)$.
}
\begin{proof}
For the first order non-homogeneous linear difference equation in \cref{eq:pal}, the solution is
\begin{equation}
    \begin{aligned}
    \mH(t)=&e^{-\lambda\int_0^t\eta(\tau)\mathrm{d}\tau}\left(\mH(0)+\int_0^t\eta(s)e^{\lambda\int_0^s\eta(\tau)\mathrm{d}\tau}\mathrm{d}s \mW\mM\right)\\
    =&e^{-\lambda\int_0^t\eta(\tau)\mathrm{d}\tau}\mH(0) + e^{-\lambda\int_0^t\eta(\tau)\mathrm{d}\tau}\int_0^t\frac{1}{\lambda}\mathrm{d}e^{\lambda\int_0^s\eta(\tau)\mathrm{d}\tau} \mW\mM\\
    =&e^{-\lambda\int_0^t\eta(\tau)\mathrm{d}\tau}\mH(0) + e^{-\lambda\int_0^t\eta(\tau)\mathrm{d}\tau}\frac{e^{\lambda\int_0^s\eta(\tau)\mathrm{d}\tau}}{\lambda}\bigg|_0^t \mW\mM\\
    =&e^{-\lambda\int_0^t\eta(\tau)\mathrm{d}\tau}\mH(0)+\frac{1-e^{-\lambda\int_0^t\eta(\tau)\mathrm{d}\tau}}{\lambda} \mW\mM,
    \end{aligned}
\end{equation}
and then $\|\mH(t)-\frac{1}{\lambda}\mW\mM\|=\|e^{-\lambda\int_0^t\eta(\tau)\mathrm{d}\tau}(\mH(0)-\frac{\mW\mM}{\lambda})\|=O(e^{-\lambda\int_0^t\eta(\tau)\mathrm{d}\tau})$.
\end{proof}

\subsection{The Projections onto $\gE_1^+$, $\gE_1^-$, $\gE_2^+$, $\gE_2^-$ and $\gE_3$}
\label{projection-of-subspaces}
\begin{lemma}
\label{projection-of-zero-centroid-protoytpes}
Let $\gS$ denote the subspace $\gS=\{\mW: \mW\1_n=0, \mW\in\sR^{m\times n}\}$, then the projection of a point $\mA\in\sR^{m\times n}$ onto $\gS$ can be denoted as $\Pi_{\gS}\mA=\mA\left(\mI_n-\frac{1}{n}\1_n\1_n^\top\right)$.
\end{lemma}
\begin{proof}
Let $\mW=(\vw_1,\vw_2,...,\vw_n)\in \gS$ and $\mA=(\va_1,\va_2,...,\va_n)\in\sR^{m\times n}$, we have
\begin{equation}
    \notag
    \begin{aligned}
    \|\mW-\mA\|_F^2&=\sum_{i=1}^n \|\vw_i-\va_i\|_2^2\ge \frac{1}{n}\left\|\sum_{i=1}^n\vw_i-\sum_{i=1}^n\va_i\right\|_2^2\\
    &=\frac{1}{n}\|\mW\1_n-\mA\1_n\|_2^2=\frac{1}{n}\|\mA\1_n\|_2^2
    \end{aligned}
\end{equation}
where we used the Cauchy-Schwarz inequality, and the equality holds if and only if $\vw_i-\va_i=\vw_n-\va_n$, $\forall i\in [n]$, and $\sum_{i=1}^n \vw_i=0$, \textit{i.e.}, $\mW=\mA\left(\mI_n-\frac{1}{n}\1_n\1_n^\top\right)$. Therefore, the projection of $\mA$ onto $\gS$ is $\Pi_{\gS}\mA=\argmin_{\mW\in \gS}\|\mW - \mA\|_F^2=\mA\left(\mI_n-\frac{1}{n}\1_n\1_n^\top\right).
$
\end{proof}

\begin{lemma}
\label{projection-of-zero-centroid-features}
Let $\gS$ denote the subspace $\{\mW: \mW(\mI_c\otimes \1_n)=0, \mW\in\sR^{m\times cn}\}$, then the projection of a point $\mA\in\sR^{m\times cn}$ onto $\gS$ can be denoted as $\Pi_{\gS}\mA=\mA\left(\mI_{cn}-\frac{1}{n}\mI_c\otimes \1_n\1_n^\top\right)$.
\end{lemma}
\begin{proof}
The proof is similar to Lemma \ref{projection-of-zero-centroid-protoytpes}. We can simply let $\mW=(\mW_1,...,\mW_n)\in\gS$ and $\mA=(\mA_1,...,\mA_n)\in\sR^{m\times cn}$, where $\mW_i,\mA_i\in\sR^{m\times c}$, then we have
\begin{equation}
    \notag
    \begin{aligned}
     \|\mW-\mA\|_F^2&=\sum_{i=1}^n \|\mW_i-\mA_i\|_F^2\ge \frac{1}{n}\left\|\sum_{i=1}^n\mW_i-\sum_{i=1}^n\mA_i\right\|_2^2=\frac{1}{n}\|\mA(\mI_c\otimes \1_n)\|_2^2
    \end{aligned}
\end{equation}
where the equality holds if and only if $\mW_i-\mA_i=\mW_n-\mA_n$, $\forall i\in [n]$, and $\sum_{i=1}^n \mW_i=0$, \textit{i.e.}, $\mW=\mA(\mI_{cn}-\tfrac{1}{n}\mI_c\otimes \1_n\1_n^\top)$. Therefore, the projection of $\mA$ onto $\gS$ is $\Pi_\gS\mA=\mA\left(\mI_{cn}-\frac{1}{n}\mI_c\otimes \1_n\1_n^\top\right)$.
\end{proof}

\begin{lemma}
\label{projection-of-subspace1}
For $\mH\in\sR^{p\times CN}$, $\mW\in\sR^{p\times C}$, the projection of $(\mH,\mW)$ onto $\gE_1^\eps$ is 
\begin{equation}
    \Pi_1^\eps (\mH,\mW)=(\tfrac{\eps}{\sqrt{N}}(\mP\otimes \1_N^\top), \mP),
\end{equation}
where $\mP=\tfrac{1}{2}(\tfrac{\eps}{\sqrt{N}}\mH(\mI_C\otimes \1_N)+\mW)(\mI_C-\tfrac{1}{C}\1_C\1_C^\top)$.
\end{lemma}
\begin{proof}
Let $\gS=\{\mZ:\mZ\1_C=0, \mZ\in\sR^{p\times C}\}$ and $\mH=\{\mH_1,...,\mH_N\}$ (where $\mH_i\in\sR^{p\times C}$), the minimizer of $\mZ\in\gS$ is
\begin{equation}
    \begin{aligned}
    &\argmin_{\mZ\in \gS}\|\tfrac{\eps}{\sqrt{N}}(\mZ\otimes \1_N^\top)-\mH\|_F^2+\|\mZ-\mW\|_F^2\\
    =&\argmin_{\mZ\in \gS}\sum_{i=1}^N\|\tfrac{1}{\sqrt{N}}\mZ-\eps\mH_i\|_F^2+\|\mZ-\mW\|_F^2\\
    =&\argmin_{\mZ\in \gS}\|\mZ\|_F^2-\tfrac{2\eps}{\sqrt{N}}\sum_{i=1}^N\langle\mZ,\mH_i\rangle+\|\mH\|_F^2+\|\mZ-\mW\|_F^2\\
    =&\argmin_{\mZ\in \gS}\|\mZ-\tfrac{\eps}{\sqrt{N}}\sum_{i=1}^N\mH_i\|_F^2-\|\frac{1}{\sqrt{N}}\sum_{i=1}^N\mH_i\|_F^2+\|\mH\|_F^2+\|\mZ-\mW\|_F^2\\
    =&\argmin_{\mZ\in \gS}\|\mZ-\tfrac{\eps}{\sqrt{N}}\mH(\mI_C\otimes \1_N)\|_F^2+\|\mZ-\mW\|_F^2+\|\mH\|_F^2-\|\tfrac{\eps}{\sqrt{N}}\mH(\mI_C\otimes \1_N)\|_F^2\\
    =&\argmin_{\mZ\in \gS}\|\mZ-\tfrac{1}{2}(\tfrac{\eps}{\sqrt{N}}\mH(\mI_C\otimes \1_N)+\mW)\|_F^2\\
    =&\tfrac{1}{2}(\tfrac{\eps}{\sqrt{N}}\mH(\mI_C\otimes \1_N)+\mW)(\mI_C-\tfrac{1}{C}\1_C\1_C^\top).
    \end{aligned}
\end{equation}
Thus, $\Pi_1^\eps (\mH,\mW)=(\tfrac{\eps}{\sqrt{N}}(\mP\otimes \1_N^\top), \mP)$ with $\mP=\tfrac{1}{2}(\tfrac{\eps}{\sqrt{N}}\mH(\mI_C\otimes \1_N)+\mW)(\mI_C-\tfrac{1}{C}\1_C\1_C^\top)$.
\end{proof}

\begin{lemma}
For $\mH\in\sR^{p\times CN}$, $\mW\in\sR^{p\times C}$, the projection of $(\mH,\mW)$ onto $\gE_2^\eps$ is 
\begin{equation}
    \Pi_2^\eps (\mH,\mW)=(\tfrac{\eps}{\sqrt{N}}\vh\1_{CN}^\top, \vh\1_C^\top),
\end{equation}
where $\vh=\tfrac{1}{2C}(\tfrac{\eps}{\sqrt{N}}\mH\1_{CN}+\mW\1_C) $.
\end{lemma}
\begin{proof}
We have
\begin{equation}
    \begin{aligned}
    &\argmin_{\vh\in\sR^p}\|\tfrac{\eps}{\sqrt{N}}\vh\1_{CN}^\top-\mH\|_F^2+\|\vh\1_C^\top-\mW\|_F^2\\
    =&\argmin_{\vh\in\sR^p}CN\|\tfrac{1}{\sqrt{N}}\vh-\tfrac{\eps}{CN}\mH\1_{CN}\|_2^2+C\|\vh-\tfrac{1}{C}\mW\1_C\|_2^2\\
    =&\argmin_{\vh\in\sR^p}\|\vh-\tfrac{\eps}{C\sqrt{N}}\mH\1_{CN}\|_2^2+\|\vh-\tfrac{1}{C}\mW\1_C\|_2^2\\
    =&\argmin_{\vh\in\sR^p}\|\vh-\tfrac{1}{2C}(\tfrac{\eps}{\sqrt{N}}\mH\1_{CN}+\mW\1_C)\|_2^2\\
    =&\tfrac{1}{2C}(\tfrac{\eps}{\sqrt{N}}\mH\1_{CN}+\mW\1_C)
    \end{aligned}
\end{equation}
\end{proof}

\begin{lemma}
For $\mH\in\sR^{p\times CN}$, $\mW\in\sR^{p\times C}$, the projection of $(\mH,\mW)$ onto $\gE_3$ is 
\begin{equation}
    \Pi_3(\mH,\mW)=(\mH(\mI_{CN}-\tfrac{1}{N}\mI_C\otimes \1_N\1_N^\top), 0).
\end{equation}
\end{lemma}
\begin{proof}
This can be easily derived by Lemma \ref{projection-of-zero-centroid-features}.
\end{proof}

\begin{figure}[t]
    \centering
    \subfigure[$\log\|\mW(t)\|_2$]{
        \includegraphics[scale=0.48]{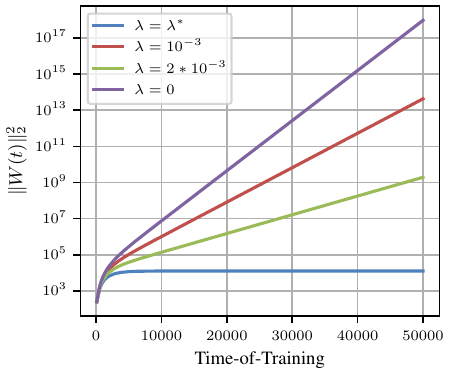}
        \label{fig:unconstrained-log-W}
    }
    \subfigure[$\log\|\mW(t)-\mP\|_2$]{
        \includegraphics[scale=0.48]{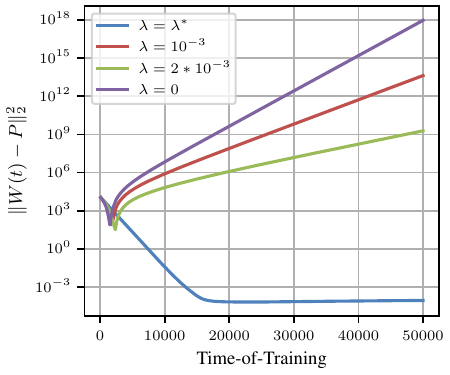}
        \label{fig:unconstrained-log-W-minus-P}
    }
    \subfigure[$\log\|\mH(t)\|_2-\log\|\mW(t)\|_2$]{
        \includegraphics[scale=0.48]{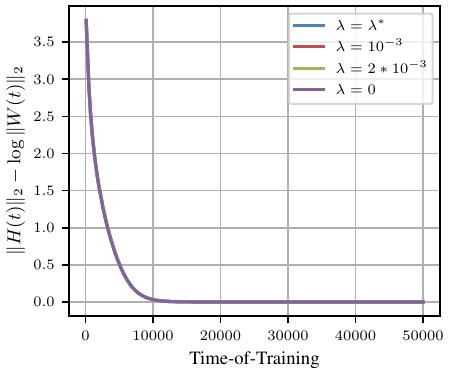}
        \label{fig:unconstrained-log-H-W}
    }
    \caption{Behavior of gradient descent iterates under the unhinged loss with different weight decay coefficients and $\eta_1(t)=\eta_2(t)=0.1$ (\textit{i.e.}, $s=1$), where $\mP$ denotes the component $\mW_1^+$ in the projection $\Pi_1^+\mZ_0$ calculated according to Lemma \ref{projection-of-subspace1}. (a) The logarithm of the norm of $\mW(t)$. As expected, the norm increases exponentially when $\lambda<\lambda^*$; (b) The difference between $\mW(t)$ and $\mP$. As expected (Corollary \ref{convergence-of-dynamics-l2}), $\mW(t)$ converges to $\mP$ when $\lambda=\lambda^*$, while other differences are dominated by $\|\mW(t)\|_2$; (c) The difference in $\ell_2$ norm between $\mH(t)$ and $\mW(t)$. The convergence is the same even if the weight decay is different.}
    \label{fig:unconstrained-behaviors}
    \vskip-10pt
\end{figure}

\begin{figure}[t]
    \centering
    \subfigure[Train Accuracy]{
     \includegraphics[scale=0.5]{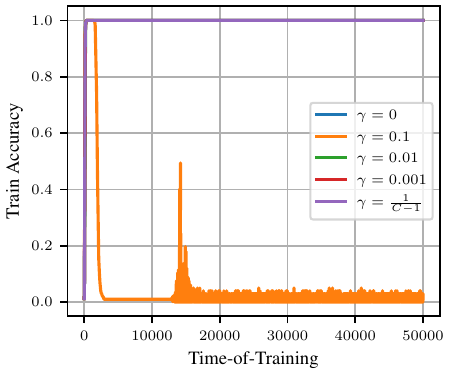}
     \label{fig:regularized-gamma-accuracy}
    }
    \subfigure[$\|\mH(t)-\overline{\mH}\|_2$]{
     \includegraphics[scale=0.5]{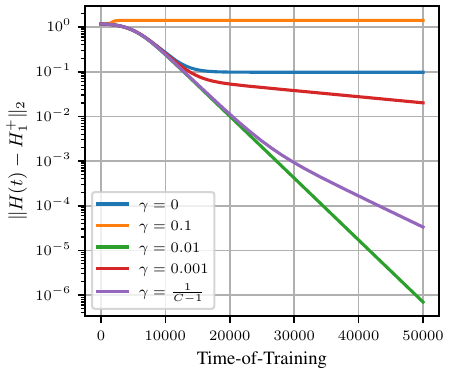}
     \label{fig:regularized-gamma-error-h}
    }
    \subfigure[$\|\mW(t)-\overline{\mW}\|_2$]{
    \includegraphics[scale=0.5]{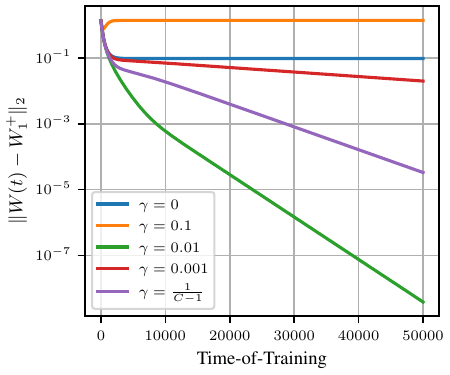}
    \label{fig:regularized-gamma-error-w}
    }
    \subfigure[$\|\mH(t)\|_2^2$]{
    \includegraphics[scale=0.5]{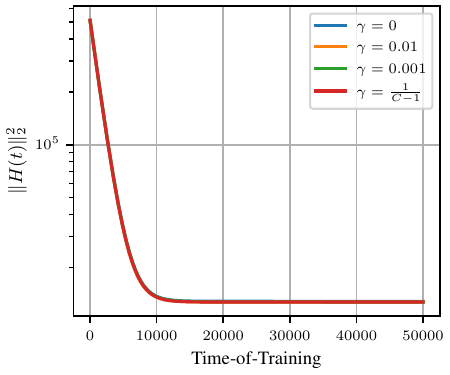}
    \label{fig:regularized-gamma-norm-h}
    }
    \subfigure[$\|\mW(t)\|_2^2$]{
    \includegraphics[scale=0.5]{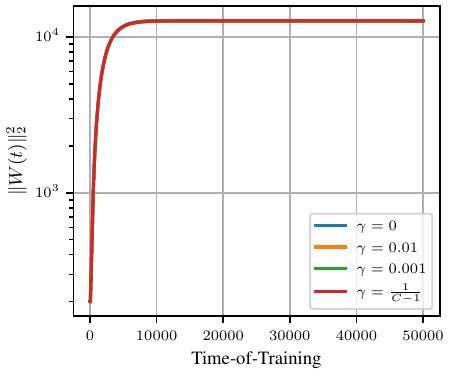}
    \label{fig:regularized-gamma-norm-w}
    }
    \caption{Verification of the behavior of regularized gradient descent iterates in \Cref{regularized-gradient-flow} with $\gamma\in\{0, 0.1, 0.01, 0.001, \frac{1}{C-1}\}$. We set $p=512$, $C=100$, $N=10$, $\lambda=\frac{(1+\gamma)}{C\sqrt{N}}$, $\eta_1(t)=\eta_2(t)=0.1$ (\textit{i.e.}, $s=1$), thus we have $\lim_{t\rightarrow\infty}\mZ(t)=\Pi_1^+\mZ_0$, according to Corollary \ref{convergence-of-dynamics-l2}, and then randomly initialize $\mH_0$ and $\mW_0$. (a) The training accuracy with the prediction rule $\argmax_c \vw_c^\top\vh$. As expected, the features align to their corresponding prototypes when $\gamma<\frac{2}{C-2}$. (b) The $\ell_2$ distance between $\mH(t)$ and $\mH_1^+$. As expected \cref{convergence-of-dynamics-l2}, the distance will decrease as exponential rate when $0<\gamma<\frac{2}{C-2}$.   (c) The $\ell_2$ distance between $\mW(t)$ and $\mW_1^+$.  (d) and (e) denote the norm of features and prototypes, respectively. As can be seen, $\|\mH\|_2$ and $\|\mW\|_2$ do not grow exponentially as in the unconstrained case, which confirms that weight decay can avoid excessive growth of feature norm and prototype norm.}
    \label{regularized-gamma}
\end{figure}

\begin{figure}[t]
    \centering
    \subfigure[Train Accuracy]{
        \includegraphics[scale=0.5]{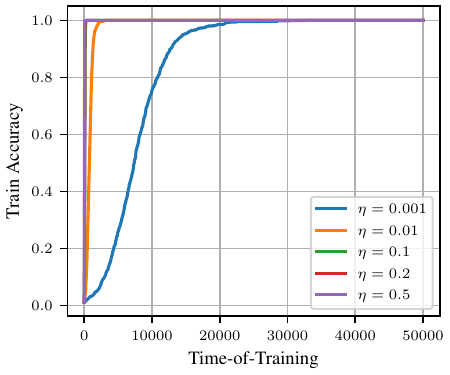}
        \label{fig:regularized-lr-accuracy}
    }
    \subfigure[$\|\mH(t)-\overline{\mH}\|_2$]{
        \includegraphics[scale=0.5]{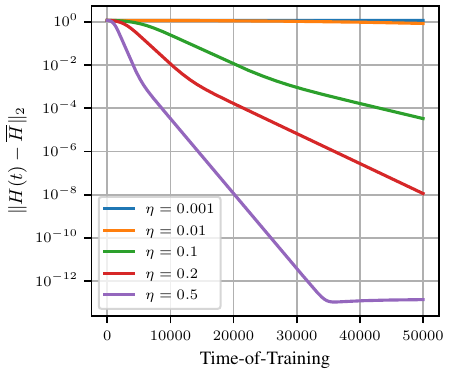}
        \label{fig:regularized-lr-error_h}
    }
    \subfigure[$\|\mW(t)-\overline{\mW}\|_2$]{
        \includegraphics[scale=0.5]{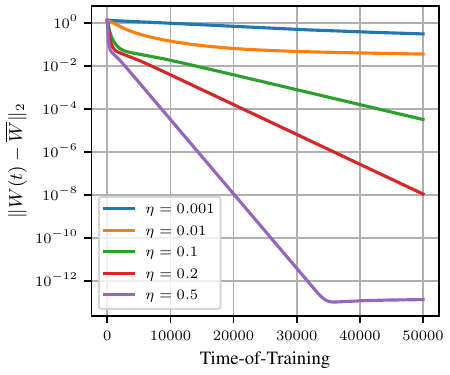}
        \label{fig:regularized-lr-error_w}
    }
    \subfigure[$\|\mH(t)\|_2$]{
        \includegraphics[scale=0.5]{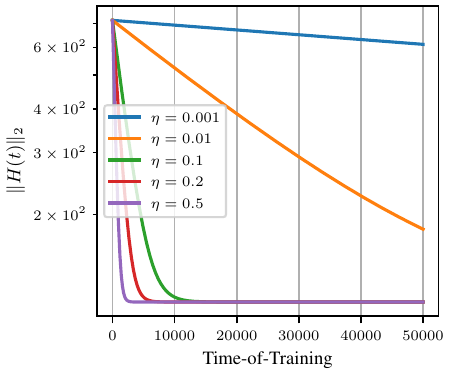}
        \label{fig:regularized-lr-norm_h}
    }
    \subfigure[$\|\mW(t)\|_2$]{
        \includegraphics[scale=0.5]{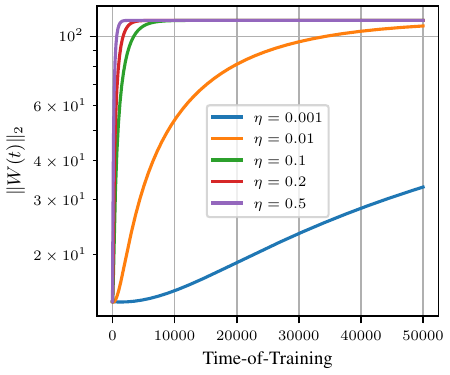}
        \label{fig:regularized-lr-norm_w}
    }
    \caption{Verification of the behavior of regularized gradient descent iterates in \Cref{regularized-gradient-flow} with different learning rates ($\eta\in \{0.001, 0.01, 0.1, 0.2, 0.5\}$). We set $p=512$, $C=100$, $N=10$, $\eta_1=\eta_2=\eta$ ($s=1$), $\gamma=\frac{1}{C-1}$, and $\lambda=\frac{1+\gamma}{C\sqrt{N}}$. As can be seen, features and prototypes converge to $(\overline{\mH}, \overline{\mW})$ exponentially, and larger learning rates can accelerate convergence.}
    \label{regularized-lr}
\end{figure}

\begin{figure}[t]
    \centering
    \subfigure[Train accuracy]{
        \includegraphics[scale=0.5]{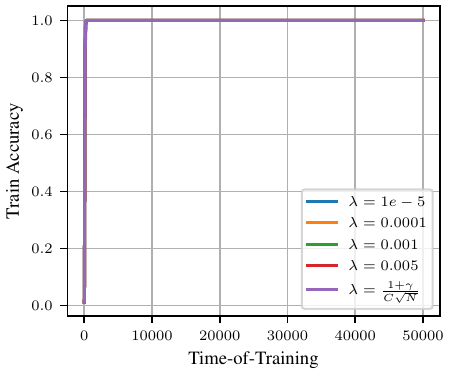}
        \label{fig:regularized-lambda-accuracy}
    }
    \subfigure[$\|\hat{\mH}(t)-\hat{\mH}_1^+\|_2$]{
        \includegraphics[scale=0.5]{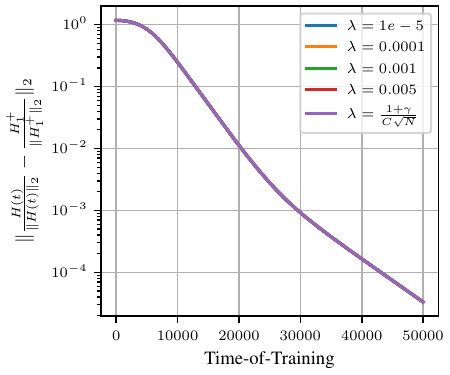}
        \label{fig:regularized-lambda-error-h}
    }
    \subfigure[$\|\hat{\mW}(t)-\hat{\mW}_1^+\|_2$]{
        \includegraphics[scale=0.5]{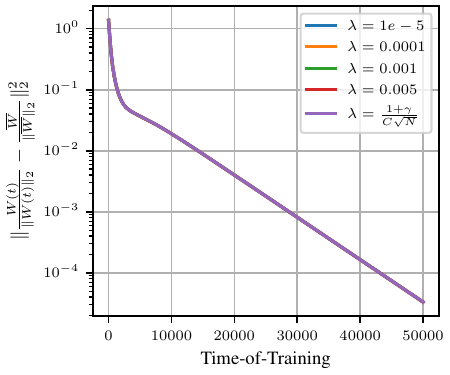}
        \label{fig:regularized-lambda-error-w}
    }
    \subfigure[$\|\mH(t)\|_2$]{
        \includegraphics[scale=0.5]{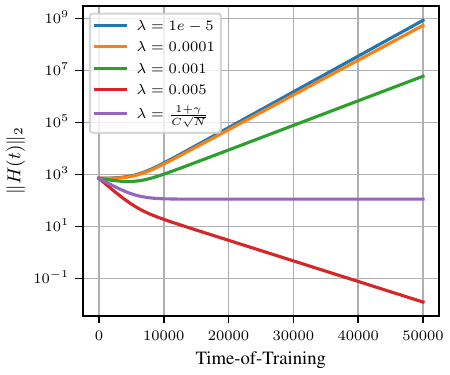}
        \label{fig:regularized-lambda-norm-h}
    }
    \subfigure[$\|\mW(t)\|_2$]{
        \includegraphics[scale=0.5]{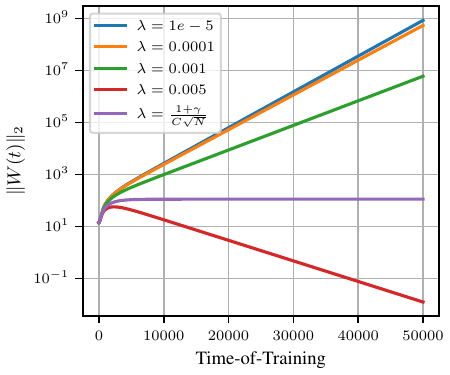}
        \label{fig:regularized-lambda-norm-w}
    }
    \caption{Verification of the behavior of regularized gradient descent iterates in \Cref{convergence-of-Z-l2} with different weight decay coefficients ($\lambda=\{1e-5, 1e-4, 1e-3,5e-3,\frac{1+\gamma}{C\sqrt{N}}\}$).  We set $p=512$, $C=100$, $N=10$, $\eta_1(t)=\eta_2(t)=0.1$ (\textit{i.e.}, $s=1$), $\gamma=\frac{1}{C-1}$, where $\overline{\mH}=\pi_h^+\mH_1^++\pi_h^-\mH_1^-$ and $\overline{\mW}=\pi_w^+\mW_1^++\pi_w^-\mW_1^-$ in Corollary \ref{convergence-of-dynamics-l2}. (a) The logarithm of the norm of $\mW(t)$. As expected, the norm increases exponentially when $\lambda<\lambda^*=\frac{1+\gamma}{C\sqrt{N}}$; (b) The difference between $\mW(t)$ and $\mP$. As expected in Corollary \ref{convergence-of-dynamics-l2}, $\mW(t)$ converges to $\mP$ when $\lambda=\lambda^*$, while other differences are dominated by $\|\mW(t)\|_2$; (c) The difference in $\ell_2$ norm between $\mH(t)$ and $\mW(t)$. The convergence is the same even if the weight decay is different.}
    \label{regularized-lambda}
\end{figure}

\begin{figure}[t]
    \centering
    \subfigure[Train accuracy]{
        \includegraphics[scale=0.5]{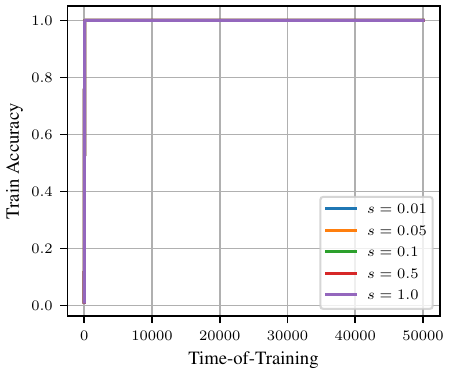}
        \label{fig:regularized-scale-accuracy}
    }
    \subfigure[$\|\hat{\mH}(t)-\hat{\mH}_1^+\|_2$]{
        \includegraphics[scale=0.5]{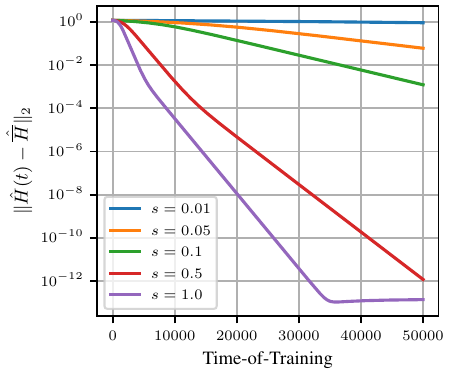}
        \label{fig:regularized-scale-error-h}
    }
    \subfigure[$\|\hat{\mW}(t)-\hat{\mW}_1^+\|_2$]{
        \includegraphics[scale=0.5]{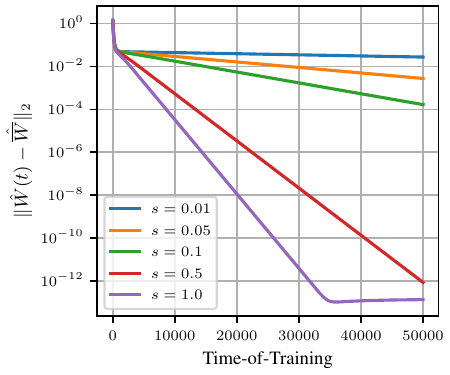}
        \label{fig:regularized-scale-error-w}
    }
    \subfigure[$\|\mH(t)\|_2$]{
        \includegraphics[scale=0.5]{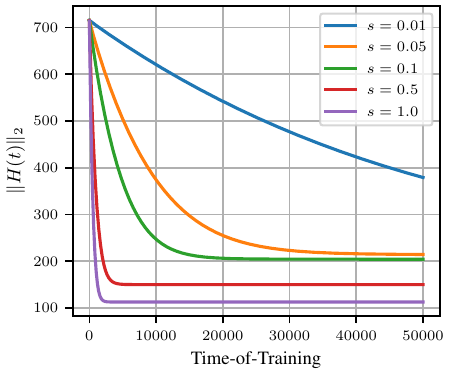}
        \label{fig:regularized-scale-norm-h}
    }
    \subfigure[$\|\mW(t)\|_2$]{
        \includegraphics[scale=0.5]{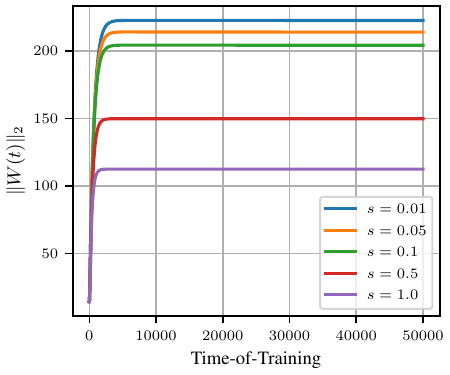}
        \label{fig:regularized-scale-norm-w}
    }
    \caption{Verification of the behavior of regularized gradient descent iterates in \Cref{convergence-of-Z-l2} with different scale parameters $s=\frac{\eta_1(0)}{\eta_2(0)} \in \{0.01, 0.05, 0.1, 0.5, 1.0\}$.  We set $p=512$, $C=100$, $N=10$, $\eta_2=0.5$, $\gamma=\frac{1}{C-1}$, where $\overline{\mH}=\pi_h^+\mH_1^++\pi_h^-\mH_1^-$ and $\overline{\mW}=\pi_w^+\mW_1^++\pi_w^-\mW_1^-$ in Corollary \ref{convergence-of-dynamics-l2}. As can be seen, larger scale parameter $s$ can achieve faster convergence speed.}
    \label{regularized-scale}
\end{figure}

\begin{figure}[t]
    \centering
    \subfigure[Train Accuracy]{
        \includegraphics[scale=0.48]{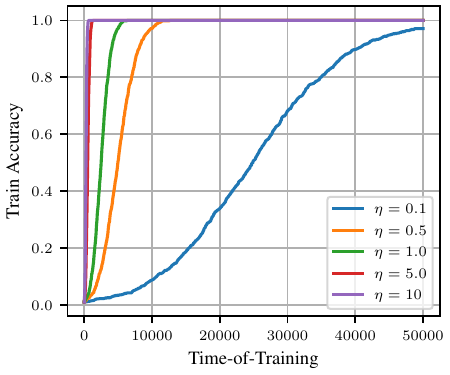}
    \label{fig:spherical-lr-accuracy}
    }
    \subfigure[$\|\hat{\mH}(t)-\hat{\mW}(\mI_C\otimes\1_N)\|_2$]{
        \includegraphics[scale=0.48]{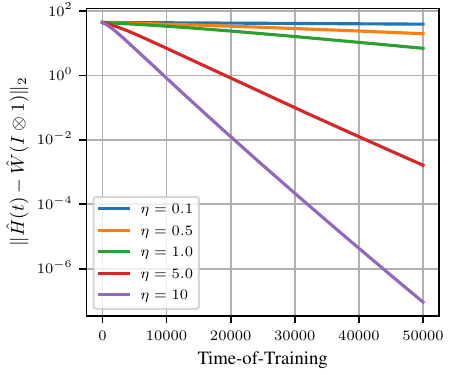}
        \label{fig:spherical-lr-error}
    }
    \subfigure[$\|\mH(t)\|_2$]{
        \includegraphics[scale=0.48]{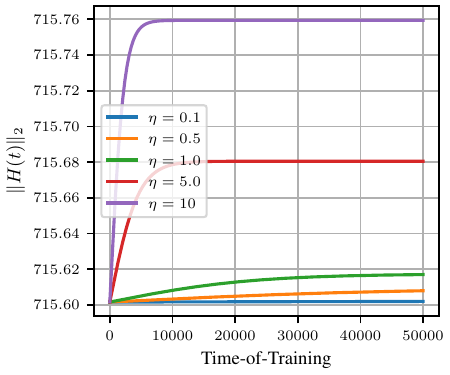}
        \label{fig:spherical-lr-faet-norm}
    }
    \subfigure[Train Accuracy]{
        \includegraphics[scale=0.48]{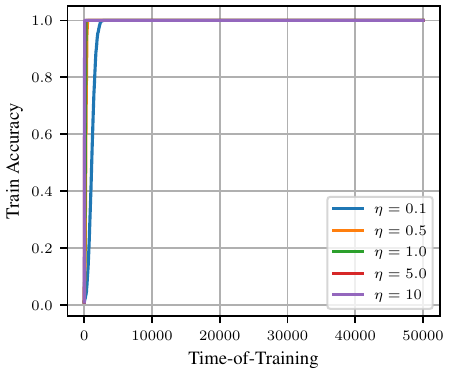}
        \label{fig:spherical-rescale-accuracy}
    }
    \subfigure[$\|\hat{\mH}(t)-\hat{\mW}(\mI_C\otimes\1_N)\|_2$]{
        \includegraphics[scale=0.48]{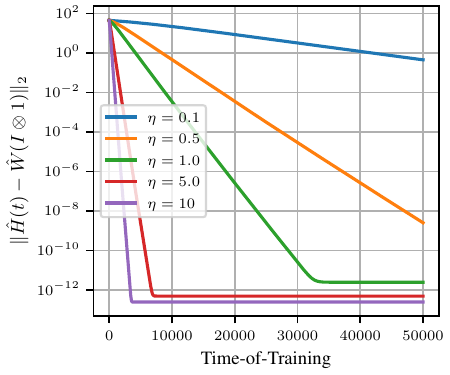}
        \label{fig:spherical-rescale-error}
    }
    \subfigure[$\|\mH(t)\|_2$]{
        \includegraphics[scale=0.48]{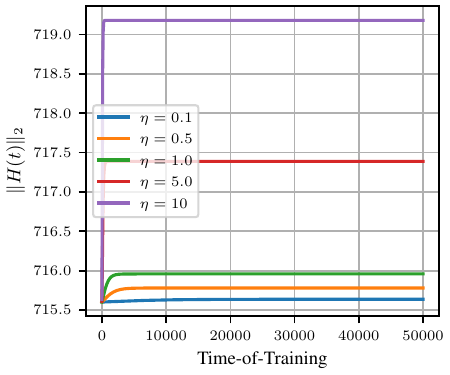}
        \label{fig:spherical-rescale-feat-norm}
    }
    \caption{Verification of the behavior of discrete gradient descent iterates in \Cref{spherical-objective} under anchored prototypes with different learning rates $\eta\in\{0.1, 0.5, 1.0, 5.0, 10\}$ and without (a-c) or with (d-f) rescaled learning rates. We set $p=512$, $C=100$, and $N=10$. As expected in Theorem \ref{convergence-of-spherical-case}, the feature norm $\|\mH(t)\|_2$ is non-decreasing, and the error $\|\hat{\mH}(t)-\hat{\mW}(\mI_C\otimes \1_N^\top)\|_2$ shows exponential decrease.}
    \label{spherical-lr-behaviors}
\end{figure}

\begin{table}[htbp]
    \scriptsize
    \centering
    \caption{Test accuracies on imbalanced CIFAR-10 under different explicit feature regularization.}
    \label{tab:imbalanced-cifar-10}
    \begin{tabular}{c|cccccccc}
    \toprule
     \textbf{Dataset} & \multicolumn{8}{c}{Imbalanced CIFAR-10}\\
     \midrule
     \textbf{Imbalance Type} &  \multicolumn{4}{c|}{long-tailed} & \multicolumn{4}{c}{step}\\
     \midrule
     \textbf{Imbalance Ratio} &  \multicolumn{1}{c|}{100} & \multicolumn{1}{c|}{50} & \multicolumn{1}{c|}{20} & \multicolumn{1}{c|}{10} &  \multicolumn{1}{c|}{100} & \multicolumn{1}{c|}{50} & \multicolumn{1}{c|}{20} & \multicolumn{1}{c}{10} \\
     \midrule
     baseline & 67.81 & 72.93 & 83.97 & 88.37 & 61.24 & 68.10 & 78.73 & 85.49 \\
     $\lambda=5e-6$ & \textbf{67.84} & 72.85 & 83.17 & \textbf{89.06} & 60.79 & \textbf{68.41} & \textbf{80.20} & \textbf{86.69} \\
     $\lambda=1e-5$ & 67.74 & \textbf{76.14} & \textbf{84.17} & \textbf{89.19} & \textbf{61.50} & 67.71 & \textbf{80.97} & \textbf{87.18} \\
     $\lambda=5e-5$ & \textbf{69.74} & \textbf{77.29} & \textbf{84.92} & \textbf{88.64} & 60.69 & \textbf{70.27} & \textbf{81.27} & \textbf{87.17}\\
     \bottomrule
    \end{tabular}
\end{table}

\begin{table}[htbp]
    \scriptsize
    \centering
    \caption{Test accuracies on imbalanced CIFAR-100 under different explicit feature regularization.}
    \label{tab:imbalanced-cifar-100}
    \begin{tabular}{c|cccccccc}
    \toprule
     \textbf{Dataset} & \multicolumn{8}{c}{Imbalanced CIFAR-100}\\
     \midrule
     \textbf{Imbalance Type} &  \multicolumn{4}{c|}{long-tailed} & \multicolumn{4}{c}{step}\\
     \midrule
     \textbf{Imbalance Ratio} &  \multicolumn{1}{c|}{100} & \multicolumn{1}{c|}{50} & \multicolumn{1}{c|}{20} & \multicolumn{1}{c|}{10} &  \multicolumn{1}{c|}{100} & \multicolumn{1}{c|}{50} & \multicolumn{1}{c|}{20} & \multicolumn{1}{c}{10}\\
     \midrule
     baseline  & 33.37 & 39.40 & 42.96 & 56.38 & 40.89 & 42.69 & 51.92 & 57.52 \\
     $\lambda=5e-6$ & \textbf{36.00} & \textbf{41.92} & \textbf{50.75} & \textbf{60.13} & \textbf{41.90} & \textbf{43.85} & 47.80 & 56.74 \\
     $\lambda=1e-5$ & \textbf{36.61} & \textbf{42.36} & \textbf{49.21} & \textbf{58.91} & \textbf{41.48} & \textbf{43.77} & 49.64 & 56.49 \\
     $\lambda=5e-5$ & \textbf{34.88} & \textbf{42.74} & \textbf{54.72} & \textbf{60.84} & \textbf{40.97} & \textbf{43.20} & 48.96 & \textbf{57.97}\\
     \bottomrule
    \end{tabular}
\end{table}

\begin{figure}
    \centering
    \subfigure[CE ($\lambda=0$)]{
        \includegraphics[scale=0.2]{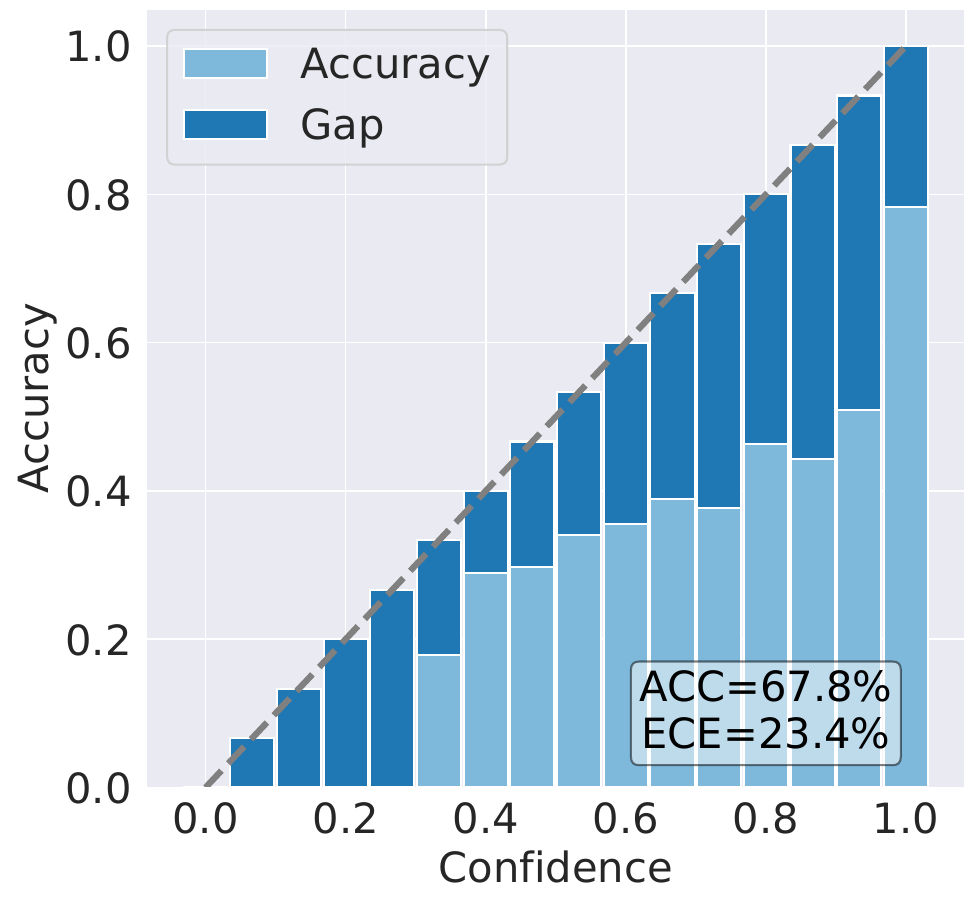}
        \label{fig:ece-lt-cifar10-100-0}
    }
    \subfigure[CE ($\lambda=5e-6$)]{
        \includegraphics[scale=0.2]{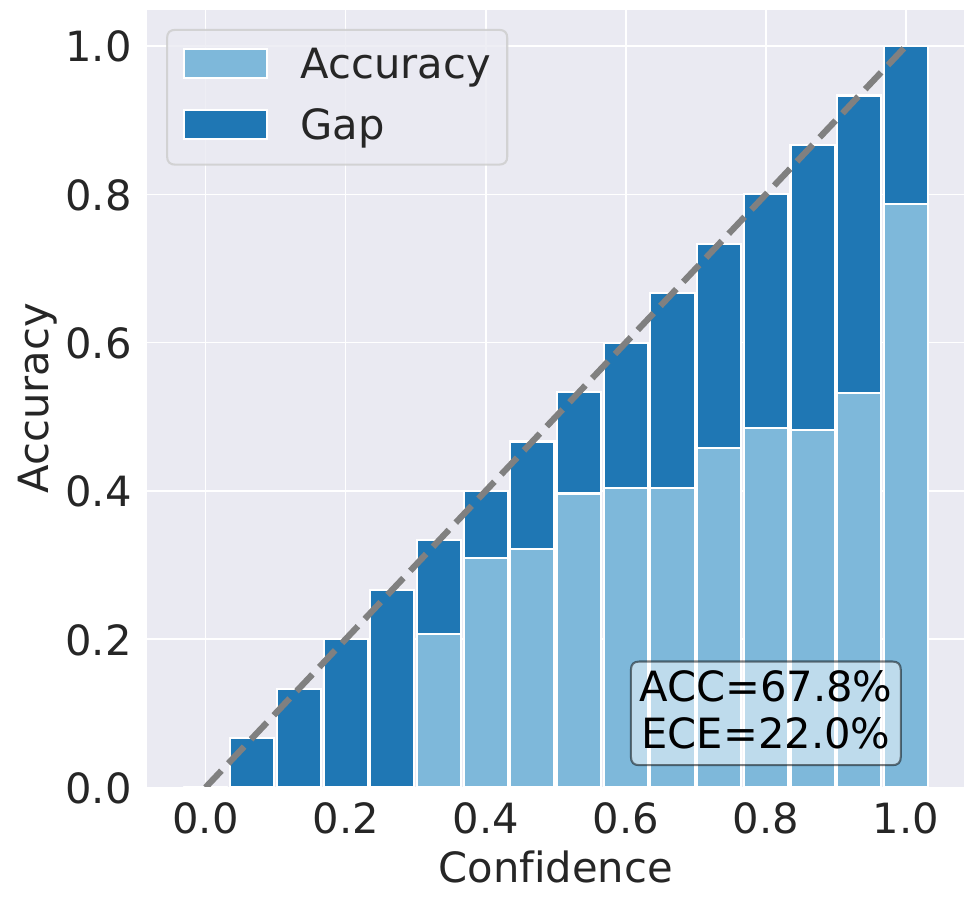}
        \label{fig:ece-lt-cifar10-100-5e-6}
    }
    \subfigure[CE ($\lambda=1e-5$)]{
        \includegraphics[scale=0.2]{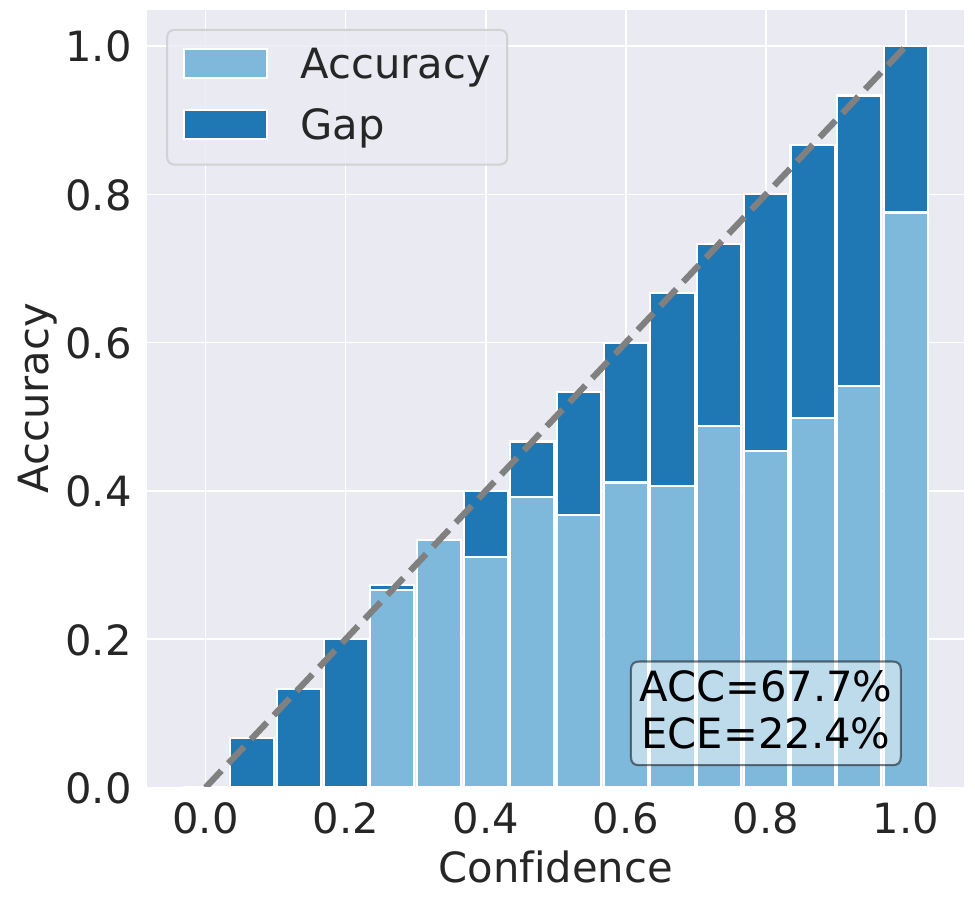}
        \label{fig:ece-lt-cifar10-100-1e-5}
    }
    \subfigure[CE ($\lambda=5e-5$)]{
        \includegraphics[scale=0.2]{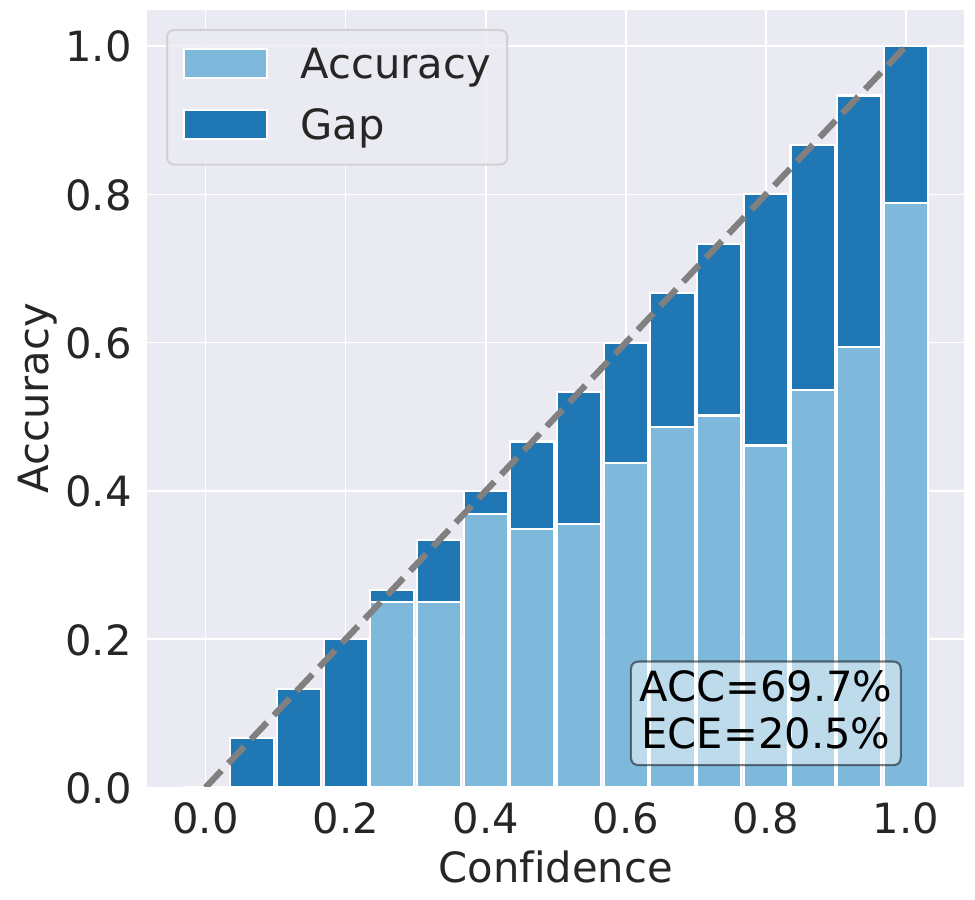}
        \label{fig:ece-lt-cifar10-100-5e-5}
    }
    \subfigure[CE ($\lambda=0$)]{
        \includegraphics[scale=0.2]{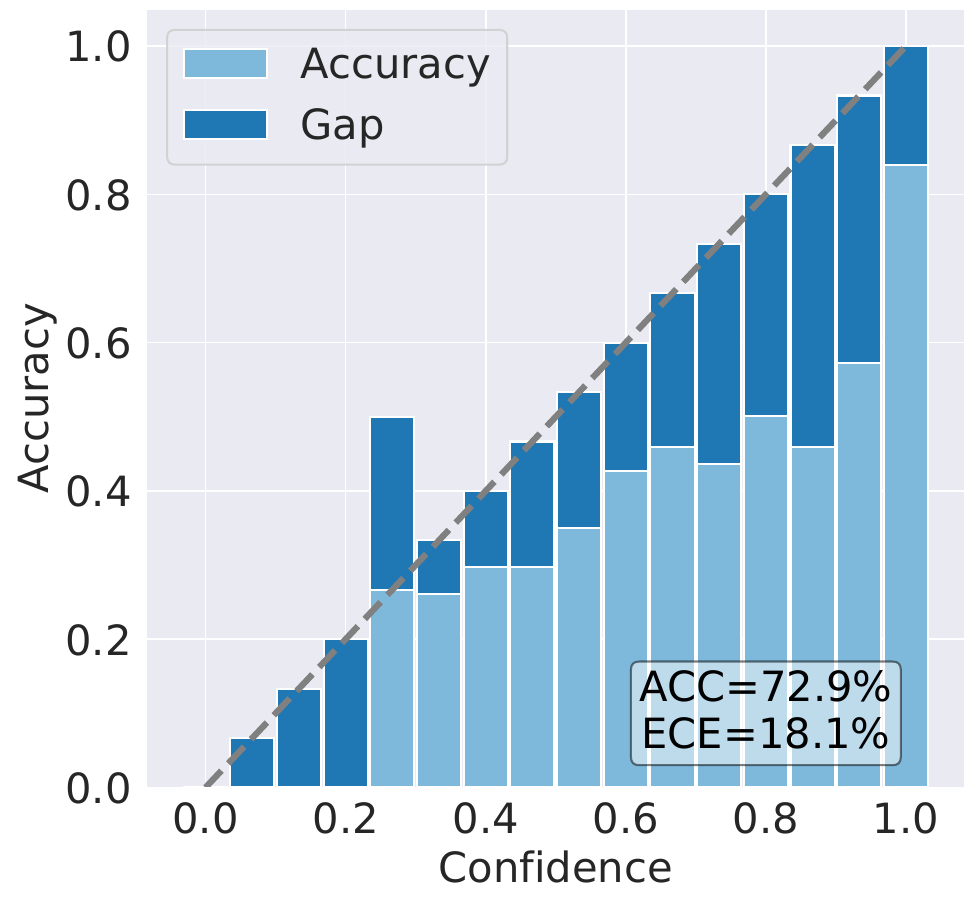}
        \label{fig:ece-lt-cifar10-50-0}
    }
   \subfigure[CE ($\lambda=5e-6$)]{
        \includegraphics[scale=0.2]{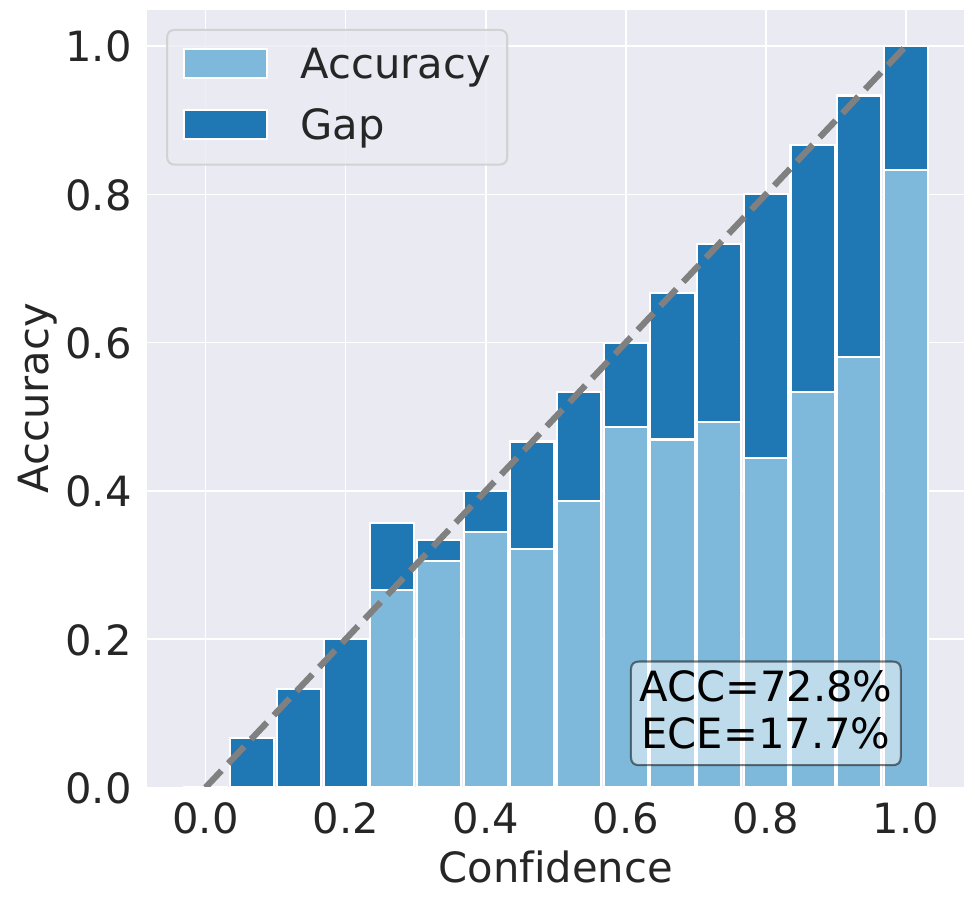}
        \label{fig:ece-lt-cifar10-50-5e-6}
    }
    \subfigure[CE ($\lambda=1e-5$)]{
        \includegraphics[scale=0.2]{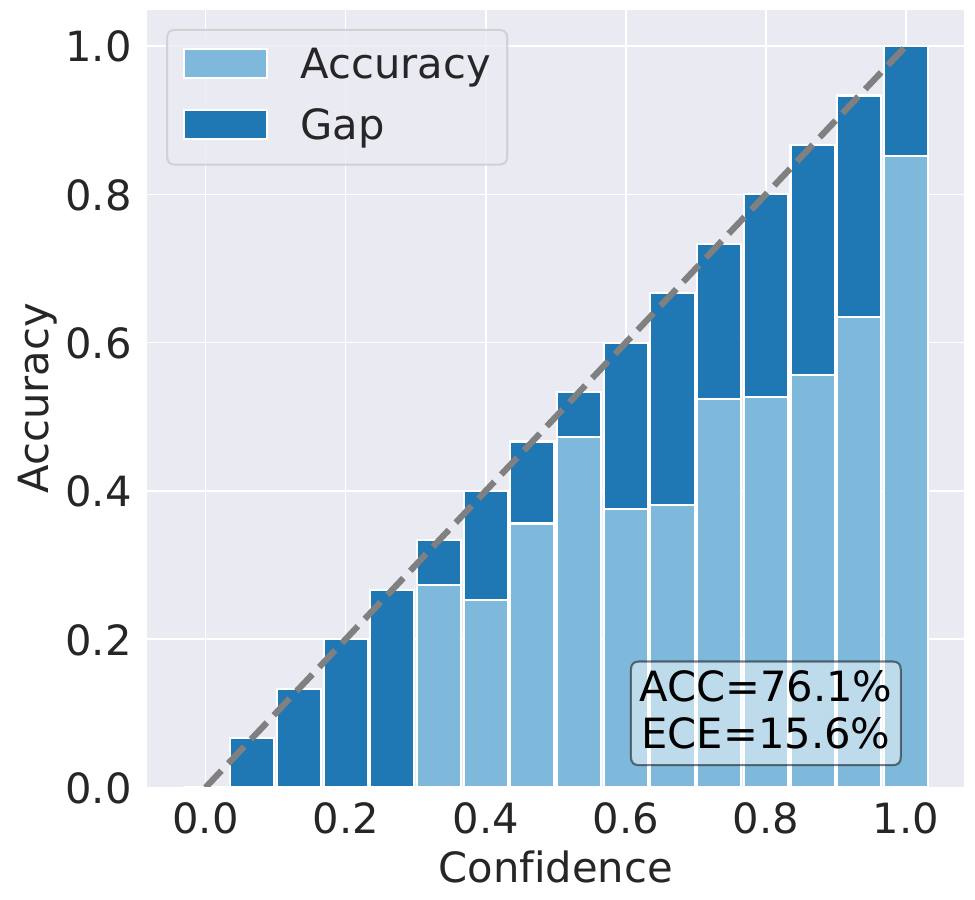}
        \label{fig:ece-lt-cifar10-50-1e-5}
    }
    \subfigure[CE ($\lambda=5e-5$)]{
        \includegraphics[scale=0.2]{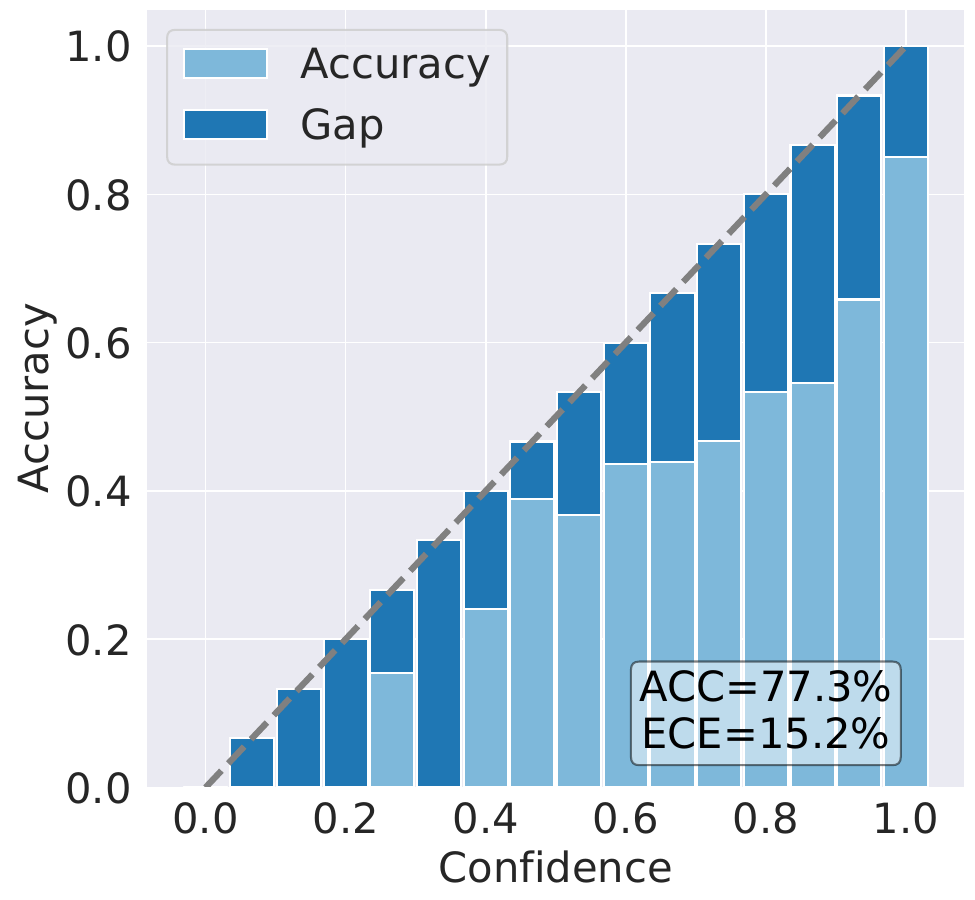}
        \label{fig:ece-lt-cifar10-50-5e-5}
    }
    \subfigure[CE ($\lambda=0$)]{
        \includegraphics[scale=0.2]{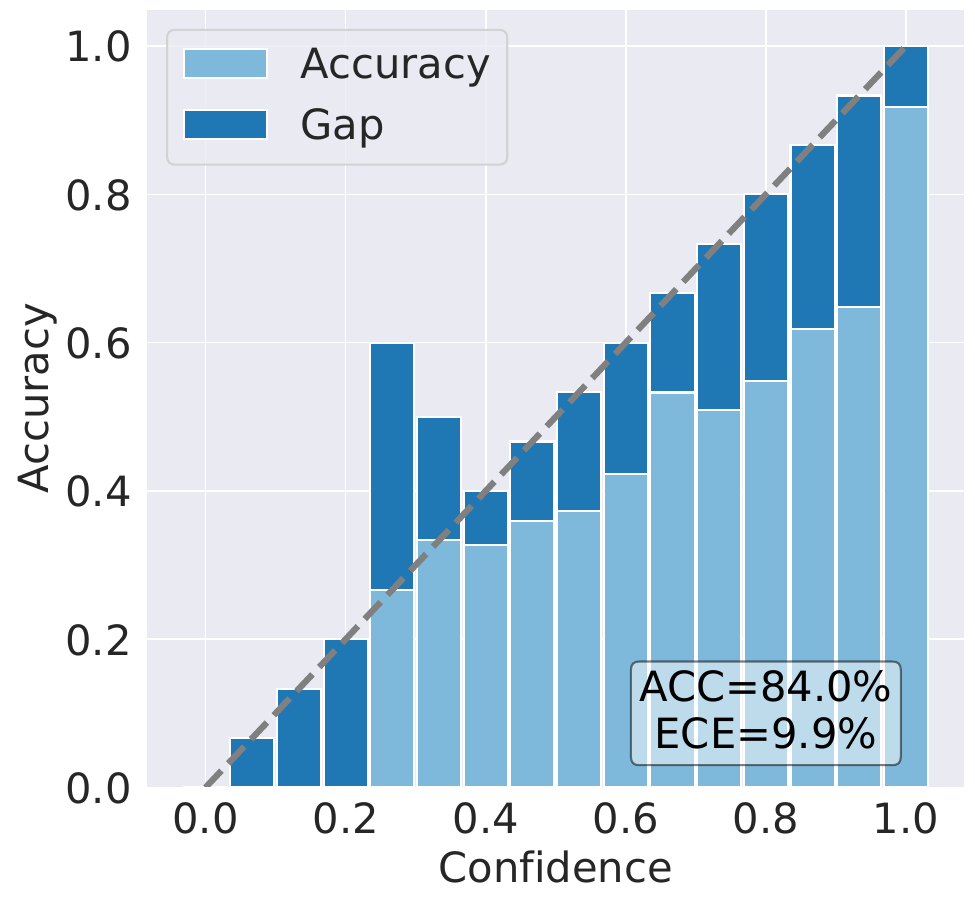}
        \label{fig:ece-lt-cifar10-20-0}
    }
    \subfigure[CE ($\lambda=5e-6$)]{
        \includegraphics[scale=0.2]{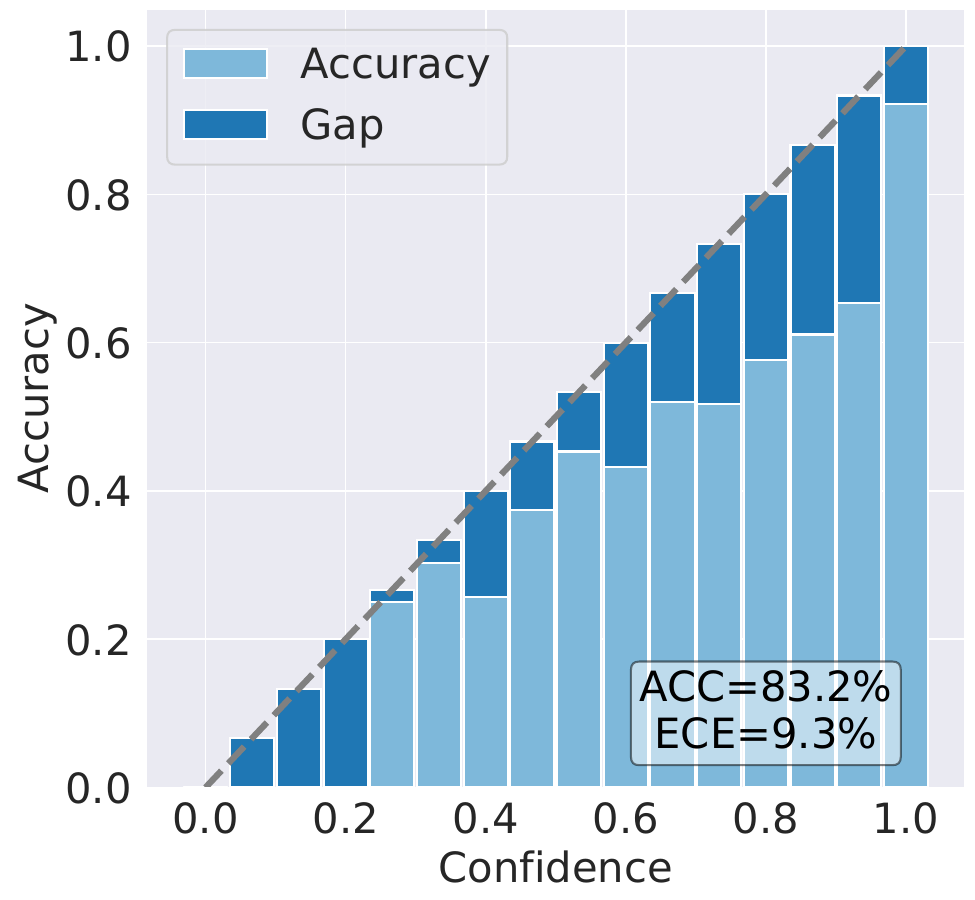}
        \label{fig:ece-lt-cifar10-20-5e-6}
    }
    \subfigure[CE ($\lambda=1e-5$)]{
        \includegraphics[scale=0.2]{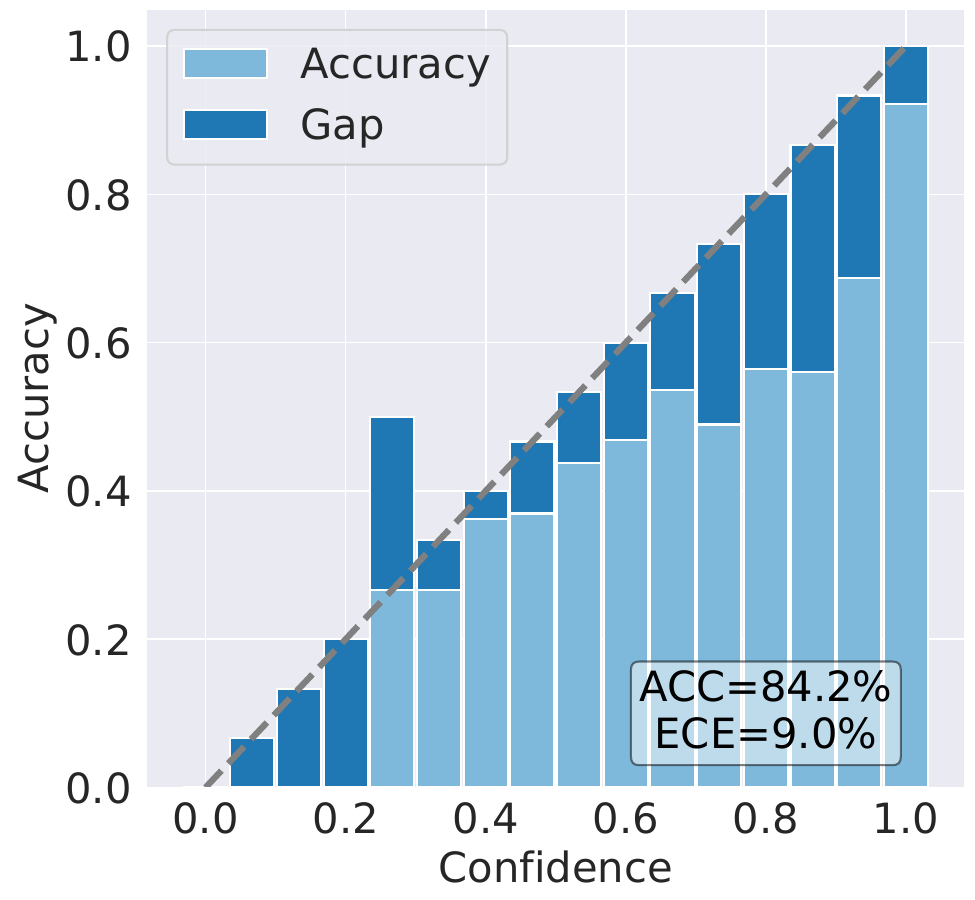}
        \label{fig:ece-lt-cifar10-20-1e-5}
    }
    \subfigure[CE ($\lambda=5e-5$)]{
        \includegraphics[scale=0.2]{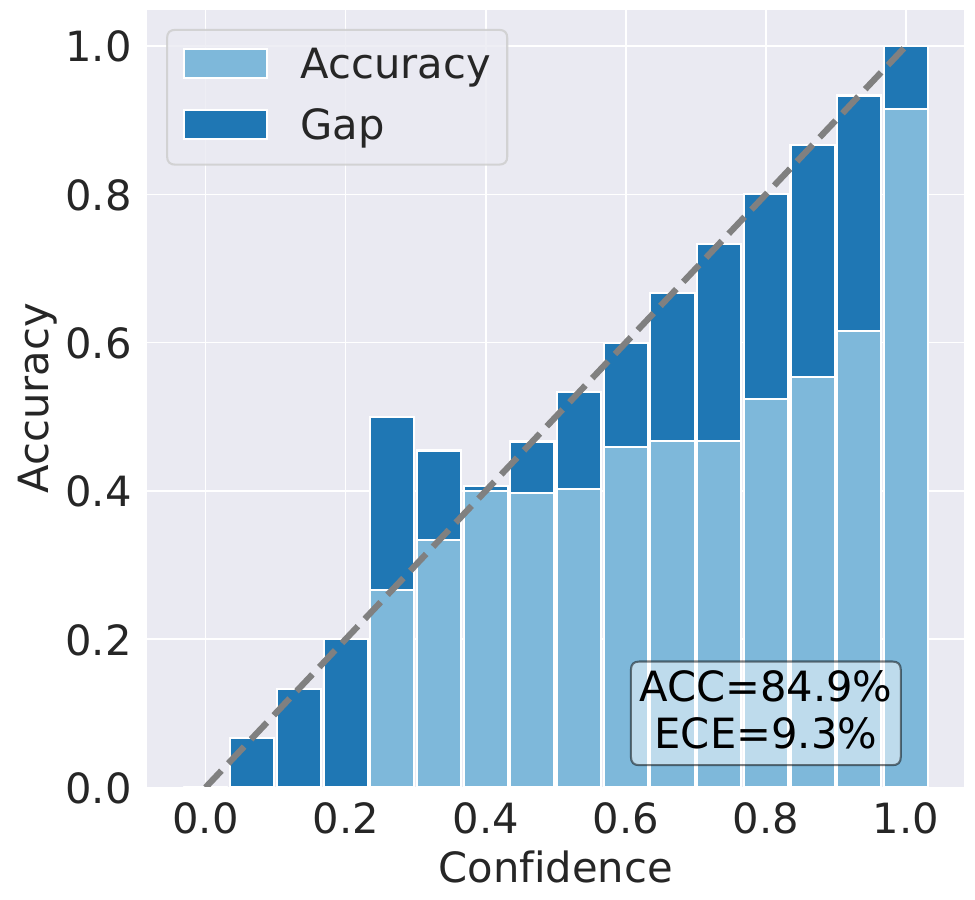}
        \label{fig:ece-lt-cifar10-20-5e-5}
    }
    \subfigure[CE ($\lambda=0$)]{
        \includegraphics[scale=0.2]{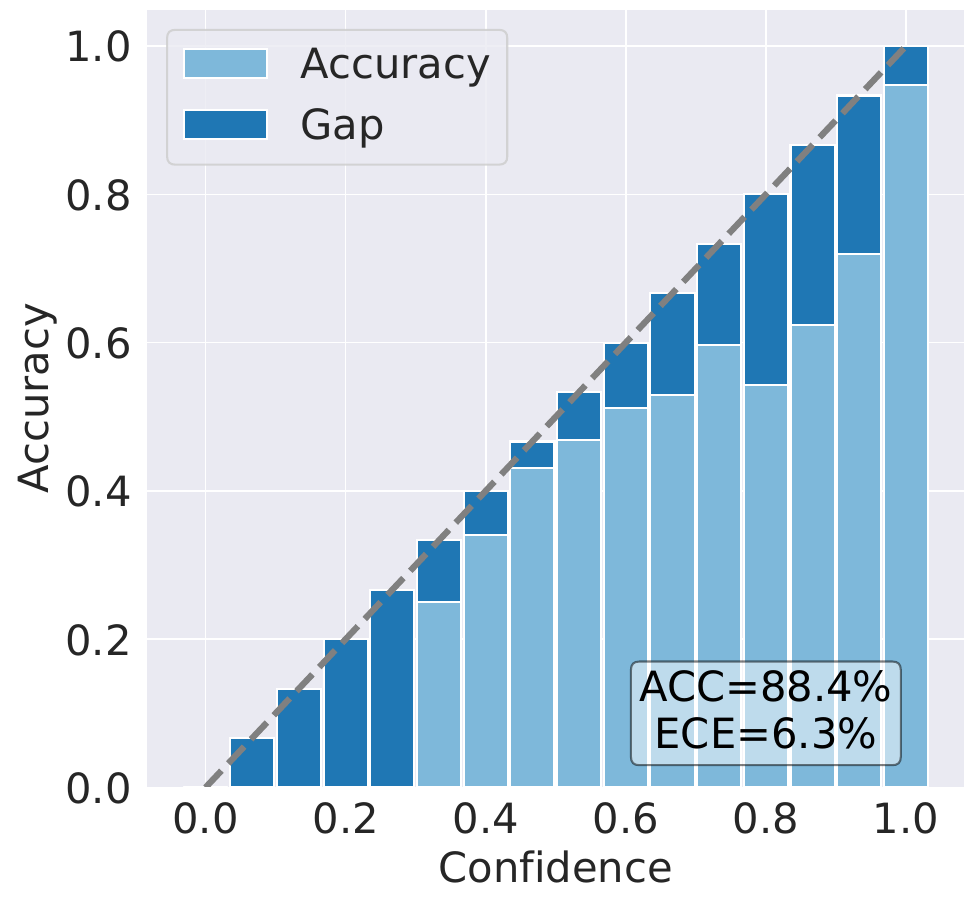}
        \label{fig:ece-lt-cifar10-10-0}
    }
    \subfigure[CE ($\lambda=5e-6$)]{
        \includegraphics[scale=0.2]{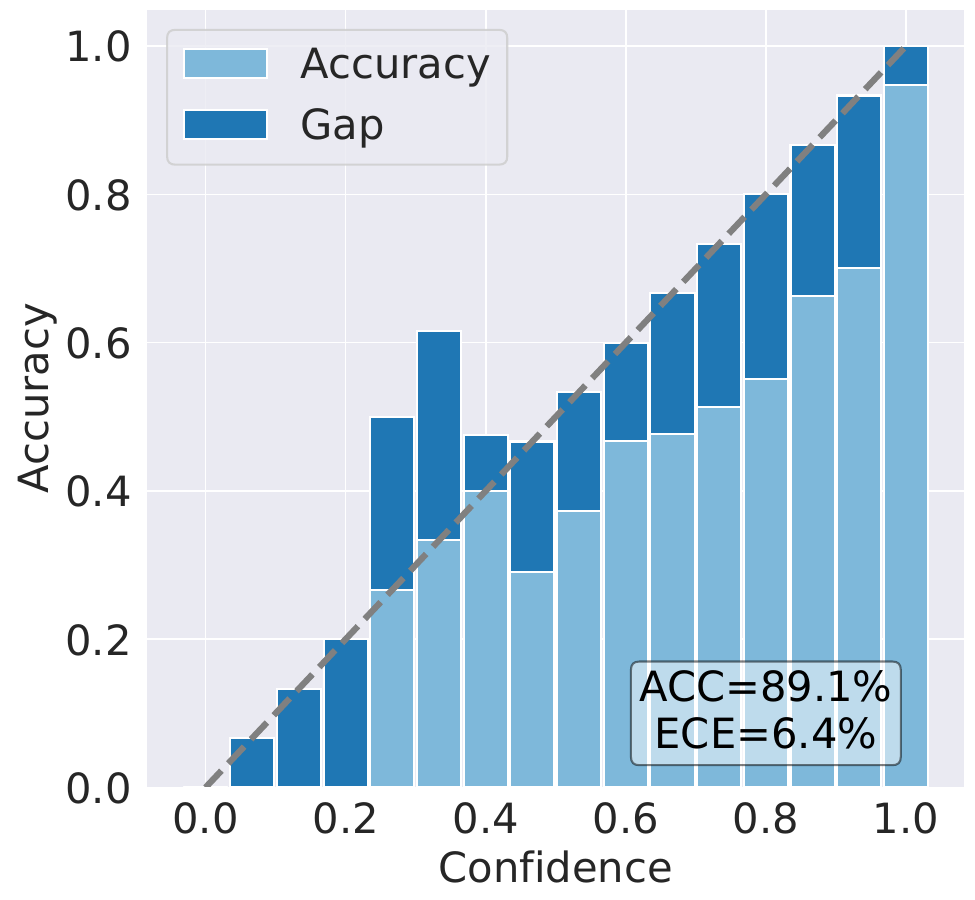}
        \label{fig:ece-lt-cifar10-10-5e-6}
    }
    \subfigure[CE ($\lambda=1e-5$)]{
        \includegraphics[scale=0.2]{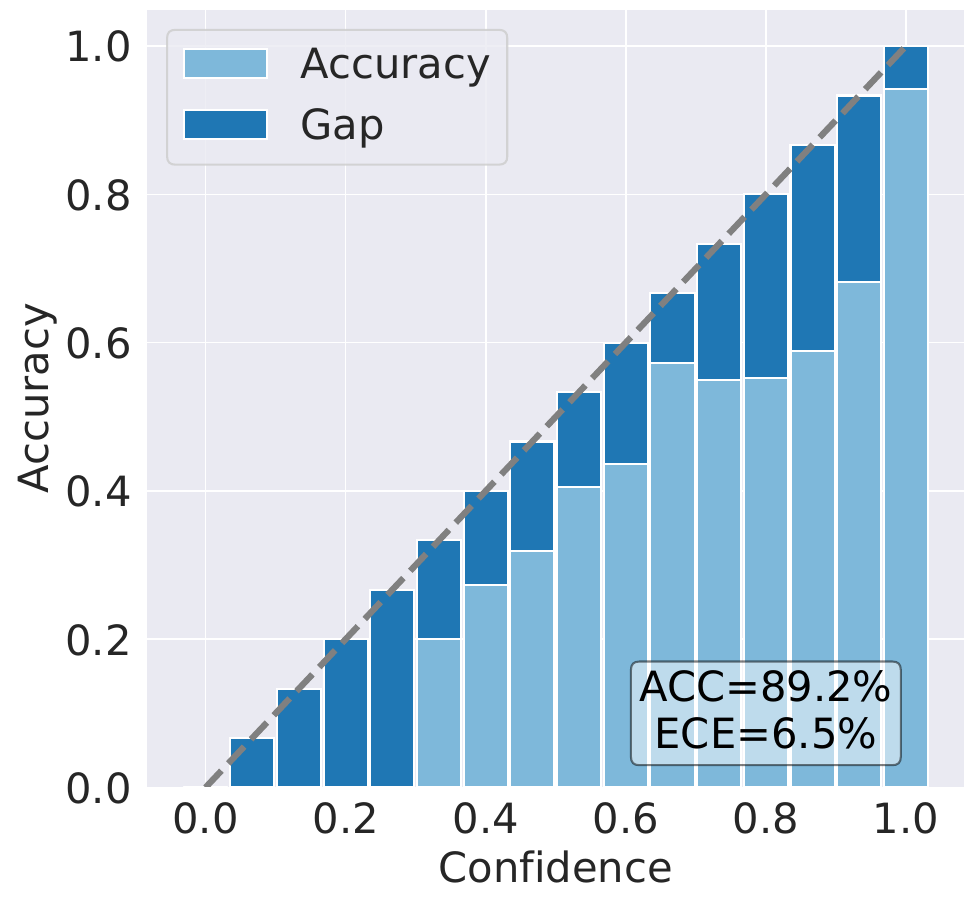}
        \label{fig:ece-lt-cifar10-10-1e-5}
    }
    \subfigure[CE ($\lambda=5e-5$)]{
        \includegraphics[scale=0.2]{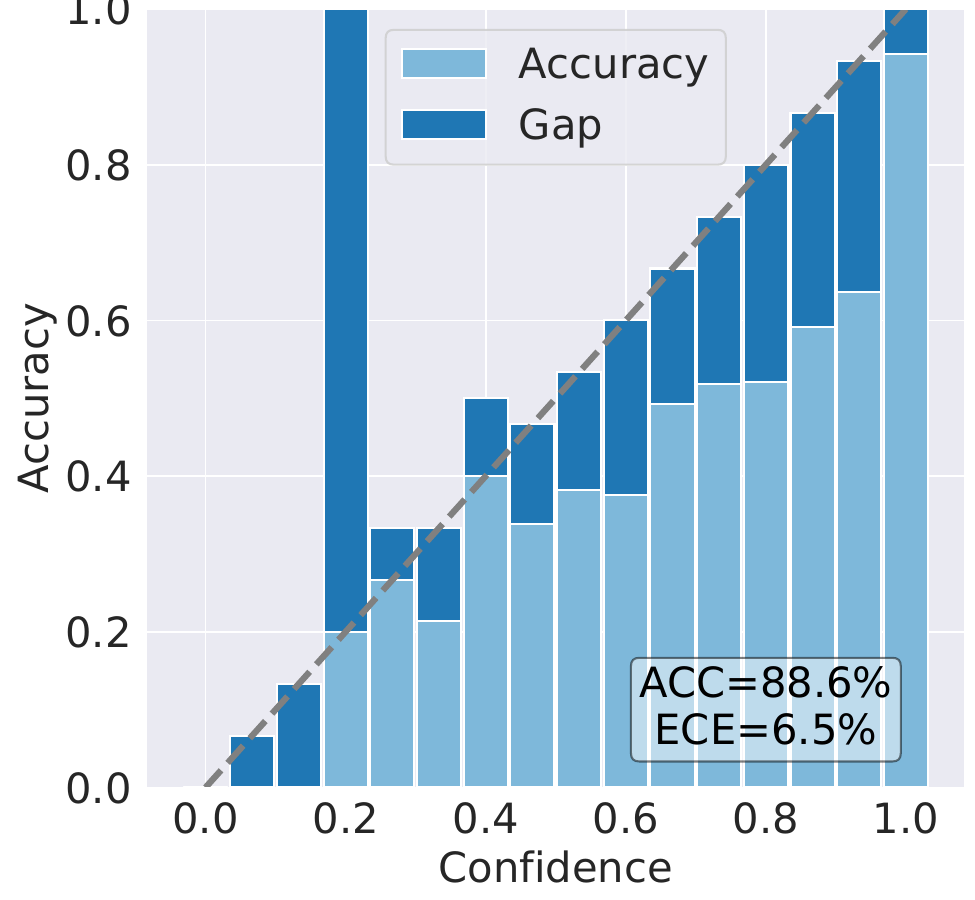}
        \label{fig:ece-lt-cifar10-10-5e-5}
    }
    \caption{Reliability diagrams of ResNet-34 \citep{he2016deep} trained by CE on CIFAR-10-LT with imbalance ratio $\rho\in\{100, 50, 20, 10\}$ under different explicit feature regularization ($\lambda\in\{0, 5e-6, 1e-5, 5e-5\}$). As can be seen, an appropriate larger weight decay can improve both accuracy and confidence}
    \label{fig:ece-lt-cifar10}
\end{figure}

\begin{figure}
    \centering
    \subfigure[CE ($\lambda=0$)]{
    \includegraphics[scale=0.2]{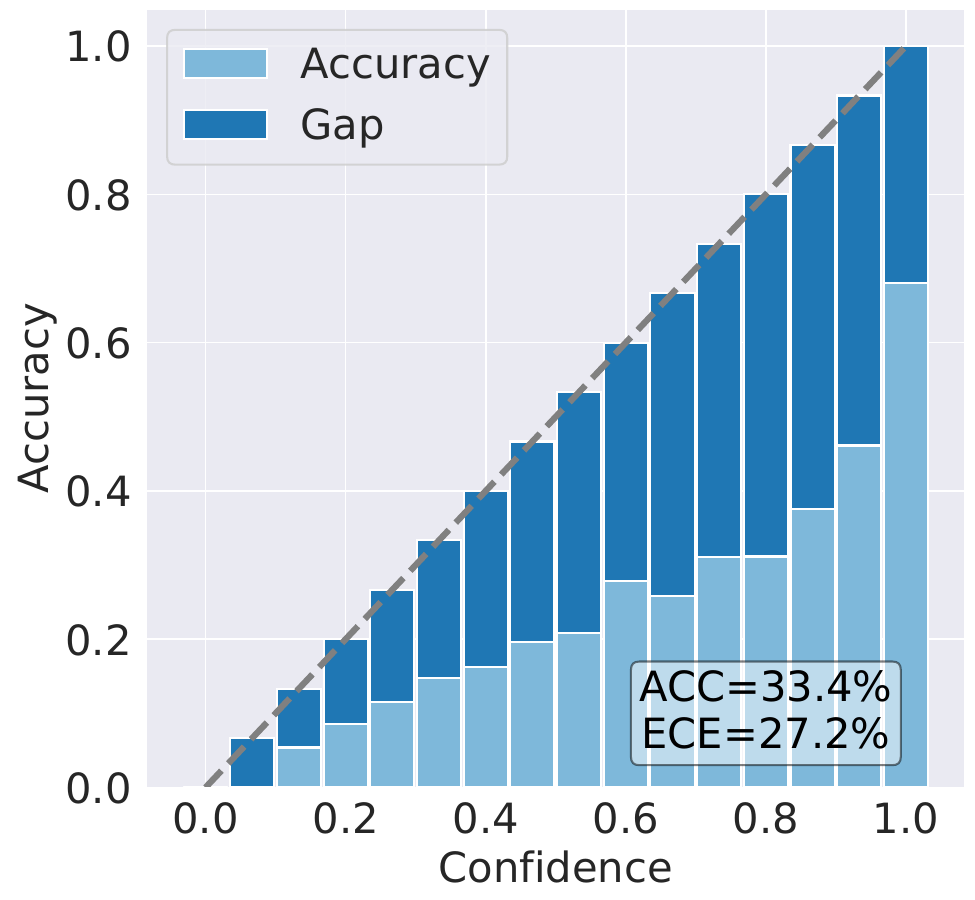}
    \label{fig:ece-lt-cifar100-100-0}
    }
    \subfigure[CE ($\lambda=5e-6$)]{
    \includegraphics[scale=0.2]{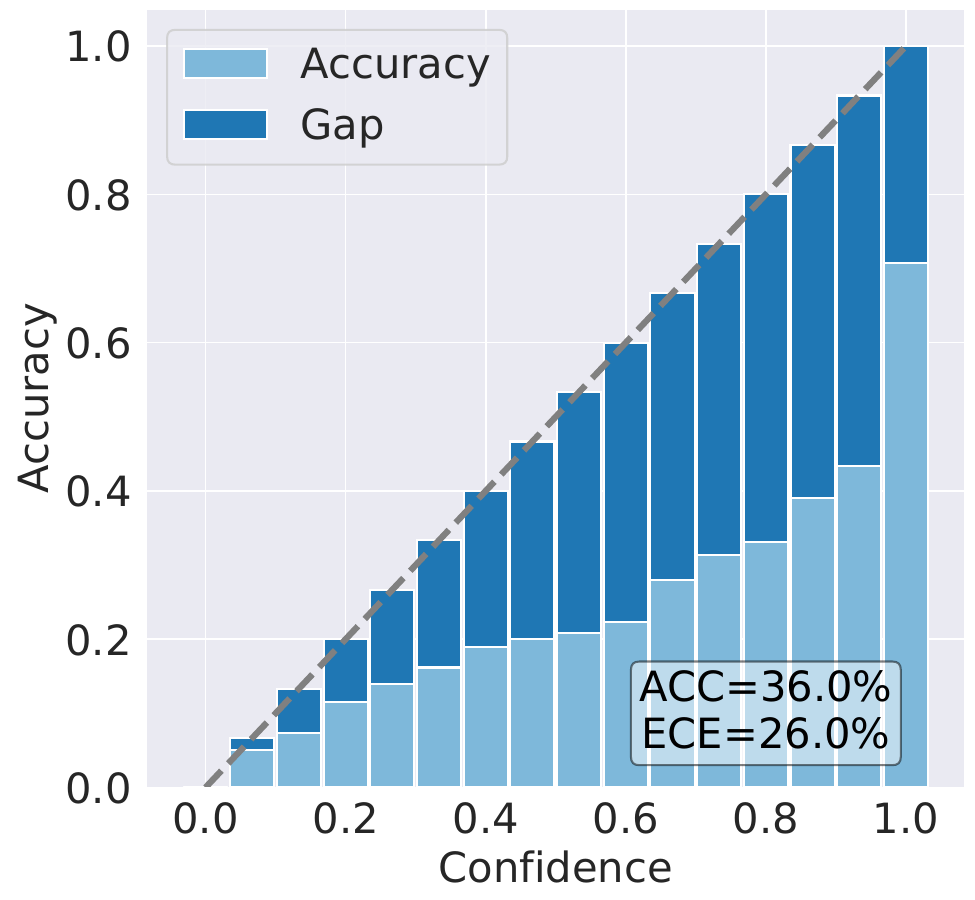}
    \label{fig:ece-lt-cifar100-100-5e-6}
    }
    \subfigure[CE ($\lambda=1e-5$)]{
    \includegraphics[scale=0.2]{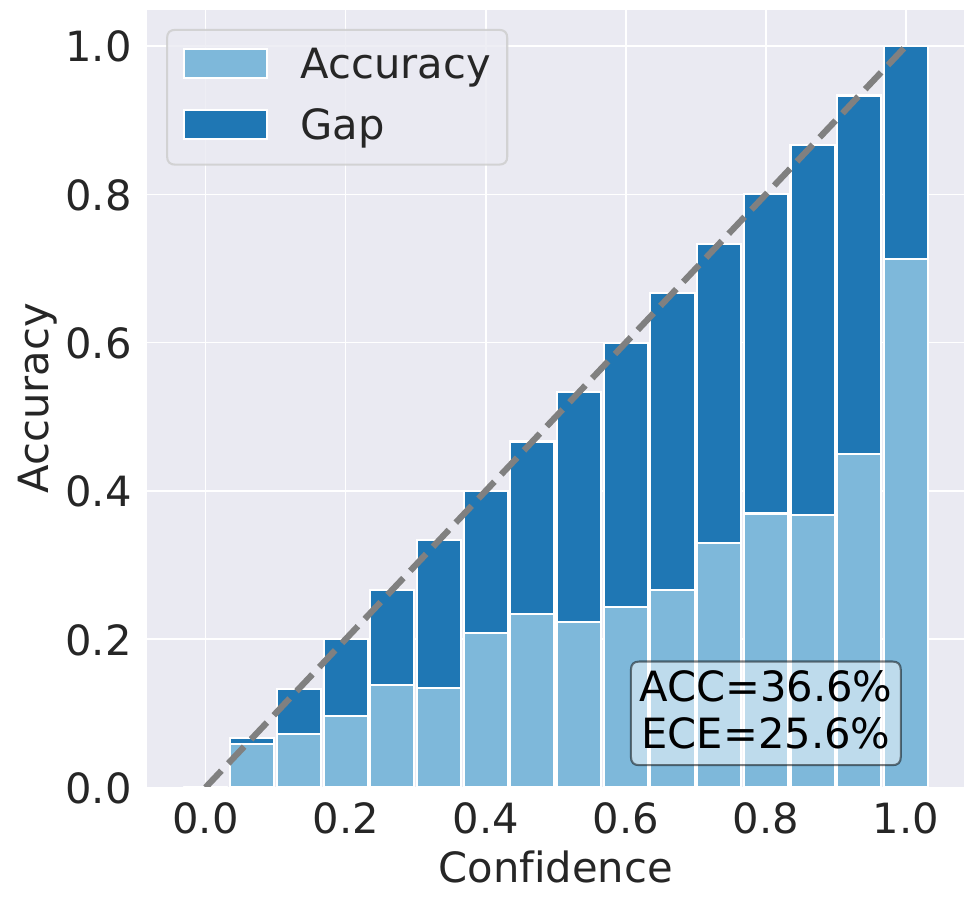}
    \label{fig:ece-lt-cifar100-100-1e-5}
    }
    \subfigure[CE ($\lambda=5e-5$)]{
    \includegraphics[scale=0.2]{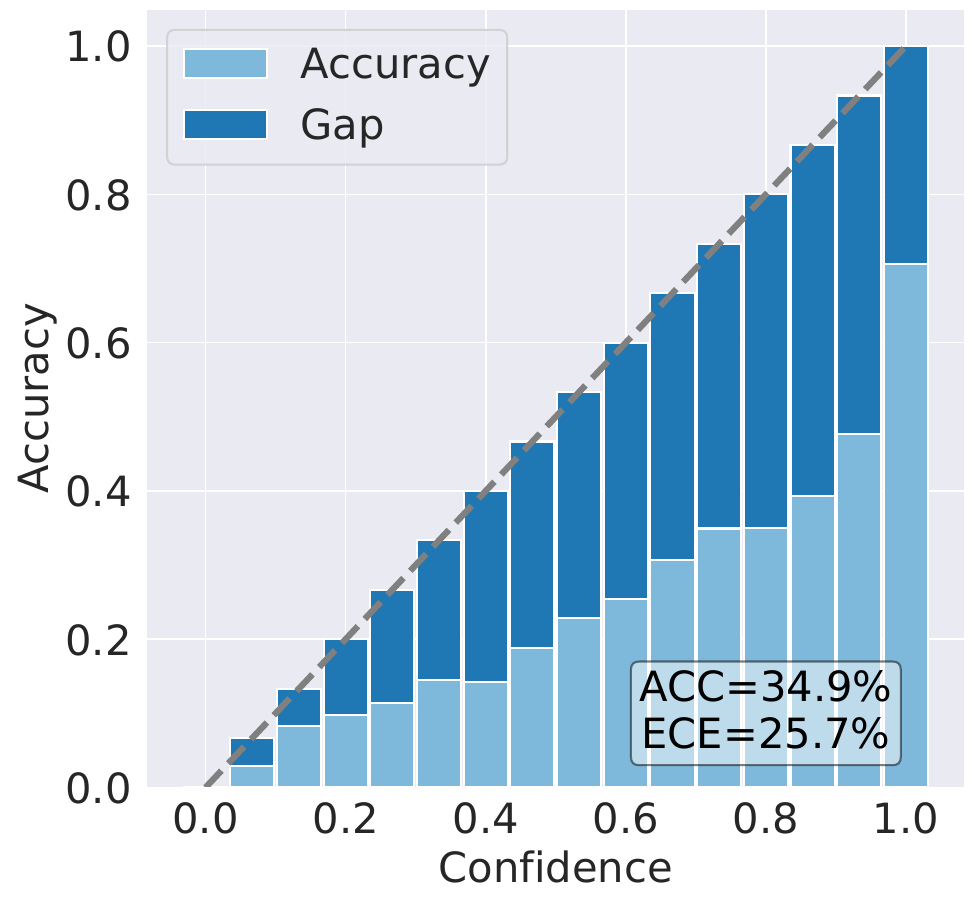}
    \label{fig:ece-lt-cifar100-100-5e-5}
    }
    \\
    \subfigure[CE ($\lambda=0$)]{
    \includegraphics[scale=0.2]{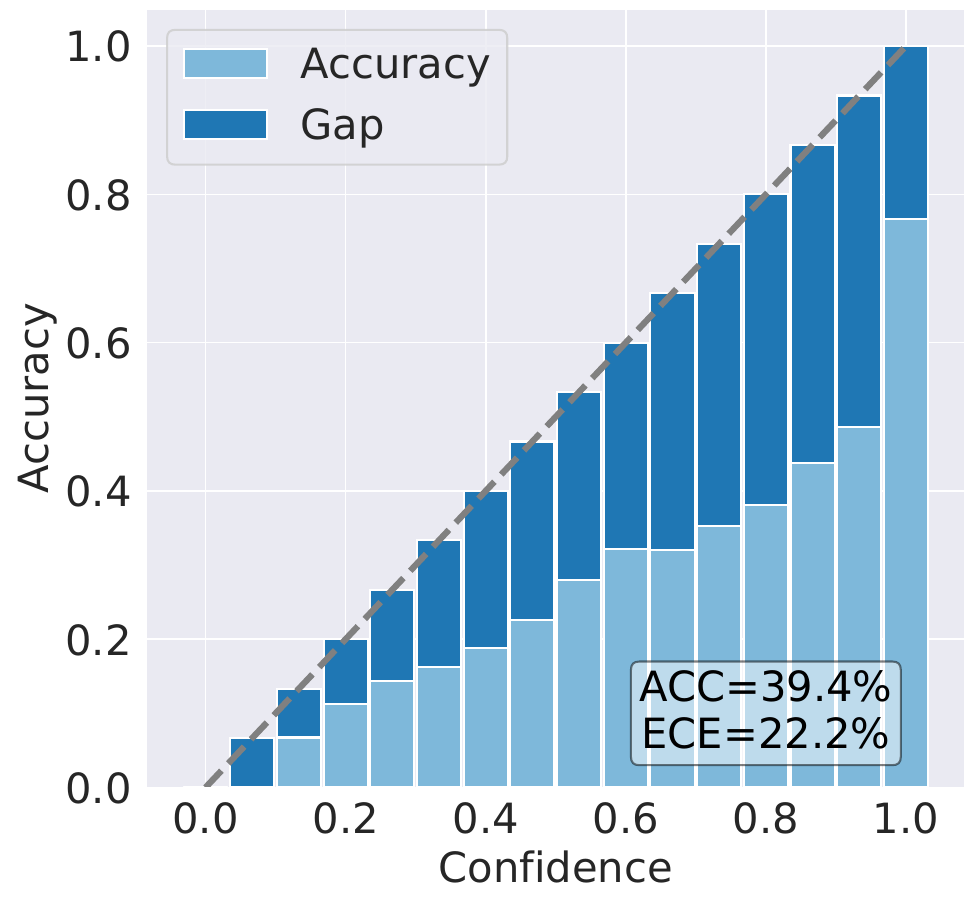}
    \label{fig:ece-lt-cifar100-50-0}
    }
    \subfigure[CE ($\lambda=5e-6$)]{
    \includegraphics[scale=0.2]{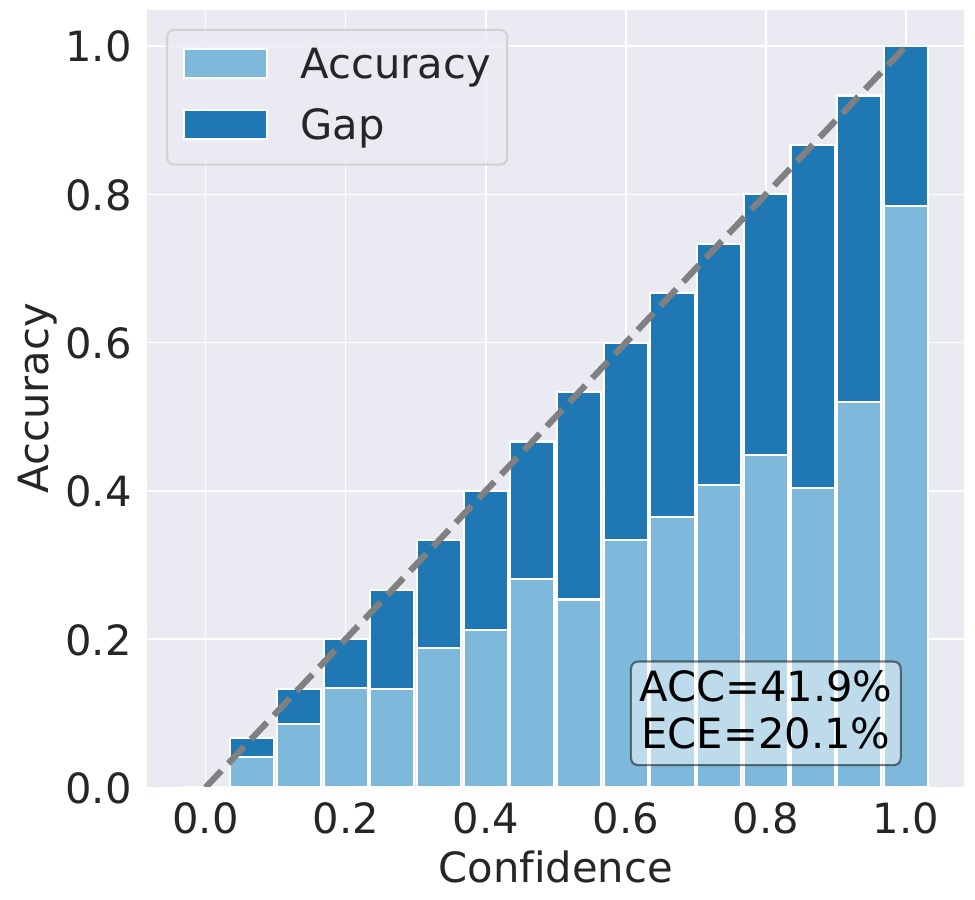}
    \label{fig:ece-lt-cifar100-50-5e-6}
    }
    \subfigure[CE ($\lambda=1e-5$)]{
    \includegraphics[scale=0.2]{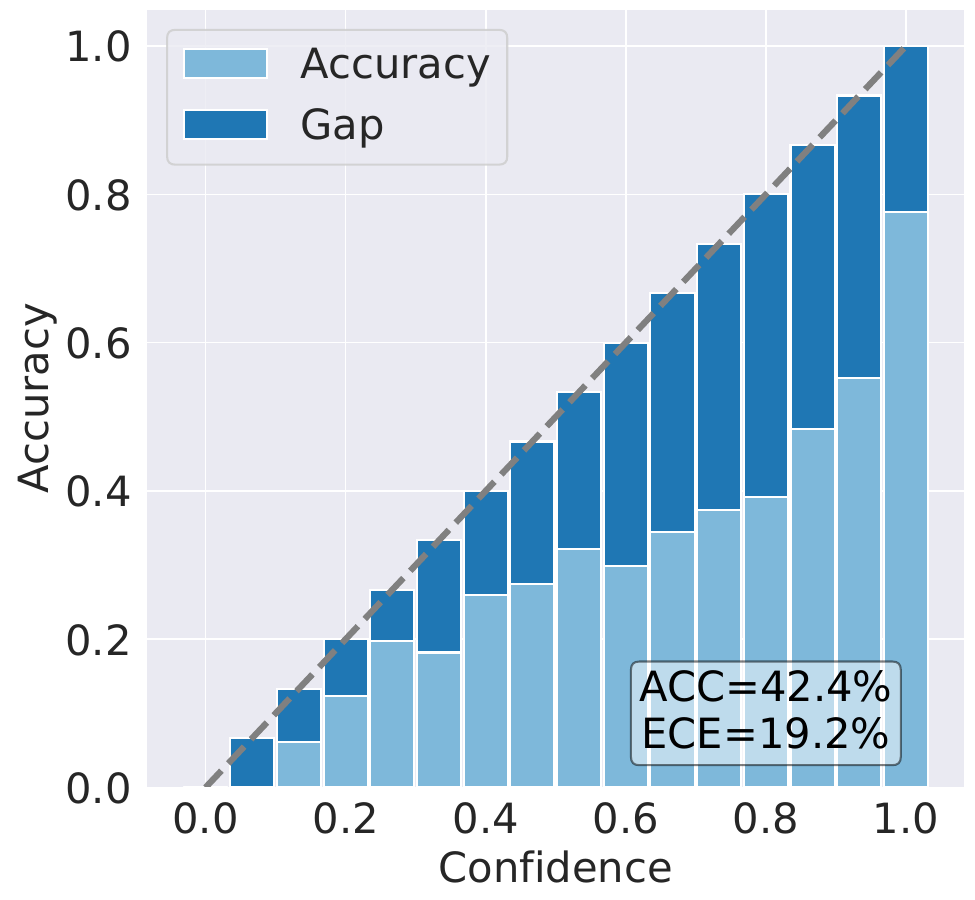}
    \label{fig:ece-lt-cifar100-50-1e-5}
    }
    \subfigure[CE ($\lambda=5e-5$)]{
    \includegraphics[scale=0.2]{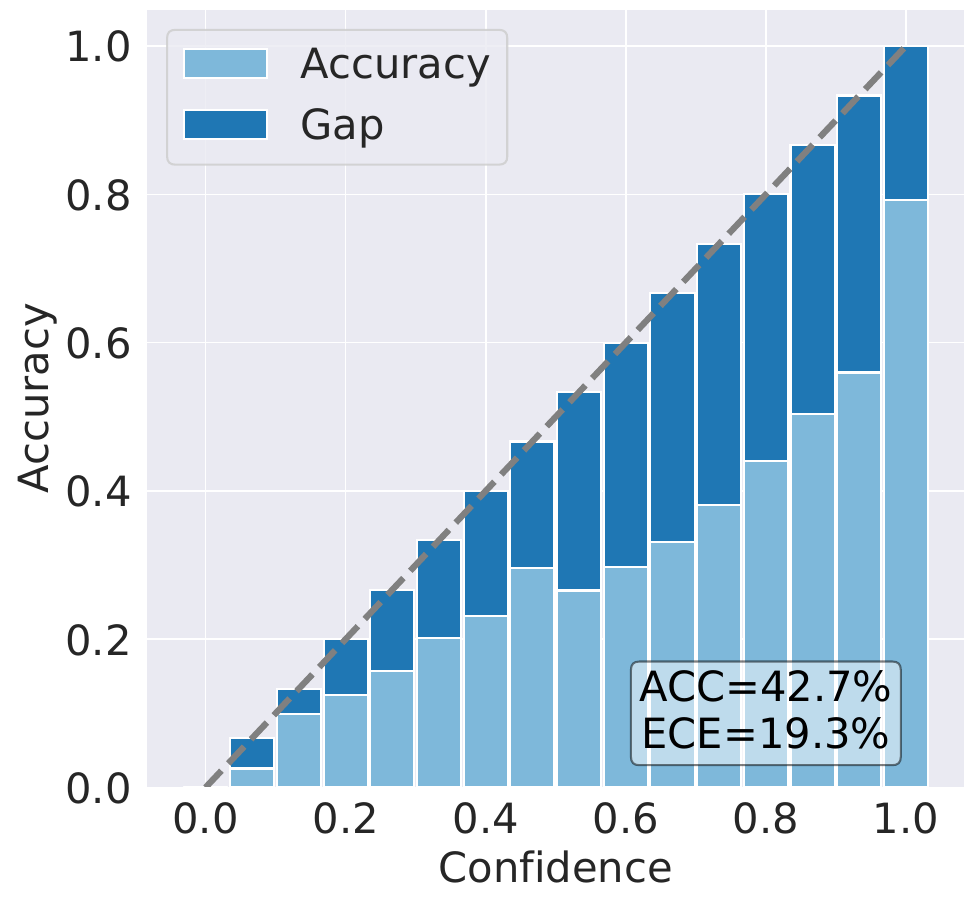}
    \label{fig:ece-lt-cifar100-50-5e-5}
    }
    \\
    \subfigure[CE ($\lambda=0$)]{
    \includegraphics[scale=0.2]{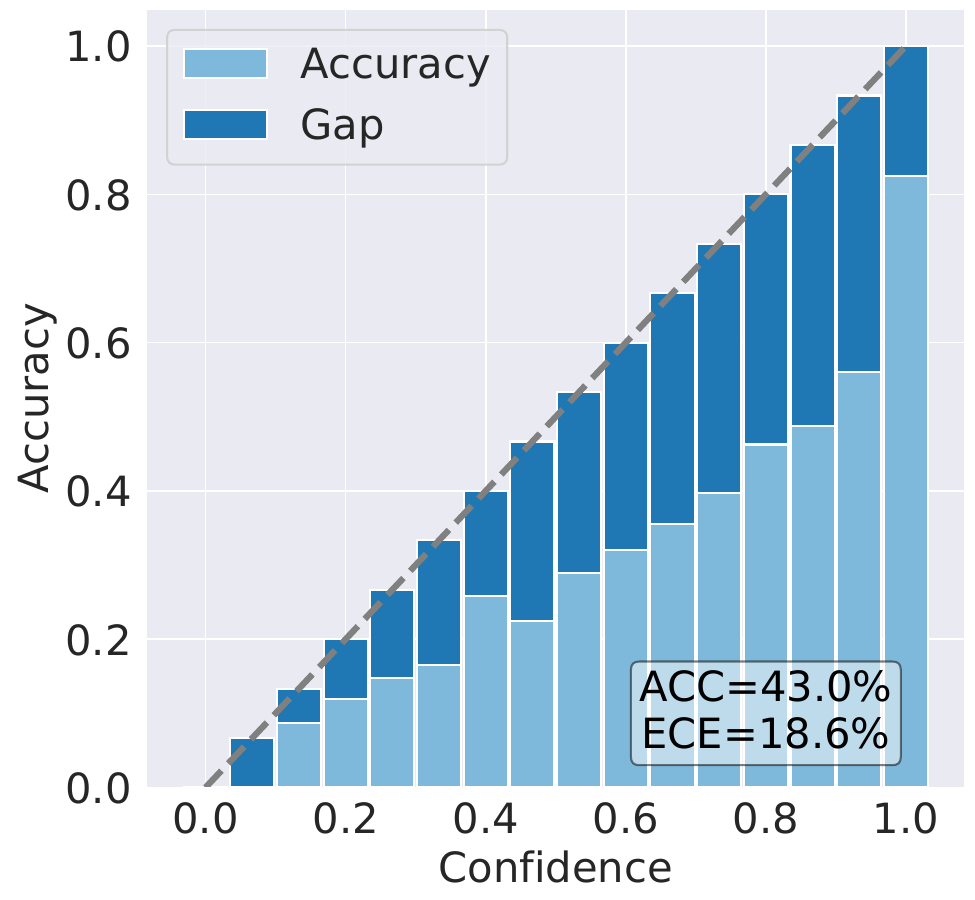}
    \label{fig:ece-lt-cifar100-20-0}
    }
    \subfigure[CE ($\lambda=5e-6$)]{
    \includegraphics[scale=0.2]{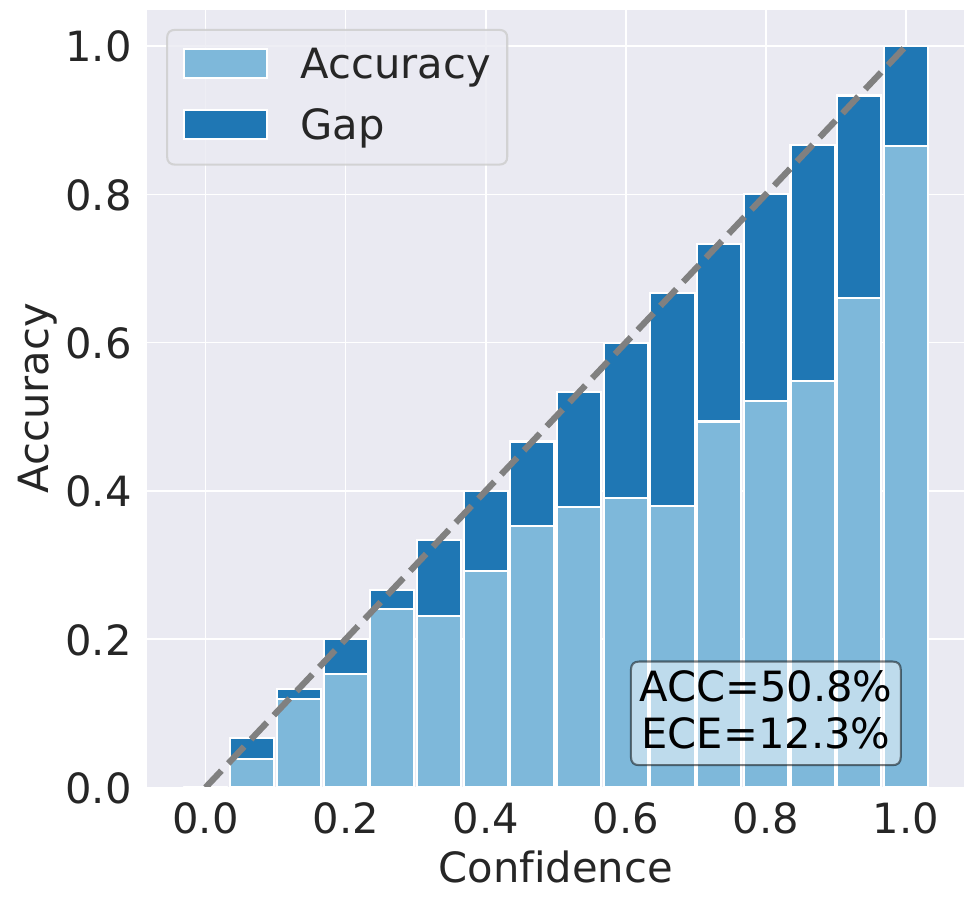}
    \label{fig:ece-lt-cifar100-20-5e-6}
    }
    \subfigure[CE ($\lambda=1e-5$)]{
    \includegraphics[scale=0.2]{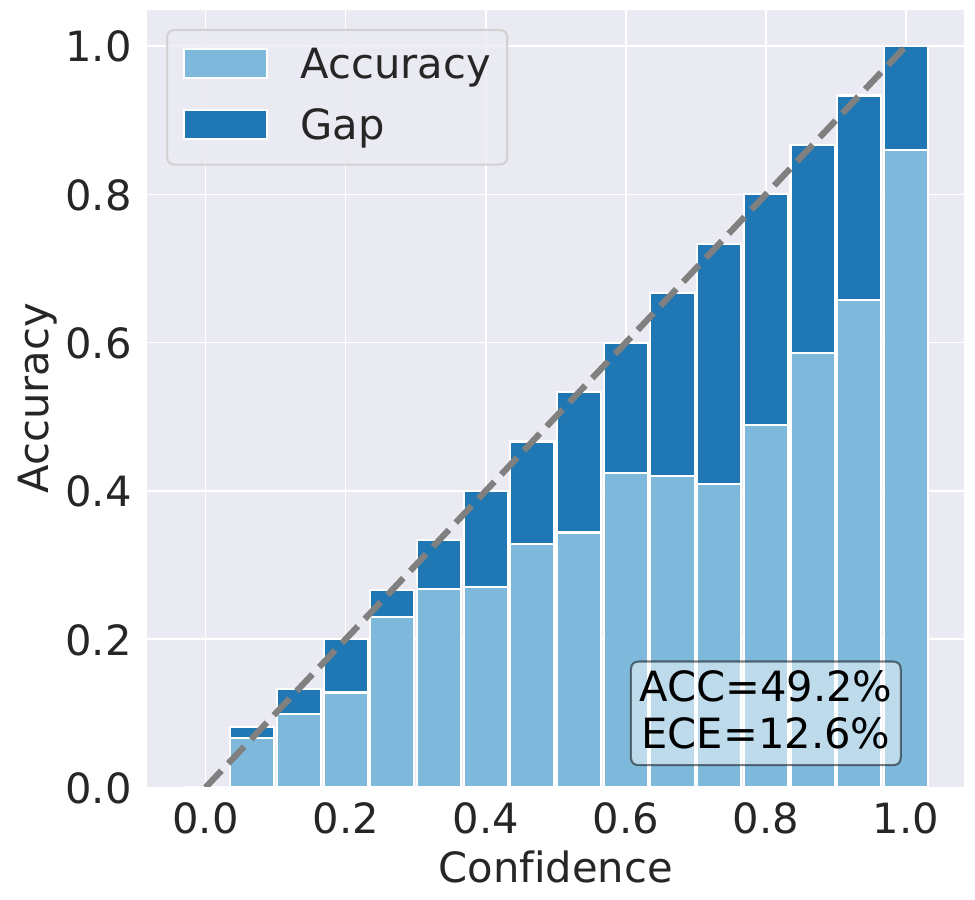}
    \label{fig:ece-lt-cifar100-20-1e-5}
    }
    \subfigure[CE ($\lambda=5e-5$)]{
    \includegraphics[scale=0.2]{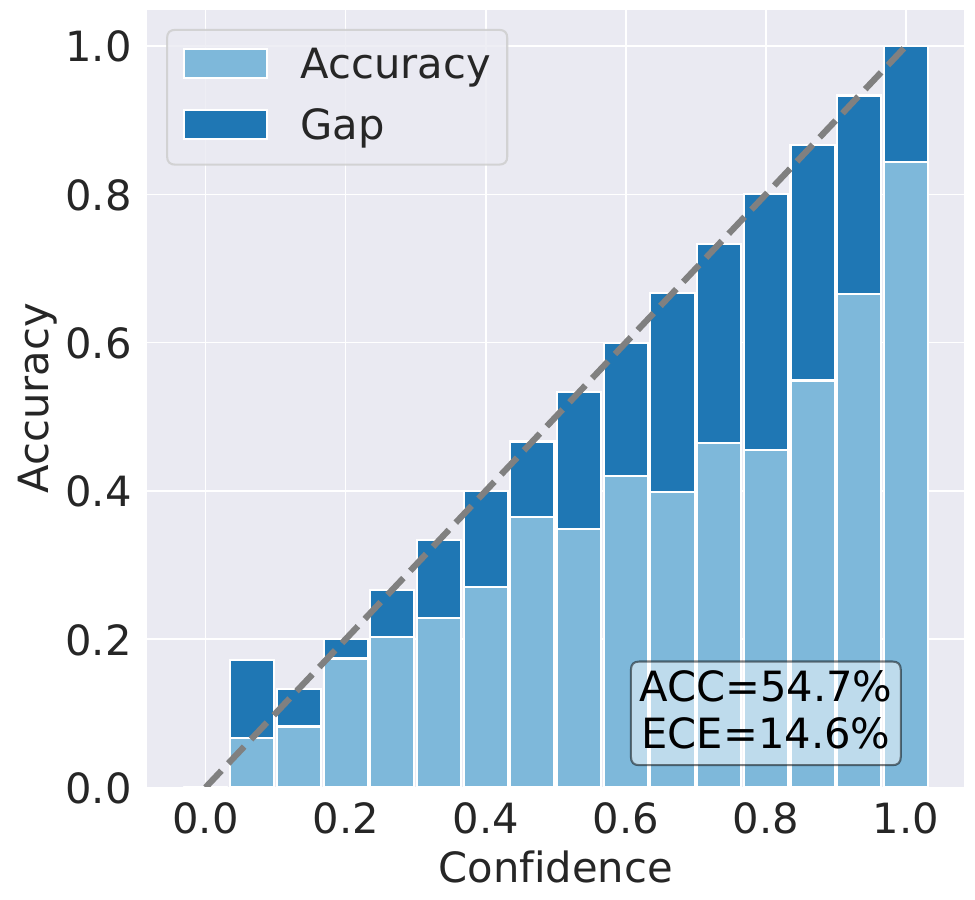}
    \label{fig:ece-lt-cifar100-20-5e-5}
    }
    \\
    \subfigure[CE ($\lambda=0$)]{
    \includegraphics[scale=0.2]{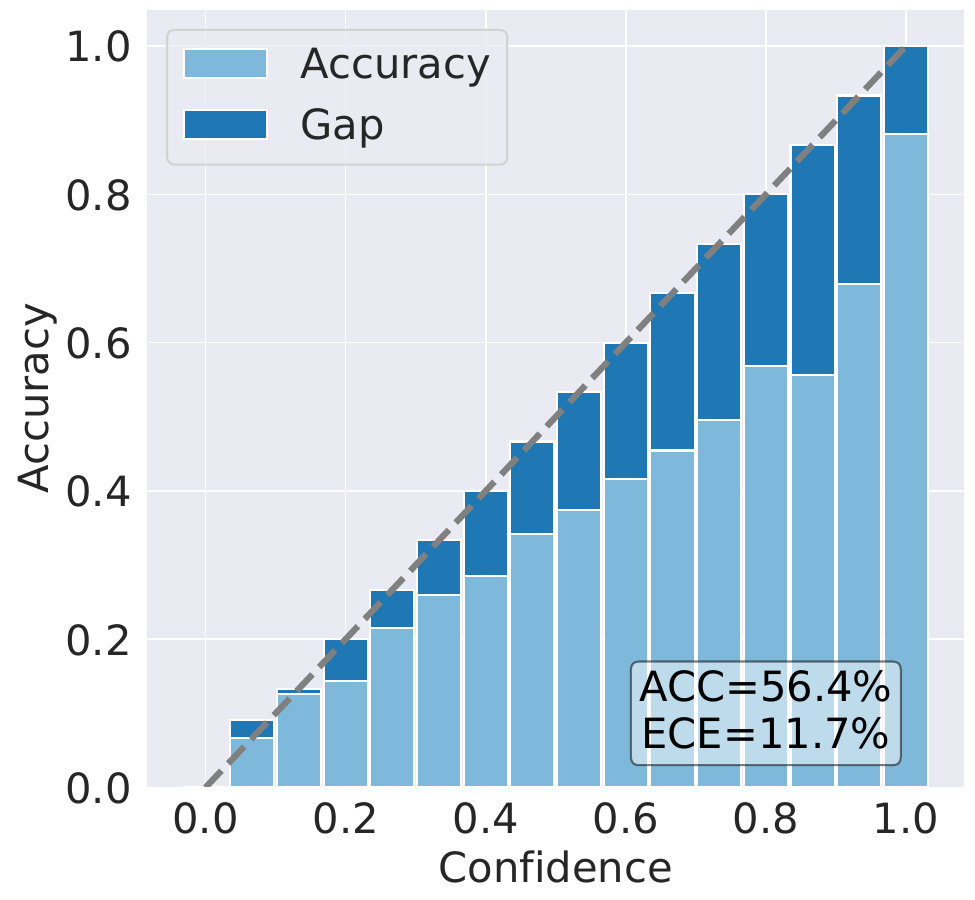}
    \label{fig:ece-lt-cifar100-10-0}
    }
    \subfigure[CE ($\lambda=5e-6$)]{
    \includegraphics[scale=0.2]{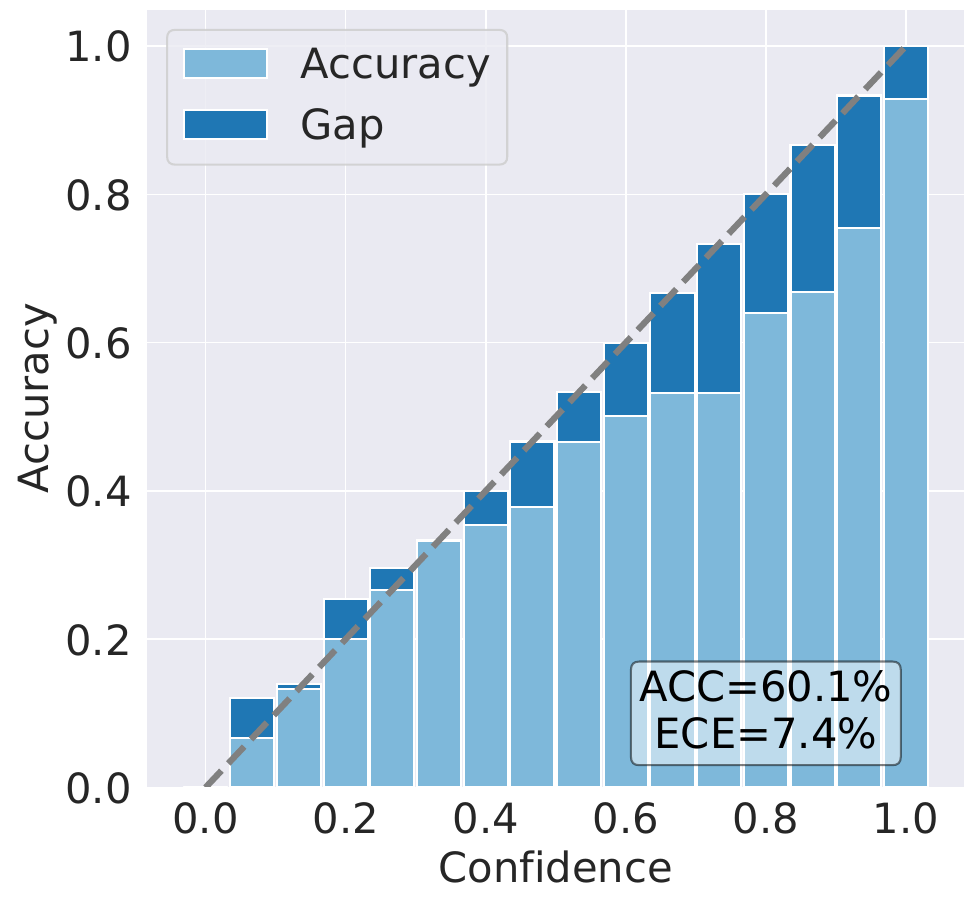}
    \label{fig:ece-lt-cifar100-10-5e-6}
    }
    \subfigure[CE ($\lambda=1e-5$)]{
    \includegraphics[scale=0.2]{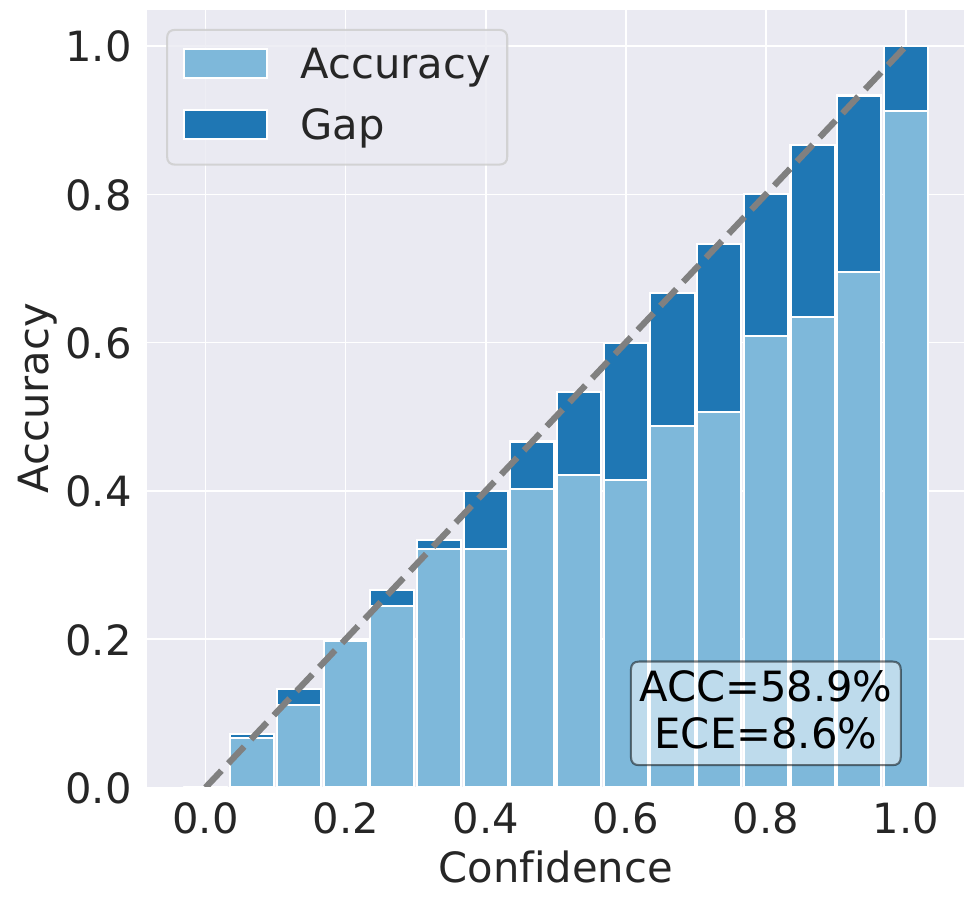}
    \label{fig:ece-lt-cifar100-10-1e-5}
    }
    \subfigure[CE ($\lambda=5e-5$)]{
    \includegraphics[scale=0.2]{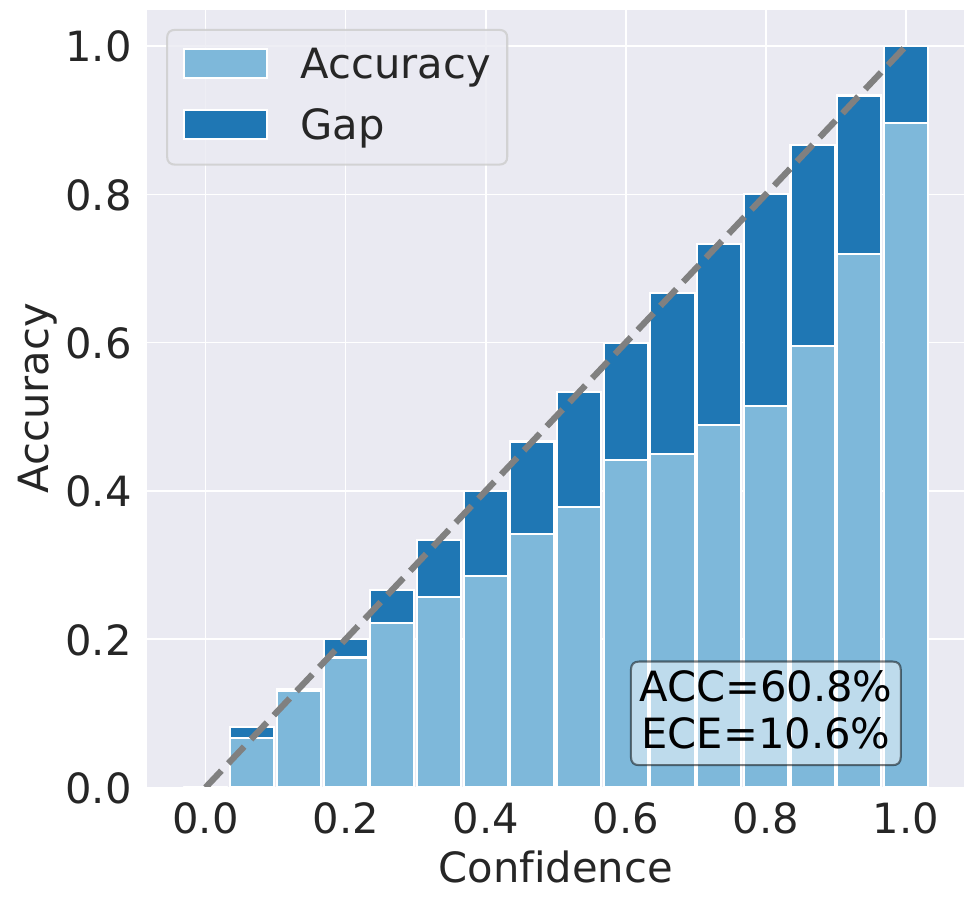}
    \label{fig:ece-lt-cifar100-10-5e-5}
    }
    \caption{Reliability diagrams of ResNet-34 \citep{he2016deep} trained by CE on CIFAR-100-LT with imbalance ratio $\rho\in\{100, 50, 20, 10\}$ under different explicit feature regularization ($\lambda\in\{0.0, 5e-6, 1e-5, 5e-5\}$), where ECE denotes the expected calibration error \citep{zhong2021improving}.  As can be seen, an appropriate larger weight decay can improve both accuracy and confidence calibration. }
    \label{fig:ece-lt-cifar100}
\end{figure}

\begin{figure}
    \centering
    \subfigure[Feature Norm]{
        \includegraphics[scale=0.48]{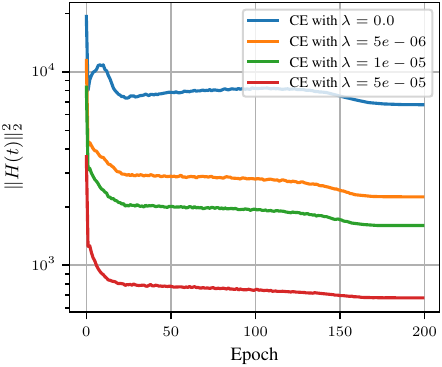}
        \label{fig:ce-feat-cifar10}
    }
    \subfigure[Train Loss]{
        \includegraphics[scale=0.48]{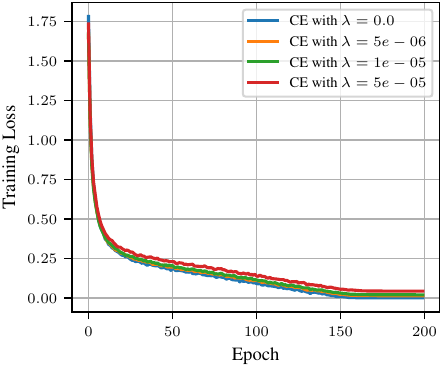}
        \label{fig:ce-loss-cifar10}
    }
    \subfigure[Validation Accuracy]{
        \includegraphics[scale=0.48]{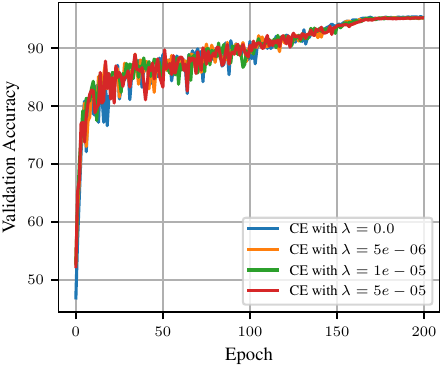}
        \label{fig:ce-top1-cifar10}
    }
    \subfigure[Feature Norm]{
        \includegraphics[scale=0.48]{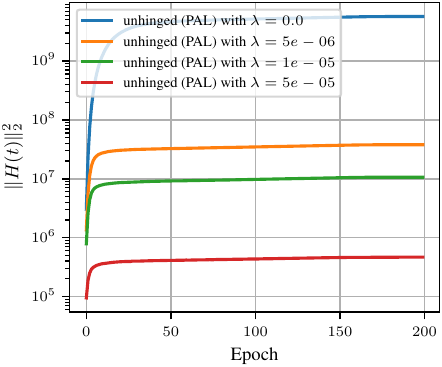}
        \label{fig:fsm-feat-cifar10}
    }
    \subfigure[Train Loss]{
        \includegraphics[scale=0.48]{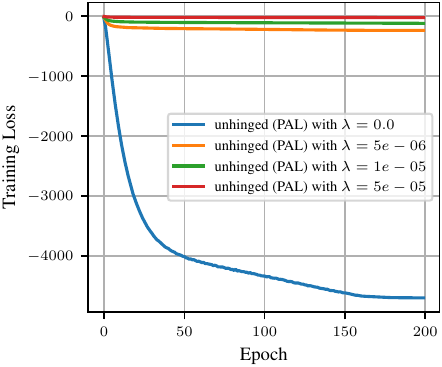}
        \label{fig:fsm-loss-cifar10}
    }
    \subfigure[Validation Accuracy]{
        \includegraphics[scale=0.48]{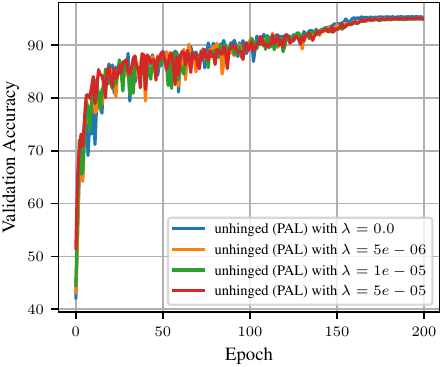}
        \label{fig:fsm-top1-cifar10}
    }
    \subfigure[Feature Norm]{
        \includegraphics[scale=0.48]{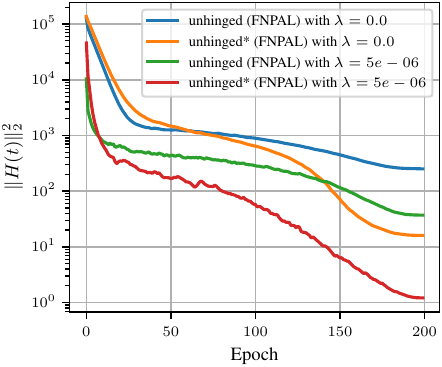}
        \label{fig:asm-feat-cifar10}
    }
    \subfigure[Train Loss]{
        \includegraphics[scale=0.48]{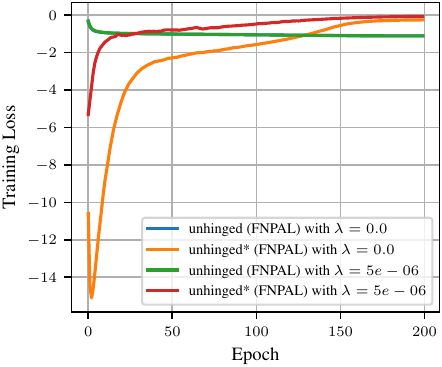}
        \label{fig:asm-loss-cifar10}
    }
    \subfigure[Validation Accuracy]{
        \includegraphics[scale=0.48]{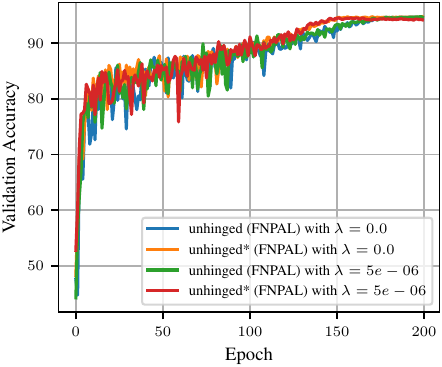}
        \label{fig:asm-top1-cifar10}
    }
    \caption{Behavior of visual classification on CIFAR-10 with CE, the unhinged loss (PAL) and the unhinged loss (FNPAL) under different weight decay coefficients.}
    \label{ce-fsm-asm-cifar10}
\end{figure}

\begin{figure}
    \centering
    \subfigure[Feature Norm]{
        \includegraphics[scale=0.48]{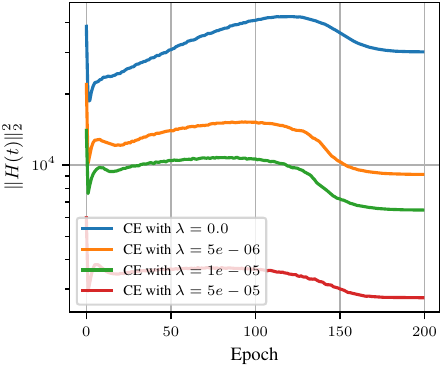}
        \label{fig:ce-feat-cifar100}
    }
    \subfigure[Train Loss]{
        \includegraphics[scale=0.48]{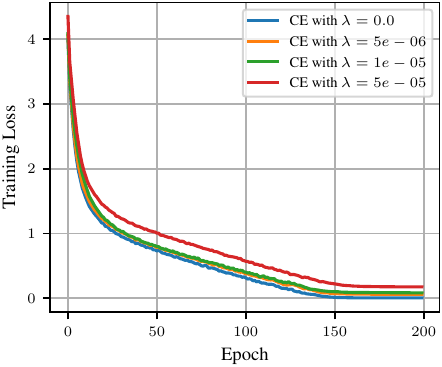}
        \label{fig:ce-loss-cifar100}
    }
    \subfigure[Validation Accuracy]{
        \includegraphics[scale=0.48]{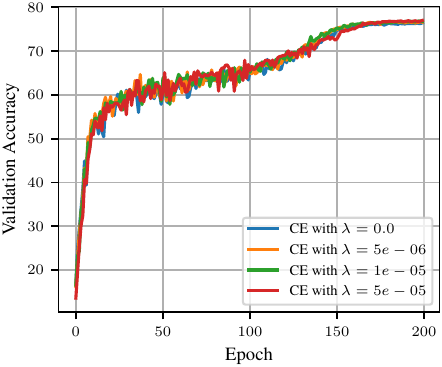}
        \label{fig:ce-top1-cifar100}
    }
   \subfigure[Feature Norm]{
        \includegraphics[scale=0.48]{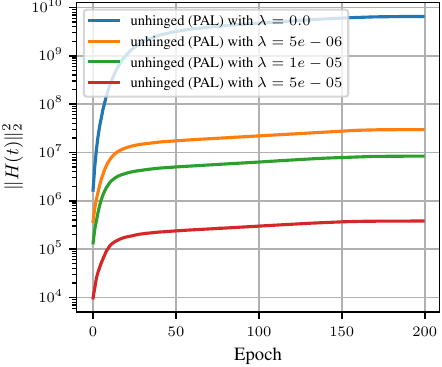}
        \label{fig:fsm-feat-cifar100}
    }
    \subfigure[Train Loss]{
        \includegraphics[scale=0.48]{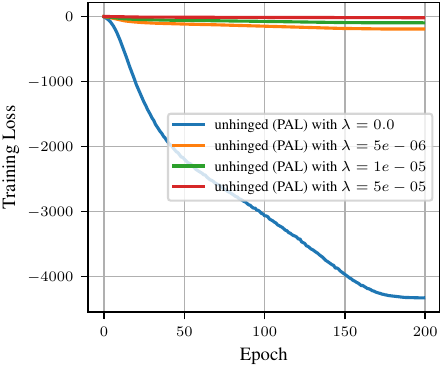}
        \label{fig:fsm-loss-cifar100}
    }
    \subfigure[Validation Accuracy]{
        \includegraphics[scale=0.48]{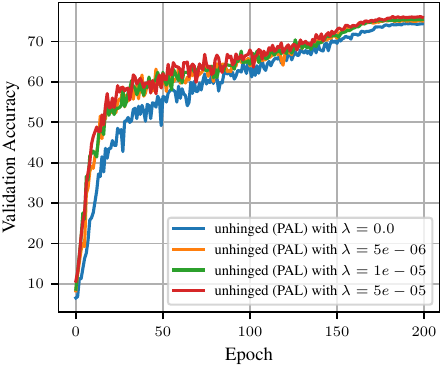}
        \label{fig:fsm-top1-cifar100}
    }
    \subfigure[Feature Norm]{
        \includegraphics[scale=0.48]{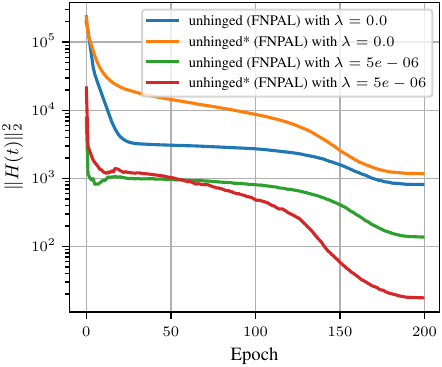}
        \label{fig:unhinged-feat-cifar100}
    }
    \subfigure[Train Loss]{
        \includegraphics[scale=0.48]{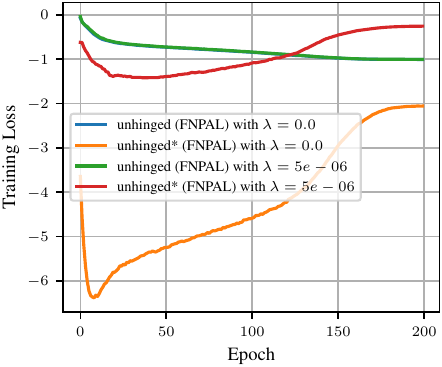}
        \label{fig:unhinged-loss-cifar100}
    }
    \subfigure[Validation Accuracy]{
        \includegraphics[scale=0.48]{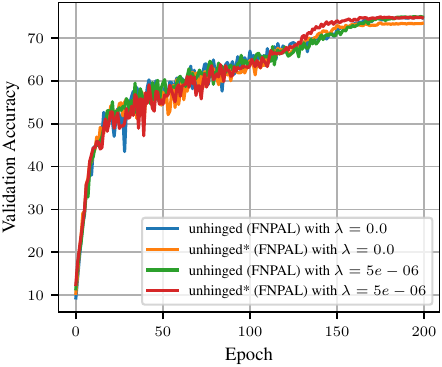}
        \label{fig:unhinged-top1-cifar100}
    }
    \caption{Behavior of visual classification on CIFAR-100 with CE, the unhinged loss (PAL) and the unhinged loss (FNPAL) under different weight decay coefficients.}
    \label{ce-fsm-asm-cifar100}
\end{figure}

\begin{figure}
    \centering
    \subfigure[Train Loss]{
        \includegraphics[scale=0.42]{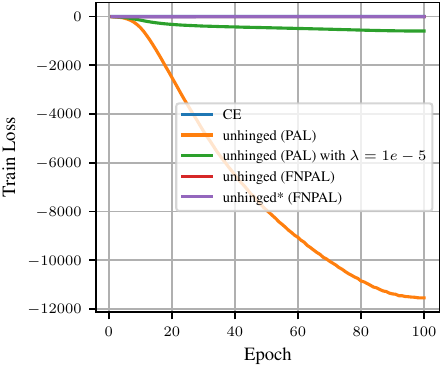}
        \label{fig:loss-imagenet100}
    }
    \subfigure[Feature Norm]{
        \includegraphics[scale=0.42]{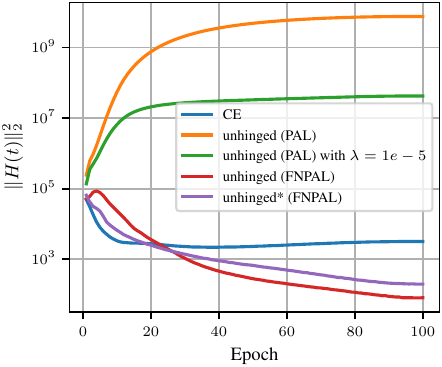}
        \label{fig:feat-imagenet100}
    }
    \subfigure[Train Accuracy]{
        \includegraphics[scale=0.42]{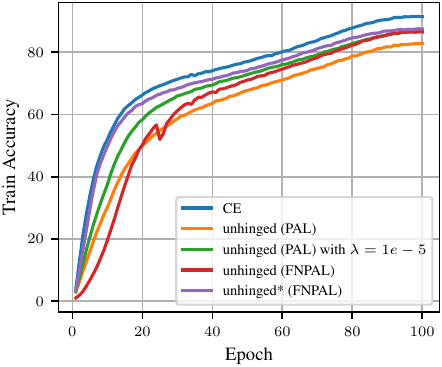}
        \label{fig:train-top1-imagenet100}
    }
    \subfigure[Validation Accuracy]{
        \includegraphics[scale=0.42]{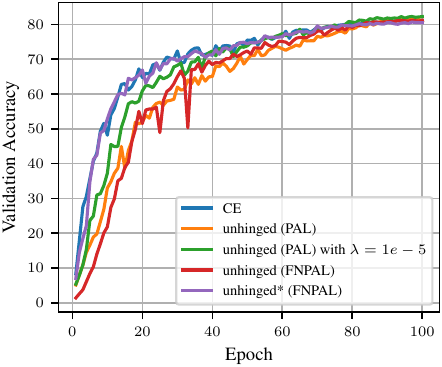}
        \label{fig:eval-top1-imagenet100}
    }
    \caption{Behavior of visual classification on ImageNet-100 with CE, the unhinged loss (PAL) and the unhinged loss (FNPAL) under different weight decay coefficients.}
    \label{imagenet100}
\end{figure}

\begin{figure}
    \centering
    \subfigure[Train Loss]{
        \includegraphics[scale=0.42]{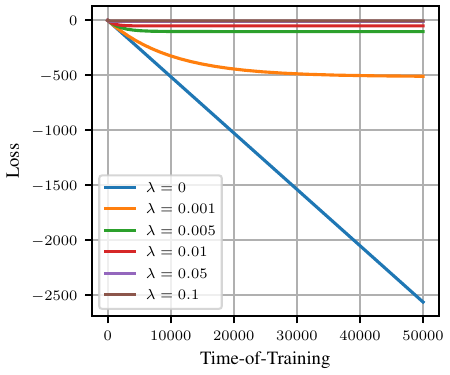}
        \label{fig:pal-loss}
    }
    \subfigure[Train Accuracy]{
        \includegraphics[scale=0.42]{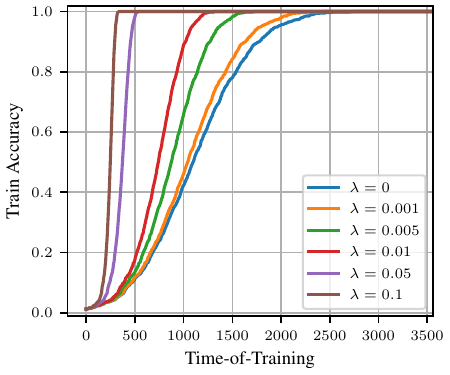}
        \label{fig:pal-acc}
    }
    \subfigure[$\|\hat{\mH}(t)-\hat{\mW\mM}\|_2$]{
        \includegraphics[scale=0.42]{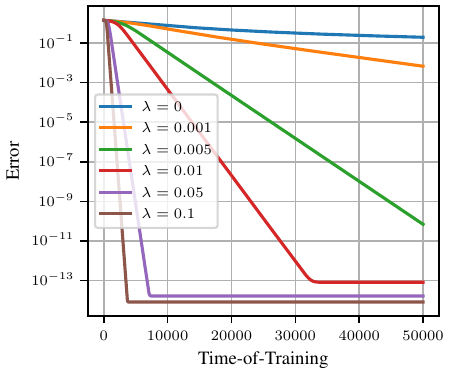}
        \label{fig:pal-error}
    }
    \subfigure[Feature Norm]{
        \includegraphics[scale=0.42]{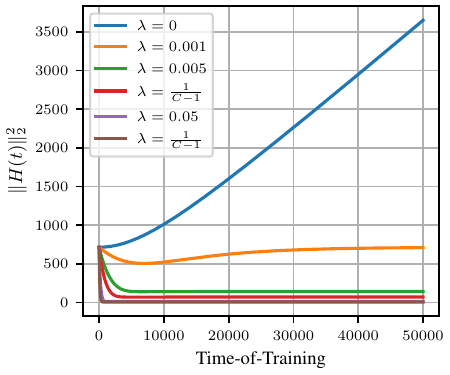}
        \label{fig:pal-feat}
    }
    \caption{Behavior of gradient descent iterates of the unhinged (PAL) loss in \cref{PAL-dynamics} with different explicit feature regularization ($\lambda\in\{0, 0.001, 0.005, 0.01, 0.05, 0.1\}$). We set $p=512$, $C=100$, $N=10$, and $\eta=0.1$. We randomly initialize $\mH_0$ and $\mW$, and then anchor prototypes $\mW$ during training. As expected in \cref{PAL-dynamics}, the error $\|\hat{\mH(t)}-\hat{\mW\mM}\|_2$ decreases as an exponential rate $O(e^{-\lambda \eta t})$, and a larger $\lambda$ can accelerate the convergence.}
    \label{pal-toy}
\end{figure}

\begin{figure}
    \centering
    \subfigure[$\lambda=0$]{
        \includegraphics[scale=0.42]{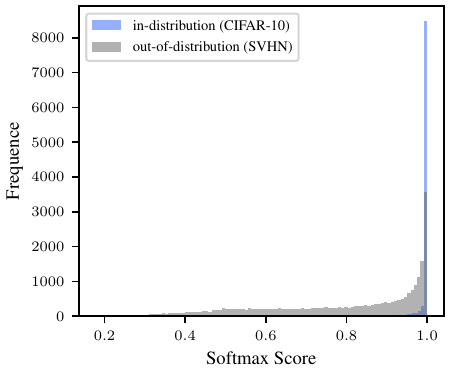}
        \label{fig:ood-cifar-10-softmax-0.0}
    }
    \subfigure[$\lambda=1e-6$]{
        \includegraphics[scale=0.42]{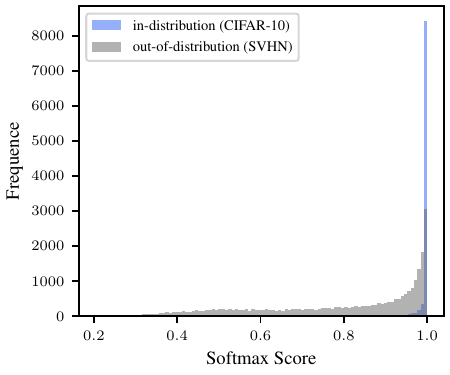}
       \label{fig:ood-cifar-10-softmax-1e-6}
    }
    \subfigure[$\lambda=5e-6$]{
        \includegraphics[scale=0.42]{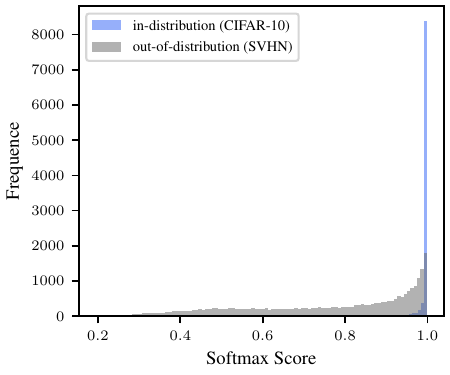}
       \label{fig:ood-cifar-10-softmax-5e-6}
    }
    \subfigure[$\lambda=1e-5$]{
        \includegraphics[scale=0.42]{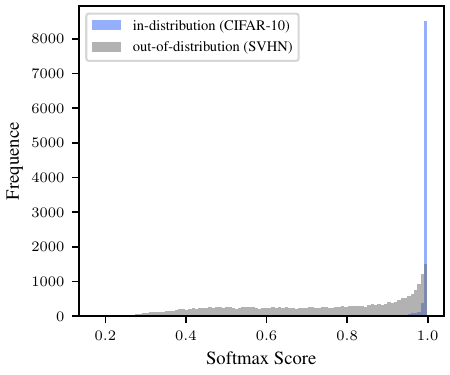}
       \label{fig:ood-cifar-10-softmax-1e-5}
    }
    
    \subfigure[$\lambda=0.0$]{
        \includegraphics[scale=0.42]{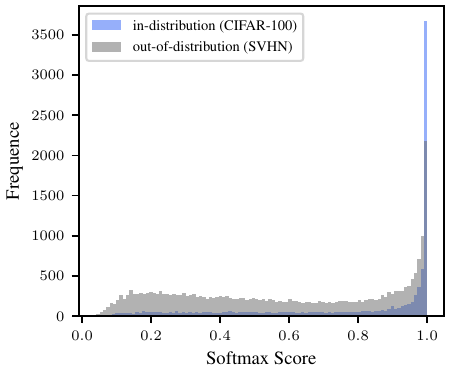}
       \label{fig:ood-cifar-100-softmax-0.0}
    }
    \subfigure[$\lambda=1e-6$]{
        \includegraphics[scale=0.42]{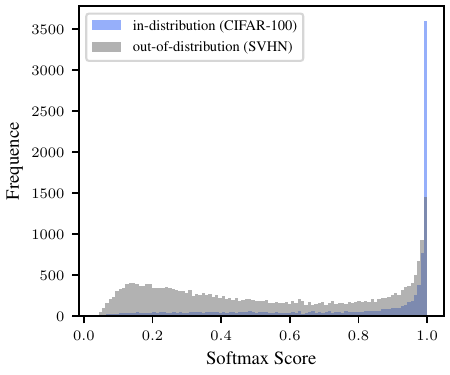}
       \label{fig:ood-cifar-100-softmax-1e-6}
    }
    \subfigure[$\lambda=5e-6$]{
        \includegraphics[scale=0.42]{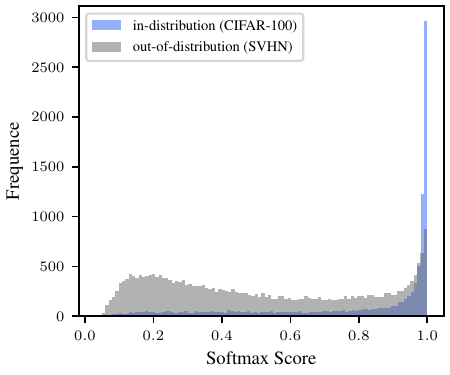}
       \label{fig:ood-cifar-100-softmax-5e-6}
    }
    \subfigure[$\lambda=1e-5$]{
        \includegraphics[scale=0.42]{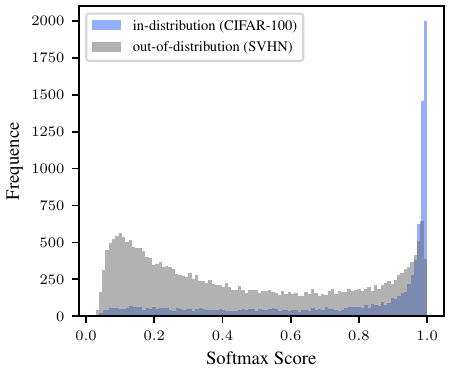}
       \label{fig:ood-cifar-100-softmax-1e-5}
    }
    
    \caption{Distribution of softmax scores \citep{hendrycks2016baseline} from models trained with different explicit feature regularization, where CE is the loss function.}
   \label{fig:ood-softmax-score}
   \vskip-15pt
\end{figure}

\begin{figure}
    \centering
    \subfigure[$\lambda=0$]{
        \includegraphics[scale=0.42]{figures/OOD/cifar10/svhn/0_energy.pdf}
        \label{fig:ood-cifar-10-energy-0.0}
    }
    \subfigure[$\lambda=1e-6$]{
        \includegraphics[scale=0.42]{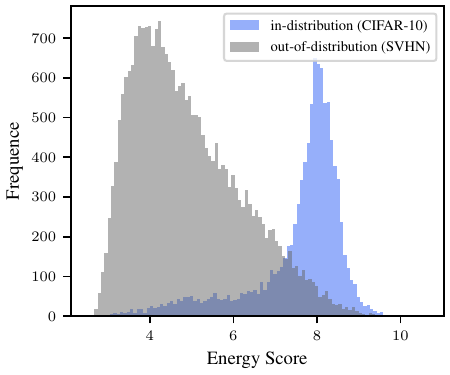}
        \label{fig:ood-cifar-10-energy-1e-6}
    }
    \subfigure[$\lambda=5e-6$]{
        \includegraphics[scale=0.42]{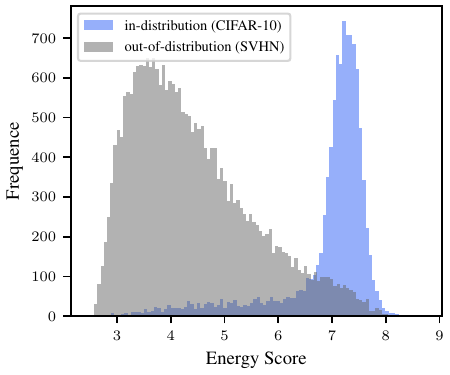}
        \label{fig:ood-cifar-10-energy-5e-6}
    }
    \subfigure[$\lambda=1e-5$]{
        \includegraphics[scale=0.42]{figures/OOD/cifar10/svhn/1e-05_energy.pdf}
        \label{fig:ood-cifar-10-energy-1e-5}
    }
    \subfigure[$\lambda=0$]{
        \includegraphics[scale=0.42]{figures/OOD/cifar100/svhn/0_energy.pdf}
        \label{fig:ood-cifar-100-energy-0.0}
    }
    \subfigure[$\lambda=1e-6$]{
        \includegraphics[scale=0.42]{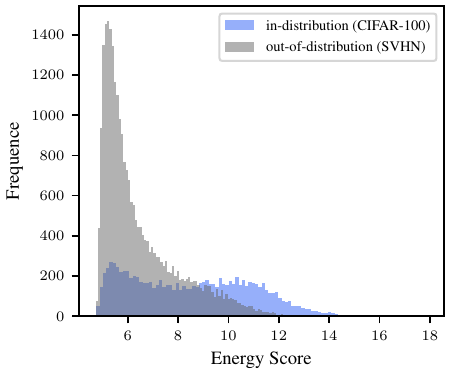}
        \label{fig:ood-cifar-100-energy-1e-6}
    }
    \subfigure[$\lambda=5e-6$]{
        \includegraphics[scale=0.42]{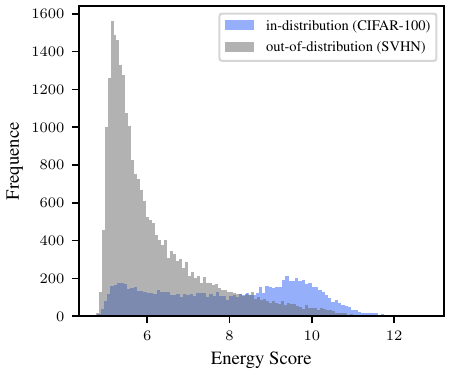}
        \label{fig:ood-cifar-100-energy-5e-6}
    }
    \subfigure[$\lambda=1e-5$]{
        \includegraphics[scale=0.42]{figures/OOD/cifar100/svhn/1e-05_energy.pdf}
        \label{fig:ood-cifar-100-energy-1e-5}
    }
    \caption{Distribution of energy scores \citep{liu2020energy} from models trained with different explicit feature regularization, where CE is the loss function.}
   \label{fig:ood-energy-score}
   \vskip-15pt
\end{figure}

\begin{figure}
    \centering
    \subfigure[$\lambda=0$]{
        \includegraphics[scale=0.42]{figures/OOD/cifar10/svhn/0_norm.pdf}
        \label{fig:ood-cifar-10-feat-0.0}
    }
    \subfigure[$\lambda=1e-6$]{
        \includegraphics[scale=0.42]{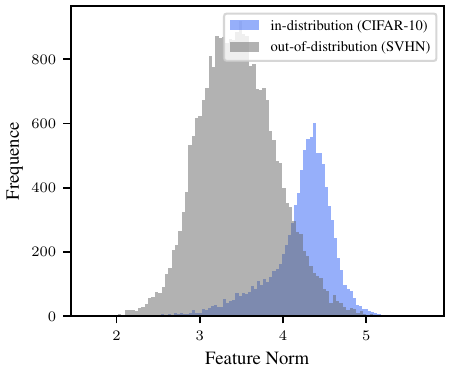}
        \label{fig:ood-cifar-10-feat-1e-6}
    }
    \subfigure[$\lambda=5e-6$]{
        \includegraphics[scale=0.42]{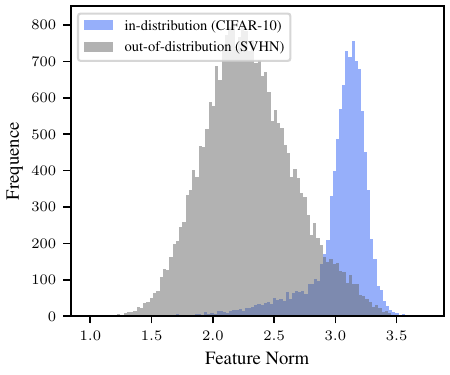}
        \label{fig:ood-cifar-10-feat-5e-6}
    }
    \subfigure[$\lambda=1e-5$]{
        \includegraphics[scale=0.42]{figures/OOD/cifar10/svhn/1e-05_norm.pdf}
        \label{fig:ood-cifar-10-feat-1e-5}
    }
    \subfigure[$\lambda=0$]{
        \includegraphics[scale=0.42]{figures/OOD/cifar100/svhn/0_norm.pdf}
        \label{fig:ood-cifar-100-feat-0.0}
    }
    \subfigure[$\lambda=1e-6$]{
        \includegraphics[scale=0.42]{figures/OOD/cifar100/svhn/1e-06_energy.pdf}
        \label{fig:ood-cifar-100-feat-1e-6}
    }
    \subfigure[$\lambda=5e-6$]{
        \includegraphics[scale=0.42]{figures/OOD/cifar100/svhn/5e-06_energy.pdf}
        \label{fig:ood-cifar-100-feat-5e-6}
    }
    \subfigure[$\lambda=1e-5$]{
        \includegraphics[scale=0.42]{figures/OOD/cifar100/svhn/1e-05_norm.pdf}
        \label{fig:ood-cifar-100-feat-1e-5}
    }
    \caption{Distribution of feature norms from models trained with different explicit feature regularization, where CE is the loss function.}
   \label{fig:ood-feat-norm}
   \vskip-15pt
\end{figure}

\begin{figure}
    \centering
    \subfigure[]{
        \includegraphics[scale=0.33]{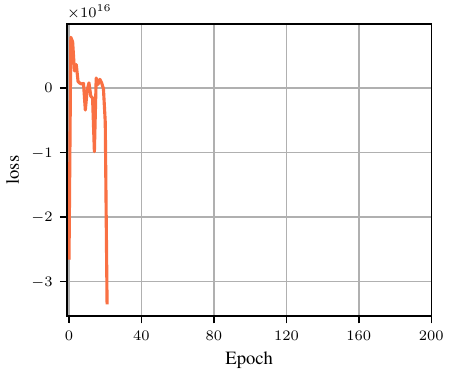}
        \label{fig:collapse-loss}
    }
    \subfigure[]{
        \includegraphics[scale=0.32]{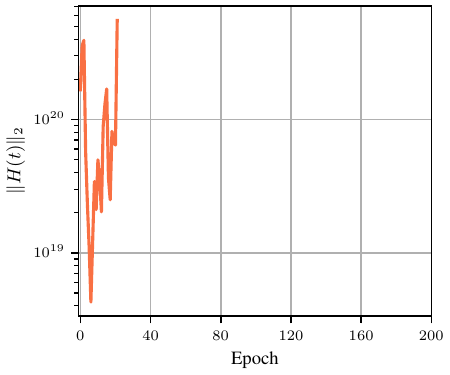}
        \label{fig:collapse-feature}
    }
    \subfigure[]{
        \includegraphics[scale=0.33]{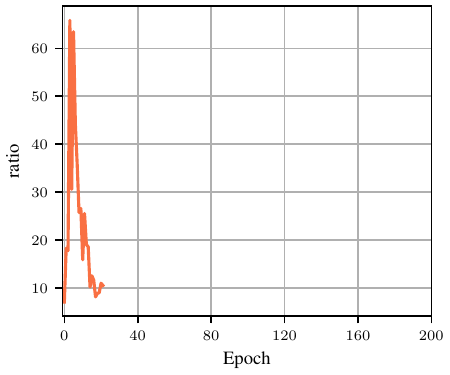}
        \label{fig:collapse-ratio}
    }
    \subfigure[]{
        \includegraphics[scale=0.33]{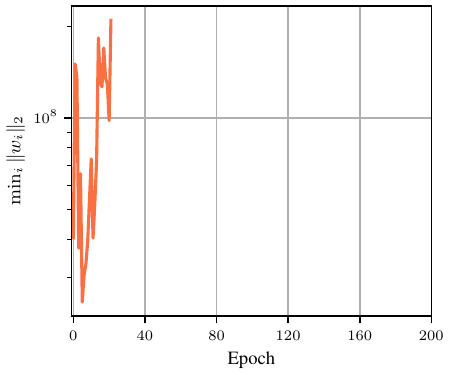}
        \label{fig:collapse-norm}
    }
    \subfigure[]{
        \includegraphics[scale=0.33]{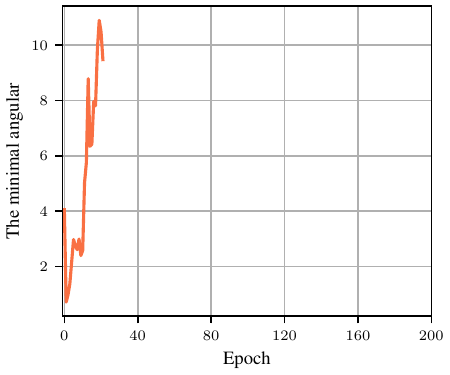}
        \label{fig:collapse-angular}
    }
    \caption{The behavior of features and prototypes when directly training ResNet-18 with the unhinged loss in \cref{the-unhinged-loss} on CIFAR-10. We set the weight decay coefficient as $5e-4$. (a) The train accuracy. (b) The feature norm. (c) the ratio $\frac{\max_i\|\vw_i\|_2}{\min_i\|\vw_i\|_2}$. (d) $\min_i\|\vw_i\|_2$. (e) The minimal angular between prototypes:  $\arccos\max_{i\neq j}\hat{\vw}_i^\top\hat{\vw}_j$. In these figures, we only show the curves for the first 21 epochs, since  ``NaN'' appears at the 22-th epoch. We can find that implicit penalization attached by other components (e.g., network architectures and weight decays) does not limit the rapid growth of the feature norm and prototype norm, indicating implicit penalization is fragile. Moreover, the ratio $\frac{\max_i\|\vw_i\|_2}{\min_i\|\vw_i\|_2}$ starts our very large and the minimal angular is very small, which indicates that there are two prototypes that are particular imbalanced. }
    \label{fig:collapse-of-unhinged}
\end{figure}

\section{Experiments}
\label{appendix-experiments}
In this section, we provide experimental details, including datasets, network architectures, optimization methods, hyperparameter settings, and more results.

\subsection{Numerical Experiments}
For numerical experiments in Figures \ref{unconstrained-illustration}, \ref{fig:unconstrained-behaviors}, \ref{regularized-gamma}, \ref{regularized-lr}, \ref{regularized-lambda}, \ref{regularized-scale}, and \ref{spherical-lr-behaviors}, we set $p=512$, $C=100$, $N=10$, and then randomly initialize $\mH_0$ and $\mW_0$. We use the SGD optimizer to optimize these free variables. 

\subsection{Visual Classification}
For classification experiments in \cref{classification-results}, \cref{ce-fsm-asm-cifar10}, \cref{ce-fsm-asm-cifar100}, \cref{imagenet100}, we experiment with ResNet-18, ResNet-34, and ResNet-50 \citep{he2016deep} trained on CIFAR-10, CIFAR-100 \citep{CIFAR}, and ImageNet-100 that takes the first 100 classes of ImageNet \citep{deng2009imagenet} , respectively. The networks are trained for 200 epochs and 100 epochs for CIFAR-10/-100 and ImageNet-100, respectively. For all training, we use SGD optimizer with momentum 0.9 and cosine learning rate annealing \cite{loshchilov2017sgdr} with $T_{\max}$ being the corresponding epochs. The initial learning rate is set to $0.1$, weight decay is set to $5\times 10^{-4}$, and batch size is set to 256. Typical data augmentations including random width/height shift and horizontal flip are applied. Moreover, to use the PAL and FNPAL \citep{pmlr-v162-zhou22f} that anchors prototypes with a neural collapse solution, we remove the ReLU layer before the linear classifier in the last layer.

\subsection{Imbalanced Classification} 
For the experiments of imbalanced learning in \cref{tab:lt-cifar}, \cref{tab:imbalanced-cifar-10}, \cref{tab:imbalanced-cifar-100}, \cref{fig:ece-lt-cifar10}, and \cref{fig:ece-lt-cifar100}, we utilize the same network architectures, and optimization settings as visual classification. We only use the imbalanced versions of CIFAR-10 and CIFAR-100 by following the setting in \citep{zhou2022learning}. The number of training examples is reduced for per class, and the test set keeps unchanged, where we use the imbalance ratio $\rho=\frac{\max_i n_i}{\min_i n_i}$ to denote the ratio between sample sizes of the most frequent and least frequent class. Moreover, long-tailed imbalance \citep{cui2019class} that utilizes an exponential decay in samples sizes and step imbalance \citep{buda2018systematic}(that sets all minority classes to have the same number of samples, as do all majority classes) are considered.

For imbalanced learning, we utilize expected calibration error (ECE) to measure calibration of the models \citep{zhong2021improving}, where all predictions are grouped into several interval bins of equal size and then calculate the error between the accuracy and confidence for each interval bin, \textit{i.e},  
\begin{equation}
    \mathop{\text{ECE}}=\sum_{b=1}^B\frac{|\gS_b|}{N}\left|\mathop{\text{acc}}(\gS_b)-\mathop{\text{conf}}(\gS_b)\right| \times 100\%,
\end{equation}
where $N$ denotes the number of predictions, $B$ is the number of interval bins, $\gS_b$ is the set of samples whose prediction scores fall into Bin-$b$,  $\mathop{\text{acc}}(\cdot)$ and $\mathop{\text{conf}}(\cdot)$ denote the accuracy and predicted confidence of $\gS_b$, respectively.

As shown in \cref{tab:imbalanced-cifar-10} and \cref{tab:imbalanced-cifar-100}, explicit feature regularization can improve imbalanced learning on CIFAR-10/-100 in most cases. 

\subsection{Out-of-Distribution Detection}
For the experiments of OOD detection in \cref{fig:cifar10-ood-energy-score}, \cref{tab:ood-cifar}, \cref{fig:ood-softmax-score}, \cref{fig:ood-energy-score}, and \cref{fig:ood-feat-norm}, we use a ResNet-18 on CIFAR-10 and a ResNet-34 on CIFAR-100 to train the classification models, and use their test dataset as the in-distribution data $\mathcal{D}_{\text{in}}^{\text{test}}$. For the OOD test dataset $\mathcal{D}_{\text{out}}^{\text{test}}$, we simply use a common benchmark: SVHN \citep{Netzer2011reading}. We measure the performance with the following metrics: (1) the false positive rate (FPR95) of OOD examples when true positive rate of in-distribution examples is at $95\%$; (2) the area under the receiver operating characteristic curve (AUROC); and (3) the area under the precision-recall curve (AUPR). We then consider the softmax-based score \citep{hendrycks2016baseline}, energy-based score \citep{liu2020energy}, and our proposed feature norm-based score to assessing the improvement of explicit feature regularization over the normal training.

\begin{table}[htbp]
    \scriptsize
    \centering
    \caption{OOD detection performance using softmax-based \citep{hendrycks2016baseline}, energy-based \citep{liu2020energy}, and feature norm-based approaches while model training with feature regularization ($\lambda=\{0,1e-6,5e-6,1e-5\}$). We use ResNet-18 and ResNet-34 to train on the in-distribution datasets CIFAR-10 and CIFAR-100, respectively. We then use SVHN \citep{Netzer2011reading} as the OOD dataset to evaluate the performance of OOD detection. All values are percentages. $\uparrow$ indicates large values are better, and $\downarrow$ indicates smaller values are better. The best results are \underline{underlined}.}
    \label{tab:ood-cifar}
    \begin{tabular}{c|c|c|c|c}
    \toprule
    \textbf{Dataset} $\gD_{\text{in}}^{\text{test}}$ & $\lambda$ & \textbf{FPR95} $\downarrow$  & \textbf{AUROC} $\uparrow$ & \textbf{AUPR} $\uparrow$ \\
    \midrule
    \multirow{5}*{CIFAR-10} & \multicolumn{4}{c}{Softmax-based / Energy-based / Feature Norm-based}\\
    ~ & 0    & 52.09 / 43.04 / 52.10 & 91.67 / 91.94 / 89.54 & 84.11 / 82.80 / 77.06 \\
    ~ & 1e-6 & {54.00} / {43.72} / \textbf{51.45} & {91.44} / \textbf{92.12} / {89.08} & {82.31} / {81.77} / {74.16} \\
    ~ & 5e-6 & \textbf{45.37} / \textbf{33.92} / \textbf{26.93} & \textbf{93.08} / \textbf{93.78} / \textbf{94.03} & \textbf{84.31} / \textbf{83.73} / \textbf{82.79} \\
    ~ & 1e-5 & \textbf{37.39} / \textbf{27.87} / \underline{\textbf{24.94}} & \textbf{93.90} / \underline{\textbf{94.60}} / \textbf{94.17} & \underline{\textbf{85.48}} / \textbf{85.34} / \textbf{83.15} \\
    \midrule
     \multirow{5}*{CIFAR-100} & \multicolumn{4}{c}{Softmax-based / Energy-based / Feature Norm-based}\\
    ~ & 0    & 87.75 / 89.84 / 95.54 & 71.01 / 71.94 / 59.54 & 55.42 / 56.69 / 43.21 \\
    ~ & 1e-6 & \textbf{82.08} / \textbf{82.57} / \textbf{88.77} & \textbf{75.36} / \textbf{76.28} / \textbf{68.83} & \textbf{61.40} / \textbf{61.90} / \textbf{51.58} \\
    ~ & 5e-6 & \textbf{79.01} / \textbf{78.68} / \textbf{85.94} & \textbf{78.70} / \underline{\textbf{79.15}} / \textbf{70.32} & \textbf{62.58} / \textbf{62.39} / \textbf{48.39} \\
    ~ & 1e-5 & \textbf{81.48} / \textbf{81.41} / \textbf{87.83} & \textbf{77.02} / \textbf{78.03} / \textbf{73.91} & \textbf{62.92} / \underline{\textbf{63.66}} / \textbf{58.81} \\
    \bottomrule
    \end{tabular}
\end{table}

\section{Other Potential Insights}
\label{potential-insights}

\subsection{A Good Initialization of Prototypes}
As depicted in \cref{main-results} and \ref{projection-of-subspaces}, the dynamics under the unhinged loss is dependent on the initialization of both features and prototypes, such as $\Pi_1^+\mZ_0=(\frac{1}{\sqrt{N}}(\mP\otimes \1_N^\top),\mP)$, where $\mP=\frac{1}{2}\left(\frac{1}{\sqrt{N}}\mH_0(\mI_C\otimes \1_N)+\mW_0\right)(\mI_C-\frac{1}{C}\1_C\1_C^\top)$. However, these features $\mH_0$ extracted from a dataset by some nonlinear layers and parameterized layers are practically intractable, but we can elaborately initialize $\mW_0$ and highlight its role in the whole. To do this, we consider two ways: (1) Initializing the structure of $\mW_0$. Inspired by the neural collapse solution that maximizes class separation, we can initialize $\mW_0$ as this structure, \textit{i.e.}, $\hat{\vw}_i^\top\hat{\vw}_j=\frac{-1}{C-1}$, $\forall i\neq j$; (2) Increasing the importance of $\mW_0$. A simple strategy is scaling up $\mW_0$, thereby implicitly weakening the importance of $\mH_0$.
However, it is difficult to handle the initialization of features because they are obtained by a complex processing a large dataset, thus we seek to initialize the prototypes in the last layer of the network.

\subsection{Refined Decision-makings}
Recalling the rule—$\argmax_{c'}\langle\vw_{c'},\vh\rangle+b_{c'}$ that makes decision by selecting the class with the largest logit (where the inner product $\langle \vw_{c'},\vh\rangle$ is dominant), which may not be good to directly use the learned features and prototypes, since learning with the unhinged loss within limited iterations (that means $\zeta_1(t)<\infty$) will introduce some residual $\bm{\Delta}(t)$ caused by gradient descent regardless of the unconstrained case or regularized case. 

\begin{example}
If we add a perturbation $\Delta$ for all features while adding $s\Delta$ for all prototypes, then the perturbed decision-making will be $\arg\max_{c'}\langle \vw_{c'}+s\Delta,
\vh+\Delta\rangle+b_{c'}$, which may not be equivalent to $\argmax_{c'}\langle\vw_{c'},\vh\rangle+b_{c'}$.
\end{example}

\subsection{Adjusted Sample Margin Loss}
As aforementioned in \cref{sec:spherical-case} and the proof of Theorem \ref{convergence-of-spherical-case}, we will encounter zero gradients when the cosine similarity $\hat{\vw}_y^\top\hat{\vh}$ is $-1$ or $1$, so we can adjust the loss to avoid the issue by to $\vw_y$ and accelerate convergence:
\begin{equation}
    L_{\gamma}'(\mW\hat{\vh},y)=
    \begin{cases}
    L_{\gamma}(\mW\hat{\vh},y) &  \text{if } \hat{\vw}_y^\top\hat{\vh}\ge -1+\varepsilon,\\
    -(1+\gamma)(\vw_y+\bm{\delta})^\top\hat{\vh} &  \text{if } \hat{\vw}_y^\top\hat{\vh}<-1+\varepsilon,
    \end{cases}
\end{equation}
where $\varepsilon\in(0,1)$ is a hyperparameter and $\bm{\delta}=-\left(1+\frac{\hat{\vw}_y^\top\hat{\vh}\sqrt{1-(1-\varepsilon)^2}}{(1-\varepsilon)\sqrt{1-(\hat{\vw}_y^\top\hat{\vh})^2}}\right)(\vw_y+\hat{\vh}\hat{\vh}^\top \vw_y)$ (performed with a stop-gradient) satisfying $\frac{(\vw_y+\bm{\delta})^\top\hat{\vh}}{\|\vw_y+\bm{\delta}\|_2}=-1+\varepsilon$.


\begin{thebibliography}{70}
\providecommand{\natexlab}[1]{#1}
\providecommand{\url}[1]{\texttt{#1}}
\expandafter\ifx\csname urlstyle\endcsname\relax
  \providecommand{\doi}[1]{doi: #1}\else
  \providecommand{\doi}{doi: \begingroup \urlstyle{rm}\Url}\fi

\bibitem[Arjovsky et~al.(2017)Arjovsky, Chintala, and Bottou]{wgan}
Martin Arjovsky, Soumith Chintala, and Léon Bottou.
\newblock Wasserstein gan, 2017.
\newblock URL \url{https://arxiv.org/abs/1701.07875}.

\bibitem[Auer et~al.(2002)Auer, Cesa-Bianchi, and Gentile]{auer2002adaptive}
Peter Auer, Nicolo Cesa-Bianchi, and Claudio Gentile.
\newblock Adaptive and self-confident on-line learning algorithms.
\newblock \emph{Journal of Computer and System Sciences}, 64\penalty0 (1):\penalty0 48--75, 2002.

\bibitem[Bubeck and Sellke(2021)]{bubeck2021universal}
S{\'e}bastien Bubeck and Mark Sellke.
\newblock A universal law of robustness via isoperimetry.
\newblock \emph{Advances in Neural Information Processing Systems}, 34:\penalty0 28811--28822, 2021.

\bibitem[Buda et~al.(2018)Buda, Maki, and Mazurowski]{buda2018systematic}
Mateusz Buda, Atsuto Maki, and Maciej~A Mazurowski.
\newblock A systematic study of the class imbalance problem in convolutional neural networks.
\newblock \emph{Neural Networks}, 106:\penalty0 249--259, 2018.

\bibitem[Cao et~al.(2019)Cao, Wei, Gaidon, Arechiga, and Ma]{cao2019learning}
Kaidi Cao, Colin Wei, Adrien Gaidon, Nikos Arechiga, and Tengyu Ma.
\newblock Learning imbalanced datasets with label-distribution-aware margin loss.
\newblock \emph{Advances in neural information processing systems}, 32, 2019.

\bibitem[Chizat et~al.(2019)Chizat, Oyallon, and Bach]{chizat2019lazy}
Lenaic Chizat, Edouard Oyallon, and Francis Bach.
\newblock On lazy training in differentiable programming.
\newblock \emph{Advances in Neural Information Processing Systems}, 32, 2019.

\bibitem[Cui et~al.(2019)Cui, Jia, Lin, Song, and Belongie]{cui2019class}
Yin Cui, Menglin Jia, Tsung-Yi Lin, Yang Song, and Serge Belongie.
\newblock Class-balanced loss based on effective number of samples.
\newblock In \emph{Proceedings of the IEEE/CVF Conference on Computer Vision and Pattern Recognition}, pages 9268--9277, 2019.

\bibitem[Demirkaya et~al.(2020)Demirkaya, Chen, and Oymak]{demirkaya2020exploring}
Ahmet Demirkaya, Jiasi Chen, and Samet Oymak.
\newblock Exploring the role of loss functions in multiclass classification.
\newblock In \emph{2020 54th annual conference on information sciences and systems (ciss)}, pages 1--5. IEEE, 2020.

\bibitem[Deng et~al.(2009)Deng, Dong, Socher, Li, Li, and Fei-Fei]{deng2009imagenet}
Jia Deng, Wei Dong, Richard Socher, Li-Jia Li, Kai Li, and Li~Fei-Fei.
\newblock Imagenet: A large-scale hierarchical image database.
\newblock In \emph{2009 IEEE conference on computer vision and pattern recognition}, pages 248--255. Ieee, 2009.

\bibitem[Duchi et~al.(2011)Duchi, Hazan, and Singer]{duchi2011adaptive}
John Duchi, Elad Hazan, and Yoram Singer.
\newblock Adaptive subgradient methods for online learning and stochastic optimization.
\newblock \emph{Journal of machine learning research}, 12\penalty0 (7), 2011.

\bibitem[Fang et~al.(2021{\natexlab{a}})Fang, Dong, and Zhang]{fang2021mathematical}
Cong Fang, Hanze Dong, and Tong Zhang.
\newblock Mathematical models of overparameterized neural networks.
\newblock \emph{Proceedings of the IEEE}, 109\penalty0 (5):\penalty0 683--703, 2021{\natexlab{a}}.

\bibitem[Fang et~al.(2021{\natexlab{b}})Fang, He, Long, and Su]{fang2021exploring}
Cong Fang, Hangfeng He, Qi~Long, and Weijie~J Su.
\newblock Exploring deep neural networks via layer-peeled model: Minority collapse in imbalanced training.
\newblock \emph{Proceedings of the National Academy of Sciences}, 118\penalty0 (43):\penalty0 e2103091118, 2021{\natexlab{b}}.

\bibitem[Galanti et~al.(2021)Galanti, Gy{\"o}rgy, and Hutter]{galanti2021role}
Tomer Galanti, Andr{\'a}s Gy{\"o}rgy, and Marcus Hutter.
\newblock On the role of neural collapse in transfer learning.
\newblock In \emph{International Conference on Learning Representations}, 2021.

\bibitem[Gao et~al.(2023)Gao, Xu, Wen, Shao, Yang, and Huang]{gao2023study}
Peifeng Gao, Qianqian Xu, Peisong Wen, Huiyang Shao, Zhiyong Yang, and Qingming Huang.
\newblock A study of neural collapse phenomenon: Grassmannian frame, symmetry, generalization.
\newblock \emph{arXiv preprint arXiv:2304.08914}, 2023.

\bibitem[Gunasekar et~al.(2018)Gunasekar, Lee, Soudry, and Srebro]{gunasekar2018implicit}
Suriya Gunasekar, Jason~D Lee, Daniel Soudry, and Nati Srebro.
\newblock Implicit bias of gradient descent on linear convolutional networks.
\newblock \emph{Advances in Neural Information Processing Systems}, 31, 2018.

\bibitem[Han et~al.(2022)Han, Papyan, and Donoho]{han2022neural}
XY~Han, Vardan Papyan, and David~L Donoho.
\newblock Neural collapse under mse loss: Proximity to and dynamics on the central path.
\newblock In \emph{International Conference on Learning Representations}, 2022.

\bibitem[Hardt et~al.(2016)Hardt, Recht, and Singer]{hardt2016train}
Moritz Hardt, Ben Recht, and Yoram Singer.
\newblock Train faster, generalize better: Stability of stochastic gradient descent.
\newblock In \emph{International conference on machine learning}, pages 1225--1234. PMLR, 2016.

\bibitem[He et~al.(2016)He, Zhang, Ren, and Sun]{he2016deep}
Kaiming He, Xiangyu Zhang, Shaoqing Ren, and Jian Sun.
\newblock Deep residual learning for image recognition.
\newblock In \emph{Proceedings of the IEEE conference on computer vision and pattern recognition}, pages 770--778, 2016.

\bibitem[Hendrycks and Gimpel(2016)]{hendrycks2016baseline}
Dan Hendrycks and Kevin Gimpel.
\newblock A baseline for detecting misclassified and out-of-distribution examples in neural networks.
\newblock \emph{arXiv preprint arXiv:1610.02136}, 2016.

\bibitem[Hui and Belkin(2021)]{hui2021evaluation}
Like Hui and Mikhail Belkin.
\newblock Evaluation of neural architectures trained with square loss vs cross-entropy in classification tasks.
\newblock In \emph{International Conference on Learning Representations}, 2021.

\bibitem[Hui et~al.(2022)Hui, Belkin, and Nakkiran]{hui2022limitations}
Like Hui, Mikhail Belkin, and Preetum Nakkiran.
\newblock Limitations of neural collapse for understanding generalization in deep learning.
\newblock \emph{arXiv preprint arXiv:2202.08384}, 2022.

\bibitem[Jacot et~al.(2018)Jacot, Gabriel, and Hongler]{jacot2018neural}
Arthur Jacot, Franck Gabriel, and Cl{\'e}ment Hongler.
\newblock Neural tangent kernel: Convergence and generalization in neural networks.
\newblock \emph{Advances in neural information processing systems}, 31, 2018.

\bibitem[Ji et~al.(2022)Ji, Lu, Zhang, Deng, and Su]{ji2022unconstrained}
Wenlong Ji, Yiping Lu, Yiliang Zhang, Zhun Deng, and Weijie~J Su.
\newblock An unconstrained layer-peeled perspective on neural collapse.
\newblock In \emph{International Conference on Learning Representations}, 2022.

\bibitem[Ji and Telgarsky(2019)]{ji2019implicit}
Ziwei Ji and Matus Telgarsky.
\newblock The implicit bias of gradient descent on nonseparable data.
\newblock In \emph{Conference on Learning Theory}, pages 1772--1798. PMLR, 2019.

\bibitem[Ji et~al.(2020)Ji, Dud{\'\i}k, Schapire, and Telgarsky]{ji2020gradient}
Ziwei Ji, Miroslav Dud{\'\i}k, Robert~E Schapire, and Matus Telgarsky.
\newblock Gradient descent follows the regularization path for general losses.
\newblock In \emph{Conference on Learning Theory}, pages 2109--2136. PMLR, 2020.

\bibitem[Kasarla et~al.(2022)Kasarla, Burghouts, van Spengler, van~der Pol, Cucchiara, and Mettes]{kasarla2022maximum}
Tejaswi Kasarla, Gertjan~J Burghouts, Max van Spengler, Elise van~der Pol, Rita Cucchiara, and Pascal Mettes.
\newblock Maximum class separation as inductive bias in one matrix.
\newblock \emph{arXiv preprint arXiv:2206.08704}, 2022.

\bibitem[Kingma and Ba(2015)]{kingma2015adam}
Diederik~P. Kingma and Jimmy Ba.
\newblock Adam: A method for stochastic optimization.
\newblock In \emph{ICLR}, 2015.

\bibitem[Koltchinskii and Panchenko(2002)]{koltchinskii2002empirical}
Vladimir Koltchinskii and Dmitry Panchenko.
\newblock Empirical margin distributions and bounding the generalization error of combined classifiers.
\newblock \emph{The Annals of Statistics}, 30\penalty0 (1):\penalty0 1--50, 2002.

\bibitem[Kornblith et~al.(2021)Kornblith, Chen, Lee, and Norouzi]{kornblith2021better}
Simon Kornblith, Ting Chen, Honglak Lee, and Mohammad Norouzi.
\newblock Why do better loss functions lead to less transferable features?
\newblock \emph{Advances in Neural Information Processing Systems}, 34:\penalty0 28648--28662, 2021.

\bibitem[Kothapalli et~al.(2022)Kothapalli, Rasromani, and Awatramani]{kothapalli2022neural}
Vignesh Kothapalli, Ebrahim Rasromani, and Vasudev Awatramani.
\newblock Neural collapse: A review on modelling principles and generalization.
\newblock \emph{arXiv preprint arXiv:2206.04041}, 2022.

\bibitem[Krizhevsky and Hinton(2009)]{CIFAR}
A.~Krizhevsky and G.~Hinton.
\newblock Learning multiple layers of features from tiny images.
\newblock \emph{Computer Science Department, University of Toronto, Tech. Rep}, 1, 01 2009.

\bibitem[Langer et~al.(2021)Langer, Oster, Speith, Hermanns, Kästner, Schmidt, Sesing, and Baum]{LANGER2021103473}
Markus Langer, Daniel Oster, Timo Speith, Holger Hermanns, Lena Kästner, Eva Schmidt, Andreas Sesing, and Kevin Baum.
\newblock What do we want from explainable artificial intelligence (xai)? – a stakeholder perspective on xai and a conceptual model guiding interdisciplinary xai research.
\newblock \emph{Artificial Intelligence}, 296:\penalty0 103473, 2021.
\newblock ISSN 0004-3702.

\bibitem[LeCun et~al.(2015)LeCun, Bengio, and Hinton]{lecun2015deep}
Yann LeCun, Yoshua Bengio, and Geoffrey Hinton.
\newblock Deep learning.
\newblock \emph{nature}, 521\penalty0 (7553):\penalty0 436--444, 2015.

\bibitem[Lin et~al.(2017)Lin, Goyal, Girshick, He, and Doll{\'a}r]{lin2017focal}
Tsung-Yi Lin, Priya Goyal, Ross Girshick, Kaiming He, and Piotr Doll{\'a}r.
\newblock Focal loss for dense object detection.
\newblock In \emph{Proceedings of the IEEE international conference on computer vision}, pages 2980--2988, 2017.

\bibitem[Liu et~al.(2020)Liu, Wang, Owens, and Li]{liu2020energy}
Weitang Liu, Xiaoyun Wang, John Owens, and Yixuan Li.
\newblock Energy-based out-of-distribution detection.
\newblock \emph{Advances in Neural Information Processing Systems}, 33:\penalty0 21464--21475, 2020.

\bibitem[Liu et~al.(2016)Liu, Wen, Yu, and Yang]{liu2016large}
Weiyang Liu, Yandong Wen, Zhiding Yu, and Meng Yang.
\newblock Large-margin softmax loss for convolutional neural networks.
\newblock In \emph{International Conference on Machine Learning}, pages 507--516. PMLR, 2016.

\bibitem[Loshchilov and Hutter(2017)]{loshchilov2017sgdr}
Ilya Loshchilov and Frank Hutter.
\newblock {SGDR}: Stochastic gradient descent with warm restarts.
\newblock In \emph{International Conference on Learning Representations}, 2017.

\bibitem[Lu and Steinerberger(2022)]{lu2022neural}
Jianfeng Lu and Stefan Steinerberger.
\newblock Neural collapse under cross-entropy loss.
\newblock \emph{Applied and Computational Harmonic Analysis}, 59:\penalty0 224--241, 2022.

\bibitem[Mei and Montanari(2021)]{Mei2021TheGE}
Song Mei and Andrea Montanari.
\newblock The generalization error of random features regression: Precise asymptotics and the double descent curve.
\newblock \emph{Communications on Pure and Applied Mathematics}, 75, 2021.

\bibitem[Mixon et~al.(2022)Mixon, Parshall, and Pi]{mixon2022neural}
Dustin~G Mixon, Hans Parshall, and Jianzong Pi.
\newblock Neural collapse with unconstrained features.
\newblock \emph{Sampling Theory, Signal Processing, and Data Analysis}, 20\penalty0 (2):\penalty0 1--13, 2022.

\bibitem[Nacson et~al.(2019)Nacson, Lee, Gunasekar, Srebro, and Soudry]{Nacson2019ConvergenceOG}
Mor~Shpigel Nacson, J.~Lee, Suriya Gunasekar, Nathan Srebro, and Daniel Soudry.
\newblock Convergence of gradient descent on separable data.
\newblock In \emph{AISTATS}, 2019.

\bibitem[Nakkiran et~al.(2019)Nakkiran, Kaplun, Bansal, Yang, Barak, and Sutskever]{nakkiran2019deep}
Preetum Nakkiran, Gal Kaplun, Yamini Bansal, Tristan Yang, Boaz Barak, and Ilya Sutskever.
\newblock Deep double descent: Where bigger models and more data hurt.
\newblock In \emph{International Conference on Learning Representations}, 2019.

\bibitem[Netzer et~al.(2011)Netzer, Wang, Coates, Bissacco, Wu, and Ng]{Netzer2011reading}
Yuval Netzer, Tao Wang, Adam Coates, Alessandro Bissacco, Bo~Wu, and Andrew~Y. Ng.
\newblock Reading digits in natural images with unsupervised feature learning.
\newblock In \emph{NIPS Workshop on Deep Learning and Unsupervised Feature Learning 2011}, 2011.

\bibitem[Neyshabur et~al.(2017)Neyshabur, Bhojanapalli, McAllester, and Srebro]{neyshabur2017exploring}
Behnam Neyshabur, Srinadh Bhojanapalli, David McAllester, and Nati Srebro.
\newblock Exploring generalization in deep learning.
\newblock \emph{Advances in neural information processing systems}, 30, 2017.

\bibitem[Papyan et~al.(2020)Papyan, Han, and Donoho]{papyan2020prevalence}
Vardan Papyan, XY~Han, and David~L Donoho.
\newblock Prevalence of neural collapse during the terminal phase of deep learning training.
\newblock \emph{Proceedings of the National Academy of Sciences}, 117\penalty0 (40):\penalty0 24652--24663, 2020.

\bibitem[Poggio and Liao(2020)]{poggio2020explicit}
Tomaso Poggio and Qianli Liao.
\newblock Explicit regularization and implicit bias in deep network classifiers trained with the square loss.
\newblock \emph{arXiv preprint arXiv:2101.00072}, 2020.

\bibitem[Qi et~al.(2021)Qi, Khorram, and Fuxin]{QI2021103435}
Zhongang Qi, Saeed Khorram, and Li~Fuxin.
\newblock Embedding deep networks into visual explanations.
\newblock \emph{Artificial Intelligence}, 292:\penalty0 103435, 2021.
\newblock ISSN 0004-3702.

\bibitem[Sekhari et~al.(2021)Sekhari, Sridharan, and Kale]{sekhari2021sgd}
Ayush Sekhari, Karthik Sridharan, and Satyen Kale.
\newblock Sgd: The role of implicit regularization, batch-size and multiple-epochs.
\newblock \emph{Advances In Neural Information Processing Systems}, 34:\penalty0 27422--27433, 2021.

\bibitem[Shamir(2021)]{Shamir2021GradientMN}
Ohad Shamir.
\newblock Gradient methods never overfit on separable data.
\newblock \emph{J. Mach. Learn. Res.}, 22:\penalty0 85:1--85:20, 2021.

\bibitem[Soudry et~al.(2018)Soudry, Hoffer, Nacson, Gunasekar, and Srebro]{soudry2018implicit}
Daniel Soudry, Elad Hoffer, Mor~Shpigel Nacson, Suriya Gunasekar, and Nathan Srebro.
\newblock The implicit bias of gradient descent on separable data.
\newblock \emph{The Journal of Machine Learning Research}, 19\penalty0 (1):\penalty0 2822--2878, 2018.

\bibitem[Thrampoulidis et~al.(2022)Thrampoulidis, Kini, Vakilian, and Behnia]{thrampoulidis2022imbalance}
Christos Thrampoulidis, Ganesh~Ramachandra Kini, Vala Vakilian, and Tina Behnia.
\newblock Imbalance trouble: Revisiting neural-collapse geometry.
\newblock \emph{Advances in Neural Information Processing Systems}, 35:\penalty0 27225--27238, 2022.

\bibitem[Tian et~al.(2021)Tian, Chen, and Ganguli]{tian2021understanding}
Yuandong Tian, Xinlei Chen, and Surya Ganguli.
\newblock Understanding self-supervised learning dynamics without contrastive pairs.
\newblock In \emph{International Conference on Machine Learning}, pages 10268--10278. PMLR, 2021.

\bibitem[Tirer and Bruna(2022)]{tirer2022extended}
Tom Tirer and Joan Bruna.
\newblock Extended unconstrained features model for exploring deep neural collapse.
\newblock In \emph{Proceedings of the 39th International Conference on Machine Learning}, volume 162 of \emph{Proceedings of Machine Learning Research}, pages 21478--21505. PMLR, 2022.

\bibitem[Tirer et~al.(2023)Tirer, Huang, and Niles-Weed]{tirer2023perturbation}
Tom Tirer, Haoxiang Huang, and Jonathan Niles-Weed.
\newblock Perturbation analysis of neural collapse.
\newblock In \emph{International Conference on Machine Learning}, pages 34301--34329. PMLR, 2023.

\bibitem[Van~Rooyen et~al.(2015)Van~Rooyen, Menon, and Williamson]{van2015learning}
Brendan Van~Rooyen, Aditya Menon, and Robert~C Williamson.
\newblock Learning with symmetric label noise: The importance of being unhinged.
\newblock \emph{Advances in neural information processing systems}, 28, 2015.

\bibitem[Wang et~al.(2017)Wang, Xiang, Cheng, and Yuille]{wang2017normface}
Feng Wang, Xiang Xiang, Jian Cheng, and Alan~Loddon Yuille.
\newblock Normface: L2 hypersphere embedding for face verification.
\newblock In \emph{Proceedings of the 25th ACM international conference on Multimedia}, pages 1041--1049, 2017.

\bibitem[Wang et~al.(2021)Wang, Muthukumar, and Thrampoulidis]{wang2021benign}
Ke~Wang, Vidya Muthukumar, and Christos Thrampoulidis.
\newblock Benign overfitting in multiclass classification: All roads lead to interpolation.
\newblock \emph{Advances in Neural Information Processing Systems}, 34:\penalty0 24164--24179, 2021.

\bibitem[Xu et~al.(2023)Xu, Rangamani, Liao, Galanti, and Poggio]{xu2023dynamics}
Mengjia Xu, Akshay Rangamani, Qianli Liao, Tomer Galanti, and Tomaso Poggio.
\newblock Dynamics in deep classifiers trained with the square loss: Normalization, low rank, neural collapse, and generalization bounds.
\newblock \emph{Research}, 6:\penalty0 0024, 2023.

\bibitem[Yang and Hu(2020)]{yang2020feature}
Greg Yang and Edward~J Hu.
\newblock Feature learning in infinite-width neural networks.
\newblock \emph{arXiv preprint arXiv:2011.14522}, 2020.

\bibitem[Yang and Hu(2021)]{yang2021tensor}
Greg Yang and Edward~J Hu.
\newblock Tensor programs iv: Feature learning in infinite-width neural networks.
\newblock In \emph{International Conference on Machine Learning}, pages 11727--11737. PMLR, 2021.

\bibitem[Yang et~al.(2022)Yang, Chen, Li, Xie, Lin, and Tao]{yang2022inducing}
Yibo Yang, Shixiang Chen, Xiangtai Li, Liang Xie, Zhouchen Lin, and Dacheng Tao.
\newblock Inducing neural collapse in imbalanced learning: Do we really need a learnable classifier at the end of deep neural network?
\newblock In Alice~H. Oh, Alekh Agarwal, Danielle Belgrave, and Kyunghyun Cho, editors, \emph{Advances in Neural Information Processing Systems}, 2022.

\bibitem[Zeiler(2012)]{zeiler2012adadelta}
Matthew~D Zeiler.
\newblock Adadelta: an adaptive learning rate method.
\newblock \emph{arXiv preprint arXiv:1212.5701}, 2012.

\bibitem[Zhang et~al.(2017)Zhang, Bengio, Hardt, Recht, and Vinyals]{2017Understanding}
Chiyuan Zhang, Samy Bengio, Moritz Hardt, Benjamin Recht, and Oriol Vinyals.
\newblock Understanding deep learning requires rethinking generalization.
\newblock In \emph{International Conference on Learning Representation}, 2017.

\bibitem[Zhong et~al.(2021)Zhong, Cui, Liu, and Jia]{zhong2021improving}
Zhisheng Zhong, Jiequan Cui, Shu Liu, and Jiaya Jia.
\newblock Improving calibration for long-tailed recognition.
\newblock In \emph{Proceedings of the IEEE/CVF conference on computer vision and pattern recognition}, pages 16489--16498, 2021.

\bibitem[Zhou et~al.(2022{\natexlab{a}})Zhou, Li, Ding, You, Qu, and Zhu]{pmlr-v162-zhou22c}
Jinxin Zhou, Xiao Li, Tianyu Ding, Chong You, Qing Qu, and Zhihui Zhu.
\newblock On the optimization landscape of neural collapse under {MSE} loss: Global optimality with unconstrained features.
\newblock In \emph{Proceedings of the 39th International Conference on Machine Learning}, volume 162 of \emph{Proceedings of Machine Learning Research}, pages 27179--27202. PMLR, 17--23 Jul 2022{\natexlab{a}}.

\bibitem[Zhou et~al.(2022{\natexlab{b}})Zhou, You, Li, Liu, Liu, Qu, and Zhu]{zhouall}
Jinxin Zhou, Chong You, Xiao Li, Kangning Liu, Sheng Liu, Qing Qu, and Zhihui Zhu.
\newblock Are all losses created equal: A neural collapse perspective.
\newblock In \emph{Advances in Neural Information Processing Systems}, 2022{\natexlab{b}}.

\bibitem[Zhou et~al.(2022{\natexlab{c}})Zhou, Liu, Zhai, Jiang, Gao, and Ji]{pmlr-v162-zhou22f}
Xiong Zhou, Xianming Liu, Deming Zhai, Junjun Jiang, Xin Gao, and Xiangyang Ji.
\newblock Prototype-anchored learning for learning with imperfect annotations.
\newblock In \emph{Proceedings of the 39th International Conference on Machine Learning}, volume 162 of \emph{Proceedings of Machine Learning Research}, pages 27245--27267. PMLR, 17--23 Jul 2022{\natexlab{c}}.

\bibitem[Zhou et~al.(2022{\natexlab{d}})Zhou, Liu, Zhai, Jiang, Gao, and Ji]{zhou2022learning}
Xiong Zhou, Xianming Liu, Deming Zhai, Junjun Jiang, Xin Gao, and Xiangyang Ji.
\newblock Learning towards the largest margins.
\newblock In \emph{International Conference on Learning Representations}, 2022{\natexlab{d}}.

\bibitem[Zhou et~al.(2023)Zhou, Liu, Zhai, Jiang, and Ji]{zhou2023asymmetric}
Xiong Zhou, Xianming Liu, Deming Zhai, Junjun Jiang, and Xiangyang Ji.
\newblock Asymmetric loss functions for noise-tolerant learning: Theory and applications.
\newblock \emph{IEEE Transactions on Pattern Analysis and Machine Intelligence}, 2023.

\bibitem[Zhu et~al.(2021)Zhu, Ding, Zhou, Li, You, Sulam, and Qu]{zhu2021geometric}
Zhihui Zhu, Tianyu Ding, Jinxin Zhou, Xiao Li, Chong You, Jeremias Sulam, and Qing Qu.
\newblock A geometric analysis of neural collapse with unconstrained features.
\newblock \emph{Advances in Neural Information Processing Systems}, 34:\penalty0 29820--29834, 2021.

\end{thebibliography}


\vskip 0.2in
\end{document}